\documentclass{colt2017} % Include author names

\title[Towards Instance Optimal Bounds for Best Arm Identification]{Towards Instance Optimal Bounds for Best Arm Identification}
\usepackage{times}
\usepackage{enumerate}
\usepackage{mathtools}
\usepackage{indentfirst}
\usepackage{color}

\coltauthor{
	\Name{Lijie Chen} \Email{chenlj13@mails.tsinghua.edu.cn}\\
	\Name{Jian Li} \Email{lijian83@mail.tsinghua.edu.cn}\\
	\Name{Mingda Qiao} \Email{qmd14@mails.tsinghua.edu.cn}\\
	\addr Institute for Interdisciplinary Information Sciences (IIIS), Tsinghua University, Beijing, China.
}

\newcommand{\eps}{\varepsilon}
\newcommand{\ME}{\textsf{Med-Elim}}
\newcommand{\US}{\textsf{Unif-Sampl}}
\newcommand{\FT}{\textsf{Frac-Test}}
\newcommand{\EL}{\textsf{Elimination}}
\newcommand{\Ex}{\mathrm{E}}
\newcommand{\KL}{\mathrm{KL}}
\newcommand{\polylog}{\operatorname*{polylog}}
\newcommand{\bestarm}{Best-$1$-Arm}
\newcommand{\bestkarm}{Best-$k$-Arm}
\newcommand{\sign}{SIGN-$\xi$}
\newcommand{\Ebad}{\mathcal{E}^{\mathrm{bad}}}
\newcommand{\INT}{\mathcal{I}}
\newcommand{\cnt}{\mathrm{cnt}}

\newcommand{\algguess}{\textsf{Complexity-Guessing}}
\newcommand{\algent}{\textsf{Entropy-Elimination}}
\newcommand{\algnewtoy}{\textsf{Known-Complexity}}
\newcommand{\lo}{\mathrm{low}}
\newcommand{\hi}{\mathrm{high}}
\newcommand{\mi}{\mathrm{mid}}
\newcommand{\rbad}{r^{\mathrm{bad}}}

\newtheorem{Theorem}{Theorem}[section]
\newtheorem{Lemma}[Theorem]{Lemma}
\newtheorem{Definition}[Theorem]{Definition}
\newtheorem{Remark}[Theorem]{Remark}
\newtheorem{Observation}[Theorem]{Observation}
\newtheorem{Conjecture}[Theorem]{Conjecture}
\newtheorem{Corollary}[Theorem]{Corollary}
\newtheorem{Fact}[Theorem]{Fact}

\newcommand{\alg}{\mathbb{A}}
\newcommand{\algp}{\mathbb{A}'}
\newcommand{\algsign}{\mathbb{A}^{\mathsf{new}}}
\newcommand{\arm}{A}
\newcommand{\arment}{\mathsf{Ent}}
\newcommand{\chigh}{c^{\hi}}
\newcommand{\clow}{c^{\lo}}
\newcommand{\CORRECT}{$\delta$-correct}

\newcommand{\dhigh}{d^{\hi}}
\newcommand{\distr}{\mathcal{D}}
\newcommand{\dlow}{d^{\lo}}
\newcommand{\dmid}{d^{\mi}}
\newcommand{\eat}[1]{{}}

\newcommand{\event}{\mathcal{E}}

\newcommand{\Gap}[1]{\Delta_{[#1]}}

\newcommand{\Mean}[1]{\mu_{[#1]}}
\newcommand{\Nbig}{N_{\mathsf{big}}}
\newcommand{\Ncur}{N_{\mathsf{cur}}}

\newcommand{\Normal}{\mathcal{N}}
\newcommand{\Nsma}{N_{\mathsf{sma}}}
\newcommand{\ordlow}{\mathcal{L}}
\newcommand{\pr}[1]{\Pr\left[#1\right]}
\newcommand{\rmax}{r_{\max}}
\newcommand{\thetahi}{\theta^{\hi}}
\newcommand{\thetalo}{\theta^{\lo}}
\newcommand{\indexset}{U}

\begin{document}

\maketitle

\begin{abstract}
In the classical best arm identification (Best-$1$-Arm) problem, we are given $n$ stochastic bandit arms, each associated with a reward distribution with an unknown mean. Upon each play of an arm, we can get a reward sampled i.i.d. from its reward distribution. We would like to identify the arm with the largest mean with probability at least $1-\delta$, using as few samples as possible. The problem has a long history and understanding its sample complexity has attracted significant attention since the last decade. However, the optimal sample complexity of the problem is still unknown.

Recently, \cite{chen2016open} made an interesting conjecture, called gap-entropy conjecture, concerning the instance optimal sample complexity of Best-$1$-Arm. Given a Best-$1$-Arm instance $I$ (i.e., a set of arms), let $\mu_{[i]}$ denote the $i$th largest mean and $\Delta_{[i]}=\mu_{[1]}-\mu_{[i]}$ denote the corresponding gap. $H(I)=\sum_{i=2}^{n}\Delta_{[i]}^{-2}$ denotes the complexity of the instance. The gap-entropy conjecture states that for any instance $I$, $\Omega\left(H(I)\cdot\left(\ln\delta^{-1} + \mathsf{Ent}(I)\right)\right)$ is an instance lower bound, where $\mathsf{Ent}(I)$ is an entropy-like term determined by the gaps, and there is a $\delta$-correct algorithm for Best-$1$-Arm with sample complexity $O\left(H(I)\cdot\left(\ln\delta^{-1} + \mathsf{Ent}(I)\right)+\Delta_{[2]}^{-2}\ln\ln\Delta_{[2]}^{-1}\right)$. We note that $\Theta\left(\Delta_{[2]}^{-2}\ln\ln\Delta_{[2]}^{-1}\right)$ is necessary and sufficient to solve the two-arm instance with the best and second best arms. If the conjecture is true, we would have a complete understanding of the instance-wise sample complexity of Best-$1$-Arm
(up to constant factors).

In this paper, we make significant progress towards a complete resolution of the gap-entropy conjecture. For the upper bound, we provide a highly nontrivial algorithm which requires \[O\left(H(I)\cdot\left(\ln\delta^{-1} + \mathsf{Ent}(I)\right)+\Delta_{[2]}^{-2}\ln\ln\Delta_{[2]}^{-1}\mathrm{polylog}(n,\delta^{-1})\right)\] samples in expectation for any instance $I$. For the lower bound, we show that for any Gaussian Best-$1$-Arm instance with gaps of the form $2^{-k}$, any $\delta$-correct monotone algorithm requires at least \[\Omega\left(H(I)\cdot\left(\ln\delta^{-1} + \mathsf{Ent}(I)\right)\right)\] samples in expectation. Here, a monotone algorithm is one which uses no more samples (in expectation) on $I'$ than on $I$, if $I'$ is a sub-instance of $I$ obtained by removing some sub-optimal arms.
\end{abstract}

\begin{keywords}
best arm identification,
instance optimality,
gap-entropy
\end{keywords}

\section{Introduction}
	The stochastic multi-armed bandit is one of the most popular and
	well-studied models for capturing the exploration-exploitation tradeoffs
	in many application domains.
	There is a huge body of literature on numerous bandit models
	from several fields including stochastic control, statistics, operation 
	research, machine learning and theoretical computer science. 
	The basic stochastic multi-armed bandit model consists of 
	$n$ stochastic arms with unknown distributions.
	One can adaptively take samples from the arms and make decision
	depending on the objective. 
	Popular objectives include maximizing the cumulative sum of rewards, or minimizing the cumulative regret (see e.g.,~\cite{cesa2006prediction,bubeck2012regret}).
	
	In this paper, we study another classical multi-armed bandit model,
	called {\em pure exploration} model,
	where the decision-maker first performs a \emph{pure-exploration phase}
	by sampling from the arms, and then
	identifies an optimal (or nearly optimal) arm,
	which serves as the exploitation phase.
	The model is motivated by many application domains
	such as medical trials~\cite{robbins1985some,audibert2010best},
	communication network~\cite{audibert2010best}, online advertisement~\cite{chen2014combinatorial}, crowdsourcing~\cite{zhou2014optimal,cao2015top}.
	The {\em best arm identification} problem (\bestarm) is the most basic pure exploration 
	problem in stochastic multi-armed bandits. 
	The problem has a long history 
	(first formulated in \cite{bechhofer1954single})
	and has attracted significant attention 
	since the last decade~\cite{audibert2010best,even2006action,mannor2004sample,jamieson2014lil,karnin2013almost,chen2015optimal, carpentier2016tight, garivier2016optimal}. 
	Now, we formally define the problem and set up some notations.
	\begin{Definition}
		\bestarm :  We are given a set of $n$ arms $\{\arm_1,\ldots, \arm_n\}$.
		Arm $\arm_i$ has a reward distribution $\distr_i$ with an unknown mean $\mu_{i}\in [0,1]$. 
		We assume that all reward distributions are Gaussian distributions with unit variance.   
		Upon each play of $\arm_i$,
		we get a reward sampled i.i.d. from $\distr_i$. Our goal is to identify the arm with the largest mean using as few samples as possible.
		We assume here that the largest mean is strictly larger than the second largest (i.e., $\mu_{[1]}>\mu_{[2]}$)
		to ensure the uniqueness of the solution,
		where $\mu_{[i]}$ denotes
		the $i$th largest mean.
	\end{Definition}
	\begin{Remark}\label{rem:perm}
	Some previous algorithms for \bestarm{} take a sequence (instead of a set) of $n$ arms as input. In this case, we may simply assume that the algorithm randomly permutes the sequence at the beginning. Thus the algorithm will have the same behaviour on two different orderings of the same set of arms.
	\end{Remark}
	\begin{Remark}
	For the upper bound, everything proved in this paper also holds
	if the distributions are 1-sub-Gaussian, which is a standard assumption
	in the bandit literature.	On the lower bound side, we need to assume that the distributions are from some family parametrized by the means and satisfy certain properties. See Remark~\ref{Rlb}. Otherwise, it is possible to distinguish two distributions using 1 sample even if their means are very close. We cannot hope for a nontrivial lower bound in such generality.
	\end{Remark}

	The \bestarm{} problem for Gaussian arms was first formulated in \cite{bechhofer1954single}. Most early works on \bestarm{} did not analyze the sample complexity of the algorithms (they proved their algorithms are \CORRECT\ though). The early advances are summarized in the monograph~\cite{bechhofer1968sequential}.

	For the past two decades, significant research efforts have been devoted to
	understanding the optimal sample complexity of the \bestarm{} problem. 
	On the lower bound side, \cite{mannor2004sample} proved that any \CORRECT{} algorithm for \bestarm{} takes $\Omega(\sum_{i=2}^{n}\Gap{i}^{-2}\ln\delta^{-1})$ samples in expectation. In fact, their result is an instance-wise lower bound (see Definition~\ref{def:lowerbound}). \cite{kaufmann2015complexity} also provided an $\Omega(\sum_{i=2}^{n} \Gap{i}^{-2} \ln \delta^{-1})$ lower bound for \bestarm, which improved the constant factor in~\cite{mannor2004sample}. \cite{garivier2016optimal} focused on the asymptotic sample complexity of \bestarm{} as the confidence level $\delta$ approaches zero (treating the gaps as fixed), and obtained a complete resolution of this case (even for the leading constant).\footnote{In contrast, our work focus on the situation that both $\delta$ and all gaps are variables that tend to zero. In fact, if we let the gaps (i.e., $\Gap{i}$'s) tend to $0$ while maintaining $\delta$ fixed, their lower bound is not tight.}
	\cite{chen2015optimal} showed that for each $n$ there exists a \bestarm{} instance with $n$ arms that require $\Omega\left(\sum_{i=2}^{n}\Gap{i}^{-2}\ln\ln n\right)$ samples, which further refines the lower bound.

	The algorithms for \bestarm{} have also been significantly improved in the last two decades~\cite{even2002pac,gabillon2012best,kalyanakrishnan2012pac,karnin2013almost,jamieson2014lil,chen2015optimal,garivier2016optimal}.
	\cite{karnin2013almost} obtained an upper bound of 
		\[O\left(\sum\nolimits_{i=2}^{n} \Gap{i}^{-2} \left(\ln\ln\Gap{i}^{-1}+\ln\delta^{-1}\right)\right).\]
	The same upper bound was obtained by \cite{jamieson2014lil} using a UCB-type algorithm called lil'UCB. Recently, the upper bound was improved to
		\[O\left(\Gap{2}^{-2}\ln\ln\Gap{2}^{-1}+\sum\nolimits_{i=2}^{n} \Gap{i}^{-2} \left(\ln\ln\min(\Gap{i}^{-1},n)+\ln\delta^{-1}\right)\right)\]
	by~\cite{chen2015optimal}.
	There is still a gap between the best known upper and lower bound.

	To understand the sample complexity of \bestarm,
	it is important to study a special case, which we term as \sign.
	The problem can be viewed as a special case of \bestarm{}
	where there are only two arms, and we know the mean of one arm.
	\sign\ will play a very important role in our lower bound proof.
	\begin{Definition}	
		\sign: $\xi$ is a fixed constant. We are given a single arm with unknown mean $\mu \ne \xi$.
		The goal is to decide whether $\mu > \xi$ or $\mu < \xi$. 
		Here, the gap of the problem is defined to be $\Delta = |\mu - \xi|$. 
		Again, we assume that the distribution of the arm is a Gaussian distribution with unit variance.
	\end{Definition}

	In this paper, we are interested in algorithms (either for \bestarm{} or for \sign{})
	that can identify the correct answer with probability at least $1-\delta$.
	This is often called the \emph{fixed confidence}
	setting in the bandit literature.	
	\begin{Definition}
		For any $\delta\in (0,1)$, we say that an algorithm $\alg$ for \bestarm\,\,  (or \sign{}) is \CORRECT, if on any \bestarm{} (or \sign{}) instance, $\alg$ returns the correct answer with probability at least $1 - \delta$. 
	\end{Definition}

	\subsection{Almost Instance-wise Optimality Conjecture}
		It is easy to see that 
		no function $f(n, \delta)$ (only depending on $n$ and $\delta$)
		can serve as an upper bound of the sample complexity of \bestarm{}
		(with $n$ arms and confidence level $1-\delta$).
		Instead, the sample complexity depends on the gaps.
		Intuitively, the smaller the gaps are, the harder the instance is (i.e., more samples are required).
		Since the gaps completely determine an instance (for Gaussian arms with 
		unit variance, up to shifting), we use $\Gap{i}$'s as the parameters to measure the sample complexity.

		Now, we formally define the notion of instance-wise lower bounds and instance optimality.For algorithm $\alg$ and instance $I$,
		we use $T_{\alg}(I)$ to denote the expected number of samples 
		taken by $\alg$ on instance $I$.

		\begin{Definition}[Instance-wise Lower Bound]
			\label{def:lowerbound}
			
			For a \bestarm\ instance $I$ and a confidence level $\delta$, we define the instance-wise lower bound of $I$ as
			\vspace{-0.2cm} 
			\[
			\ordlow(I,\delta) := \inf_{\alg: \alg \text{ is } 
				\delta\text{-correct for \bestarm}} \,\, T_{\alg}(I).
			\]
			\vspace{-0.2cm} 
		\end{Definition}

		We say a \bestarm{} algorithm $\alg$ is instance optimal, if it is \CORRECT, and for every instance $I$,
		$T_{\alg}(I) = O(\ordlow(I,\delta))$.

		Now, we consider the \bestarm\ problem from the perspective of instance optimality.
		Unfortunately, even for the two-arm case,
		no instance optimal algorithm may exist.
		In fact, \cite{farrell1964asymptotic} showed that for any \CORRECT\
		algorithm $\alg$ for \sign, we must have 
		\[
		\liminf_{\Delta \to 0} \frac{T_{\alg}(I)}{\Delta^{-2} \ln\ln \Delta^{-1}} = \Omega(1).
		\]
		This implies that any \CORRECT\ algorithm requires
		$\Delta^{-2} \ln\ln \Delta^{-1}$ samples in the worst case.
		Hence, the upper bound of $\Delta^{-2} \ln\ln \Delta^{-1}$ for \sign\
		is generally not improvable.
		However, for a particular \sign\ instance $I_\Delta$ with gap $\Delta$,
		there is an \CORRECT\ algorithm that only needs 
		$O(\Delta^{-2} \ln\delta^{-1})$ samples for this instance,
		implying $\ordlow(I_{\Delta},\delta) = \Theta(\Delta^{-2} \ln\delta^{-1})$.
		See~\cite{chen2015optimal} for details.

		Despite the above fact, 
		\cite{chen2016open} conjectured that the two-arm case
		is the {\em only} obstruction toward an instance optimal algorithm. 
		Moreover, based on some evidence from the previous work~\cite{chen2015optimal}, 
		they provided an explicit formula and conjecture that
		$\ordlow(I,\delta)$ can be expressed by the formula. Interestingly, the formula involves an entropy term (similar entropy terms also appear in~\cite{afshani2009instance} for completely different problems). 
		In order to state Chen and Li's conjecture formally, we define the entropy term first.

		\begin{Definition}
			Given a \bestarm\ instance $I$ and $k\in\mathbb{N}$, 
			let 
			\vspace{-0.2cm}
			\[
			G_k = \{i \in [2,n] \mid  2^{-(k+1)} < \Gap{i} \le 2^{-k} \},\quad
			H_k = \sum\nolimits_{i \in G_k} \Gap{i}^{-2},
			\quad\text{ and }\quad
			p_k = H_k/\sum\nolimits_j H_j.
			\]
			We can view $\{p_k\}$
			as a discrete probability distribution.
			We define the following quantity as the \textbf{gap entropy} of instance $I$:
			\[
			\arment(I) = \sum\nolimits_{k\in\mathbb{N}:G_k \ne \emptyset} p_k \ln p_k^{-1}.%
			\footnote{
				Note that it is exactly the Shannon entropy for the distribution defined by $\{p_k\}$.
			}
			\]
		\end{Definition}
		\begin{Remark}
			We choose to partition the arms based on the powers of $2$. 
			There is nothing special about the constant $2$,
			and replacing it by any other constant only changes $\arment(I)$ by a constant factor.	
		\end{Remark}

		\begin{Conjecture}[Gap-Entropy Conjecture~\citep{chen2016open}]\label{conj:gap-entropy}
			There is an algorithm for \bestarm\ with sample complexity
			\vspace{-0.2cm}
			\[O\left(\ordlow(I,\delta) + \Gap{2}^{-2}\ln\ln\Gap{2}^{-1}\right),\]
			for any instance $I$ and $\delta < 0.01$. 
			And we say such an algorithm is almost instance-wise optimal for \bestarm. Moreover,
			\[
			\ordlow(I,\delta) = \Theta\left(\sum\nolimits_{i=2}^{n} \Gap{i}^{-2}\cdot\left(\ln\delta^{-1} + \arment(I)\right)\right).
			\]
		\end{Conjecture}

		\begin{Remark}
		As we mentioned before, the term $\Delta^{-2}\ln\ln\Delta^{-1}$ 
		is sufficient and necessary for distinguishing the best and the 
		second best arm, even though it is not an instance-optimal bound.
		The gap entropy conjecture states that modulo this additive term,
		we can obtain an instance optimal algorithm.
		Hence, the resolution of the conjecture would provide
		a complete understanding of the sample complexity of \bestarm\
		(up to constant factors).
		All the previous bounds for \bestarm{} agree with Conjecture~\ref{conj:gap-entropy},
		i.e., existing upper (lower) bounds are no smaller (larger) the conjectured bound. See~\cite{chen2016open} for details.
		\end{Remark}

	\subsection{Our Results}
		In this paper, we make significant progress toward the resolution of the gap-entropy conjecture.
		On the upper bound side, we provide an algorithm that almost matches the conjecture.
		\begin{Theorem}\label{theo:uppb-main}
			There is a \CORRECT\ algorithm for \bestarm{} with expected sample complexity
			\[
			O\left( \sum\nolimits_{i=2}^{n} \Gap{i}^{-2}\cdot\left(\ln\delta^{-1} + \arment(I)\right) + \Gap{2}^{-2} \ln\ln \Gap{2}^{-1} \cdot \polylog(n,\delta^{-1}) \right).
			\]
		\end{Theorem}
		Our algorithm matches the main term
			$\sum\nolimits_{i=2}^{n} \Gap{i}^{-2}\cdot\left(\ln\delta^{-1} + \arment(I)\right)$
		in Conjecture~\ref{conj:gap-entropy}.
		For the additive term (which is typically small), we lose a $\polylog(n,\delta^{-1})$ factor. In particular, for those instances where 
		the additive term is $\polylog(n,\delta^{-1})$ times smaller than the main term,
		our algorithm is optimal.

		On the lower bound side, despite that we are not able to completely solve the lower bound, we do obtain a rather strong bound. 
		We need to introduce some notations first. We say an instance is {\em discrete}, if the gaps of all the sub-optimal arms are of the form $2^{-k}$ for some positive integer $k$. We say an instance $I'$ is a {\em sub-instance} of an instance $I$, if $I'$ can be obtained by deleting some \emph{sub-optimal} arms from $I$.
		Formally, we have the following theorem.
		\begin{Theorem}\label{theo:lowb-main}\label{T4}
			For any discrete instance $I$, confidence level $\delta < 0.01$, and any \CORRECT\ algorithm $\alg$ for \bestarm, there exists a sub-instance $I'$ of $I$ such that
			\[
			T_{\alg}(I') \ge c \cdot \left(\sum\nolimits_{i=2}^{n} \Gap{i}^{-2}\cdot\left(\ln\delta^{-1} + \arment(I)\right)\right),
			\]
			where $c$ is a universal constant.
		\end{Theorem}

		We say an algorithm is {\em monotone}, if $T_{\alg}(I') \le T_{\alg}(I) $ for every $I'$ and $I$ such that $I'$ is a sub-instance of $I$. Then we immediately have the following corollary.

		\begin{Corollary}
		For any discrete instance $I$, and confidence level $\delta < 0.01$, for any monotone \CORRECT\ algorithm $\alg$ for \bestarm, we have that
		\[
		T_{\alg}(I) \ge c \cdot \left(\sum\nolimits_{i=2}^{n} \Gap{i}^{-2}\cdot\left(\ln\delta^{-1} + \arment(I)\right)\right),
		\]
		where $c$ is a universal constant.
		\end{Corollary}

		We remark that all previous algorithms for \bestarm\ have monotone sample complexity bounds. The above corollary also implies that if an algorithm has a monotone sample complexity bound, then the bound must be
			$\Omega\left(\sum\nolimits_{i=2}^{n} \Gap{i}^{-2}\cdot\left(\ln\delta^{-1} + \arment(I)\right)\right)$
		on all discrete instances.

\section{Related Work}
	\paragraph{\sign\ and A/B testing.}
	In the A/B testing problem, we are asked to decide which arm between the two given arms has the larger mean. A/B testing is in fact equivalent to the \sign{} problem. It is easy to reduce \sign{} to A/B testing by constructing a fictitious arm with mean $\xi$. For the other direction, given an instance of A/B testing, we may define an arm as the difference between the two given arms and the problem reduces to \sign{} where $\xi=0$. In particular, our refined lower bound for \sign{} stated in Lemma~\ref{L10} also holds for A/B testing.
	\cite{kaufmann2015complexity,garivier2016optimal} studied the limiting behavior of the sample complexity of A/B testing as the confidence level $\delta$ approaches to zero. In contrast, we focus on the case that both $\delta$ and the gap $\Delta$ tend to zero, so that the complexity term due to not knowing the gap in advance will not be dominated by the $\ln\delta^{-1}$ term.

	\paragraph{\bestkarm.}
	The \bestkarm{} problem, in which we are required to identify the $k$ arms with the $k$ largest means, is a natural extension of \bestarm{}.
	\bestkarm{} has been extensively studied in the past few years~\cite{kalyanakrishnan2010efficient,gabillon2011multi,gabillon2012best,kalyanakrishnan2012pac,bubeck2013multiple,kaufmann2013information,zhou2014optimal,kaufmann2015complexity,chen2017nearly}, and most results for \bestkarm{} are generalizations of those for \bestarm{}. As in the case of \bestarm{}, the sample complexity bounds of \bestkarm{} depend on the gap parameters of the arms, yet the gap of an arm is typically defined as the distance from its mean to either $\Mean{k+1}$ or $\Mean{k}$ (depending on whether the arm is among the best $k$ arms or not) in the context of \bestkarm{} problem. The \emph{Combinatorial Pure Exploration} problem, which further generalizes the cardinality constraint in \bestkarm{} (i.e., to choose exactly $k$ arms) to general combinatorial constraints, was also studied~\cite{chen2014combinatorial,chen2016pure,gabillon2016improved}.

	\paragraph{PAC learning.}
	The sample complexity of \bestarm{} and \bestkarm{} in the probably approximately correct (PAC) setting has also been well studied in the past two decades. For \bestarm{}, the tight \emph{worst-case} sample complexity bound was obtained by~\cite{even2002pac,mannor2004sample,even2006action}. \cite{kalyanakrishnan2010efficient,kalyanakrishnan2012pac,zhou2014optimal,cao2015top} also studied the worst case sample complexity of \bestkarm{} in the PAC setting.

\section{Preliminaries}
	Throughout the paper, $I$ denotes an instance of \bestarm{} (i.e., $I$ is a set of arms).
	The arm with the largest mean in $I$ is called the optimal arm,
	while all other arms are \emph{sub-optimal}.
	We assume that every instance has a unique optimal arm.
	$A_i$ denotes the arm in $I$ with the $i$-th largest mean, unless stated otherwise.
	The mean of an arm $A$ is denoted by $\mu_{A}$,
	and we use $\Mean{i}$ as a shorthand notation for $\mu_{A_i}$
	(i.e., the $i$-th largest mean in an instance).
	Define $\Delta_{A}=\Mean{1}-\mu_A$ as the gap of arm $A$,
	and let $\Gap{i}=\Delta_{A_i}$ denote the gap of arm $A_i$. 
	We assume that $\Gap{2}>0$ to ensure the optimal arm is unique.

	We partition the sub-optimal arms into different groups based on their gaps.
	For each $k\in\mathbb{N}$, group $G_k$ is defined as
		$\left\{A_i:\Gap{i}\in\left(2^{-(k+1)},2^{-k}\right]\right\}$.
	For brevity, let $G_{\ge k}$ and $G_{\le k}$ denoted
	$\bigcup_{i=k}^{\infty}G_i$ and $\bigcup_{i=1}^{k}G_i$ respectively.
	The \emph{complexity} of arm $A_i$ is defined as $\Gap{i}^{-2}$,
	while the complexity of instance $I$ is denoted by $H(I)=\sum_{i=2}^{n}\Gap{i}^{-2}$
	(or simply $H$, if the instance is clear from the context).
	Moreover, $H_k=\sum_{A\in G_k}\Delta_{A}^{-2}$ denotes
	the total complexity of the arms in group $G_k$.
	$(H_k)_{k=1}^{\infty}$ naturally defines a probability distribution on $\mathbb{N}$,
	where the probability of $k$ is given by
		$p_k = H_k / H$. 
	The gap-entropy of the instance $I$ is then denoted by
		\[\arment(I)=\sum_{k}p_k\ln p_k^{-1}.\]
	Here and in the following, we adopt the convention that $0\ln 0^{-1}=0$.

\section{A Sketch of the Lower Bound}

	\subsection{A Comparison with Previous Lower Bound Techniques}
		We briefly discuss the novelty of our new lower bound technique,
		and argue why the previous techniques are not sufficient to obtain our result.
		To obtain a lower bound on the sample complexity of \bestarm, 
		all the previous work~\cite{mannor2004sample,chen2014combinatorial,kaufmann2015complexity,garivier2016optimal} are based on
		creating two similar instances with different answers,
		and then applying the \emph{change of distribution} method (originally developed in~\cite{kaufmann2015complexity}) to argue that a certain number of samples are necessary to 
		distinguish such two instances. The idea was further refined by~\cite{garivier2016optimal}. They formulated a max-min game between the algorithm and some instances (with different answers than the given instance) created by an adversary. The value of the game at equilibrium would be a lower bound of the samples one requires to distinguish the current instance and several worst adversary instances. However, we notice that even in the two-arm case, 
		one cannot prove the $\Omega(\Delta^{-2}\ln\ln\Delta^{-1})$ lower bound by
		considering only one max-min game to distinguish the current instance from other instance.
		Roughly speaking, the $\ln\ln\Delta^{-1}$ factor is due to not knowing the actual 
		gap $\Delta$, and any lower bound that can bring out the $\ln\ln\Delta^{-1}$ factor
		should reflect the union bound paid for the uncertainty of the instance.
		In fact, for the \bestarm\ problem with $n$ arms, the gap entropy $\arment(I)$ term
		exists for a similar reason (not knowing the gaps).
		Hence, any lower bound proof for \bestarm\ that can bring out the $\arment(I)$ term necessarily has to consider the uncertainty of current instance as well
		(in fact, the random permutation of all arms is the kind of uncertainty we need for the new lower bound).
		In our actual lower bound proof, we first obtain a very tight understanding 
		of the \sign\ problem (Lemma~\ref{L10}).%
		\footnote{Farrell's lower bound~\cite{farrell1964asymptotic} is not sufficient for
		our purpose.}
		Then, we provide an elegant reduction from \sign\ to \bestarm,
		by embedding the \sign\ problem to a collection of \bestarm\ instances.

	\subsection{Proof of Theorem~\ref{theo:lowb-main}}
		Following the approach in~\cite{chen2015optimal}, we establish the lower bound by a reduction from \sign{} to discrete \bestarm{} instances, together with a more refined lower bound for \sign{} stated in the following lemma.
		\begin{Lemma}\label{L10}
			Suppose $\delta\in(0,0.04)$, $m\in\mathbb{N}$ and
			$\mathbb{A}$ is a $\delta$-correct algorithm for \sign{}.
			$P$ is a probability distribution on $\{2^{-1},2^{-2},\ldots,2^{-m}\}$
			defined by $P(2^{-k})=p_k$.
			$\arment(P)$ denotes the Shannon entropy of distribution $P$.
			Let $T_{\alg}(\mu)$ denote the expected number of samples taken by $\alg$
			when it runs on an arm with distribution $\Normal(\mu,1)$ and $\xi=0$.
			Define $\alpha_k=T_{\alg}(2^{-k})/4^k$. Then,
				\[\sum_{k=1}^{m}p_k\alpha_k=\Omega(\arment(P)+\ln\delta^{-1}).\]
		\end{Lemma}

		It is well known that to distinguish the normal distribution $\Normal(2^{-k},1)$ from $\Normal(-2^{-k},1)$, $\Omega(4^k)$ samples are required. Thus, $\alpha_k=T_{\alg}(2^{-k})/4^k$ denotes the ratio between the expected number of samples taken by $\alg$ and the corresponding lower bound, which measures the ``loss'' due to not knowing the gap in advance. Then Lemma~\ref{L10} can be interpreted as follows: when the gap is drawn from a distribution $P$, the \emph{expected loss} is lower bounded by the sum of the entropy of $P$ and $\ln\delta^{-1}$. We defer the proof of Lemma~\ref{L10} to Appendix~\ref{app:lb}.

		Now we prove Theorem~\ref{T4} by applying Lemma~\ref{L10} and an elegant reduction from \sign{} to \bestarm{}.
		\begin{proof}[Proof of Theorem~\ref{T4}]
		Let $c_0$ be the hidden constant in the big-$\Omega$ in Lemma~\ref{L10}, i.e.,
			$$\sum_{k=1}^{m}p_k\alpha_k\ge c_0\cdot(\arment(P)+\ln\delta^{-1})\text{.}$$
		We claim that Theorem~\ref{T4} holds for constant $c = 0.25c_0$.

		Suppose towards a contradiction that $\alg$ is
		a $\delta$-correct (for some $\delta<0.01$) algorithm for \bestarm{}
		and $I=\{A_1,A_2,\ldots,A_n\}$ is a discrete instance,
		while for all sub-instance $I'$ of $I$,
			$$T_{\alg}(I')<c\cdot H(I)(\arment(I)+\ln\delta^{-1})\text{.}$$
		Recall that $H(I)$ and $\arment(I)$ denote the complexity and entropy of instance $I$, respectively.

		\paragraph{Construct a distribution of \sign{} instances.}
		Let $n_k$ be the number of arms in $I$ with gap $2^{-k}$,
		and $m$ be the greatest integer such that $n_m>0$.
		Since $I$ is discrete, the complexity of instance $I$ is given by
			$$H(I)=\sum_{k=1}^{m}4^kn_k\text{.}$$
		Let $p_k=4^kn_k/H(I)$.
		Then $(p_k)_{k=1}^{m}$ defines a distribution $P$ on
			$\{2^{-1},2^{-2},\ldots,2^{-m}\}$.
		Moreover, the Shannon entropy of distribution $P$ is
		exactly the entropy of instance $I$, i.e., $\arment(P)=\arment(I)$.
		Our goal is to construct an algorithm for \sign{}
		that violates Lemma~\ref{L10} on distribution $P$.

		\paragraph{A family of sub-instances of $I$.}
		Let
			$\indexset=\{k\in[m]:n_k>0\}$
		be the set of ``types'' of arms
		that are present in $I$.
		We consider the following family of instances obtained from $I$.
		For $S\subseteq\indexset$,
		define $I_S$ as the instance obtained from $I$
		by removing exactly one arm of gap $2^{-k}$ for each $k\in S$.
		Note that $I_S$ is a sub-instance of $I$.

		Let $\overline{S}$ denote $\indexset\setminus S$,
		the complement of set $S$ relative to $\indexset$.
		For $S\subseteq\indexset$ and $k\in\overline{S}$,
		let $\tau_k^{S}$ denote the expected number of samples
		taken on all the $n_k$ arms with gap $2^{-k}$
		when $\alg$ runs on $I_S$.
		Define $\alpha_k^{S}=4^{-k}\tau_k^{S}/n_k$.
		We note that $4^k\alpha_k^{S}$ is the expected number of samples
		taken on \emph{every} arm with gap $2^{-k}$ in instance $I_S$.%
		\footnote{
			Recall that a \bestarm{} algorithm is defined on a \emph{set} of arms,
			so the arms with identical means in the instance cannot be distinguished by $\alg$.
			See Remark~\ref{rem:perm} for details.
		}

		We have the following inequality:
		\begin{equation}\label{eq6}
		\sum_{S\subseteq \indexset}\sum_{k\in\overline{S}}4^kn_k\alpha_k^{S}
		=\sum_{S\subseteq \indexset}\sum_{k\in\overline{S}}\tau_k^{S}
		\le\sum_{S\subseteq \indexset}T_{\alg}(I_S)
		<c\cdot2^{|\indexset|}H(I)(\arment(I)+\ln\delta^{-1})\text{.}
		\end{equation}
		The second step holds because the lefthand side only counts part of the samples taken by $\alg$. The last step follows from our assumption and the fact that $I_S$ is a sub-instance of $I$.

		\paragraph{Construct algorithm $\algsign$ from $\alg$.}
		Now we define an algorithm $\algsign$ for \sign{} with $\xi=0$. Given an arm $A$, we first choose a set $S\subseteq\indexset$ uniformly at random from all subsets of $\indexset$. Recall that $\Mean{1}$ denotes the mean of the optimal arm in $I$. $\algsign$ runs the following four algorithms $\alg_1$ through $\alg_4$ in parallel:
		\begin{enumerate}
		\item Algorithm $\alg_1$ simulates $\alg$ on $I_S\cup\{\Mean{1}+A\}$.
		\item Algorithm $\alg_2$ simulates $\alg$ on $I_{\overline{S}}\cup\{\Mean{1}+A\}$. 
		\item Algorithm $\alg_3$ simulates $\alg$ on $I_S\cup\{\Mean{1}-A\}$. 
		\item Algorithm $\alg_4$ simulates $\alg$ on $I_{\overline{S}}\cup\{\Mean{1}-A\}$. 
		\end{enumerate}
		More precisely, when one of the four algorithms requires a new sample from $\Mean{1}+A$ (or $\Mean{1}-A$), we draw a sample $x$ from arm $A$, feed $\Mean{1}+x$ to $\alg_1$ and $\alg_2$, and then feed $\Mean{1}-x$ to $\alg_3$ and $\alg_4$. Note that the samples taken by the four algorithms are the same up to negation and shifting.

		$\algsign$ terminates as soon as one of the four algorithms terminates. If one of $\alg_1$ and $\alg_2$ identifies $\Mean{1}+A$ as the optimal arm, or one of $\alg_3$ and $\alg_4$ identifies an arm other than $\Mean{1}-A$ as the optimal arm, $\algsign$ outputs ``$\mu_A>0$''; otherwise it outputs ``$\mu_A<0$''.

		Clearly, $\algsign$ is correct if all of $\alg_1$ through $\alg_4$ are correct, which happens with probability at least $1-4\delta$.
		Note that since $4\delta<0.04$, the condition of Lemma~\ref{L10} is satisfied.

		\paragraph{Upper bound the sample complexity of $\algsign$.}
		The crucial observation is that when $\mu_{A}=-2^{-k}$ and $k\in S$,
		$\alg_1$ effectively simulates the execution of $\alg$ on $I_{S\setminus\{k\}}$.
		In fact, since all arms are Gaussian distributions with unit variance,
		the arm $\Mean{1}+A$ is the same as an arm with gap $2^{-k}$ in the original \bestarm{} instance.
		Recall that the number of samples taken on each of the arms with gap $2^{-k}$
		in instance $I_{S\setminus\{k\}}$ is
			$4^k\alpha_k^{S\setminus\{k\}}$.
		Therefore, the expected number of samples taken on $A$
		is upper bounded by
			$4^k\alpha_k^{S\setminus\{k\}}$.%
		\footnote{
				Recall that if $\alg_1$ terminates after taking $T$ samples from $\Mean{1}+A$,
				the number of samples taken by $\algsign$ on $A$ is also $T$ (rather than $4T$).
			}
		Likewise, when $\mu_{A}=-2^{-k}$ and $k\in\overline{S}$,
		$\alg_2$ is equivalent to the execution of $\alg$ on $I_{\overline{S}\setminus\{k\}}$,
		and thus the expected number of samples on $A$ is less than or equal to
			$4^k\alpha_k^{\overline{S}\setminus\{k\}}$.
		Analogous claims hold for the case $\mu_A=+2^{-k}$
		and algorithms $\alg_3$ and $\alg_4$ as well.

		It remains to compute the expected loss of $\algsign$ on distribution $P$ and derive a contradiction to Lemma~\ref{L10}. It follows from a simple calculation that
			\begin{align*}
			\sum_{k=1}^{m}p_k\alpha_k
			&\le\sum_{k\in\indexset}p_k\cdot\frac{1}{2^{|\indexset|}}\left(\sum_{S\subseteq\indexset: k\in S}\alpha_k^{S\setminus\{k\}}+\sum_{S\subseteq\indexset: k\in\overline{S}}\alpha_k^{\overline{S}\setminus\{k\}}\right)\\
			&=	\frac{1}{2^{|\indexset|-1}}\sum_{k\in\indexset}\sum_{S\subseteq\indexset: k\in S}p_k\alpha_k^{S\setminus\{k\}}\\
			&=	\frac{1}{2^{|\indexset|-1}}\sum_{S\subseteq \indexset}\sum_{k\in\overline{S}}\frac{4^kn_k}{H(I)}\cdot\alpha_k^{S}\\
			&\le\frac{2^{|\indexset|}}{2^{|\indexset|-1}}\cdot c\cdot(\arment(I)+\ln\delta^{-1})
			<	c_0(\arment(P)+\ln(4\delta)^{-1}).
			\end{align*}
		The first step follows from our discussion on algorithm $\algsign$.
		The third step renames the variables and rearranges the summation.
		The last line applies \eqref{eq6}.
		This leads to a contradiction to Lemma~\ref{L10} and thus finishes the proof.
		\end{proof}

\section{Warmup: \bestarm{} with Known Complexity}\label{sec:toy}
	To illustrate the idea of our algorithm for \bestarm{},
	we consider the following simplified yet still non-trivial version of \bestarm{}:
	the complexity of the instance,
		$H(I)=\sum_{i=2}^n\Gap{i}^{-2}$,
	is given, yet the means of the arms are still unknown.

	\subsection{Building Blocks}
		We introduce some subroutines that are used throughout our algorithm.

		\paragraph{Uniform sampling.} The first building block is a uniform sampling procedure, $\US(S,\eps,\delta)$, which takes $2\eps^{-2}\ln(2/\delta)$ samples from 
		each arm in set $S$. Let $\hat\mu_{A}$ be the empirical mean of arm $A$
		(i.e., the average of all sampled values from $A$).
		It obtains an $\eps$-approximation of the mean of each arm with probability $1-\delta$.
		The following fact directly follows by the Chernoff bound.

		\begin{Fact}\label{F1}
		$\US(S,\eps,\delta)$ takes $O(|S|\eps^{-2}\ln\delta^{-1})$ samples. For each arm $A\in S$, we have
		$$\Pr\left[|\hat\mu_A-\mu_A|\le\eps\right]\ge1-\delta\text{.}$$
		\end{Fact}

		We say that a call to procedure $\US(S,\eps,\delta)$ returns correctly, if $|\hat\mu_A-\mu_A|\le\eps$ holds for every arm $A\in S$.
		Fact~\ref{F1} implies that when $|S|=1$, the probability of returning correctly is at least $1-\delta$.

		\paragraph{Median elimination.} \cite{even2002pac} introduced the Median Elimination algorithm for the PAC version of \bestarm. $\ME(S,\eps,\delta)$ returns an arm in $S$ with mean at most $\eps$ away from the largest mean. Let $\Mean{1}(S)$ denote the largest mean among all arms in $S$. The performance guarantees of $\ME$ is formally stated in the next fact.

		\begin{Fact}\label{F2}
		$\ME(S,\eps,\delta)$ takes $O(|S|\eps^{-2}\ln\delta^{-1})$ samples. Let $A$ be the arm returned by $\ME$. Then
		$$\Pr[\mu_A\ge\Mean{1}(S)-\eps]\ge1-\delta\text{.}$$
		\end{Fact}

		We say that $\ME(S,\eps,\delta)$ returns correctly, if it holds that $\mu_A\ge\Mean{1}(S)-\eps$.

		\paragraph{Fraction test.} Procedure $\FT(S,\clow,c^{\hi},\theta^{\lo},\theta^{\hi},\delta)$ decides whether a sufficiently large fraction (compared to thresholds $\theta^{\lo}$ and $\theta^{\hi}$) of arms in $S$ have small means (compared to thresholds $\clow$ and $c^{\hi}$). The procedure randomly samples a certain number of arms from $S$ and estimates their means using $\US$. Then it compares the fraction of arms with small means to the thresholds and returns an answer accordingly. The detailed implementation of $\FT$ is relegated to Appendix~\ref{app:alg}, where we also prove the following fact.

		\begin{Fact}\label{F3}
		$\FT(S,\clow,c^{\hi},\theta^{\lo},\theta^{\hi},\delta)$ takes
			$O\left((\eps^{-2}\ln\delta^{-1})\cdot(\Delta^{-2}\ln\Delta^{-1})\right)$
		samples, where $\eps=c^{\hi}-\clow$ and $\Delta=\theta^{\hi}-\theta^{\lo}$. With probability $1-\delta$, the following two claims hold simultaneously:
		\begin{itemize}
		\item If $\FT$ returns True, $|\{A\in S:\mu_A<c^{\hi}\}|>\theta^{\lo}|S|$.
		\item If $\FT$ returns False, $|\{A\in S:\mu_A<\clow\}|<\theta^{\hi}|S|$.
		\end{itemize}
		\end{Fact}

		We say that a call to procedure $\FT$ returns correctly, if both the two claims above hold; otherwise the call fails.

		\paragraph{Elimination.} Finally, procedure $\EL(S,\dlow,\dhigh,\delta)$ eliminates the arms with means smaller than threshold $\dlow$ from $S$. More precisely, the procedure guarantees that at most a $0.1$ fraction of arms in the result have means smaller than $\dlow$. On the other hand, for each arm with mean greater than $\dhigh$, with high probability it is not eliminated. We postpone the pseudocode of procedure $\EL$ and the proof of the following fact to Appendix~\ref{app:alg}.

		\begin{Fact}\label{F4}
		$\EL(S,\dlow,\dhigh,\delta)$ takes $O(|S|\eps^{-2}\ln\delta^{-1})$ samples in expectation, where $\eps=\dhigh-\dlow$. Let $S'$ denote the set returned by $\EL(S,\dlow,\dhigh,\delta)$. Then with probability at least $1 - \delta / 2$,
		$$|\{A\in S':\mu_A<\dlow\}|\le0.1|S'|\text{.}$$
		Moreover, for each arm $A\in S$ with $\mu_A\ge \dhigh$, we have
		$$\pr{A\in S'}\ge1-\delta/2\text{.}$$
		\end{Fact}

		We say that a call to $\EL$ returns correctly if both $|\{A\in S':\mu_A<\dlow\}|\le0.1|S'|$ and $A_1(S)\in S'$ hold; otherwise the call fails. Here $A_1(S)$ denotes the arm with the largest mean in set $S$. Fact~\ref{F4} directly implies that procedure $\EL$ returns correctly with probability at least $1-\delta$.

	\subsection{Algorithm}
		Now we present our algorithm for the special case that the complexity of the instance
		is known in advance.
		The $\algnewtoy$ algorithm takes as its input a \bestarm{} instance $I$,
		the complexity $H$ of the instance, as well as a confidence level $\delta$.
		The algorithm proceeds in rounds,
		and maintains a sequence $\{S_r\}$ of arm sets,
		each of which denotes the set of arms that are still considered as candidate answers
		at the beginning of round $r$.

		Roughly speaking,
		the algorithm eliminates the arms with $\Omega(\eps_r)$ gaps at the $r$-th round,
		if they constitute a large fraction of the remaining arms.
		Here $\eps_r = 2^{-r}$ is the accuracy parameter that we use in round $r$.
		To this end, \algnewtoy{} first calls procedures $\ME$ and $\US$
		to obtain $\hat\mu_{a_r}$,
		which is an estimation of the largest mean among all arms in $S_r$
		up to an $O(\eps_r)$ error.
		After that, $\FT$ is called to determine whether a large proportion of arms in $S_r$ have $\Omega(\eps_r)$ gaps. If so, $\FT$ returns True, and then \algnewtoy{} calls the $\EL$ procedure with carefully chosen parameters
		to remove suboptimal arms from $S_r$.

		\begin{algorithm2e}
			\caption{$\algnewtoy(I, H, \delta)$}
			\KwIn{Instance $I$ with complexity $H$ and risk $\delta$.}
			\KwOut{The best arm.}
			$S_1\leftarrow I$;
			$\hat H \leftarrow 4096H$;

			\For{\upshape $r=1$ to $\infty$} {
				\lIf{$|S_r|=1$} {
					\textbf{return} the only arm in $S_r$;
				}
				
				$\eps_r\leftarrow2^{-r}$;
				$\delta_r\leftarrow\delta/(10r^2)$;

				$a_r\leftarrow\ME(S_r,0.125\eps_r,0.01)$;

				$\hat\mu_{a_r}\leftarrow\US(\{a_r\},0.125\eps_r,\delta_r)$;

				\uIf{\upshape$\FT(S_r,\hat\mu_{a_r}-1.75\eps_r,\hat\mu_{a_r}-1.125\eps_r, 0.3, 0.5, \delta_r)$} {
					$\delta'_r\leftarrow \left(|S_r|\eps_r^{-2} / \hat H\right)\delta$;

					$S_{r+1}\leftarrow\EL(S_r,\hat\mu_{a_r}-0.75\eps_r,\hat\mu_{a_r}-0.625\eps_r,\delta'_r)$;
				} \uElse {
					$S_{r+1}\leftarrow S_r$;
				}
			}
		\end{algorithm2e}

		The following two lemmas imply that there is a \CORRECT{} algorithm for \bestarm{} that matches the instance-wise lower bound up to an $O\left(\Gap{2}^{-2}\ln\ln\Gap{2}^{-1}\right)$ additive term.%
		\footnote{
			Lemma~\ref{lem:newtoysample} only bounds the number of samples
			conditioning on an event that happens with probability $1 - \delta$,
			so the algorithm may take arbitrarily many samples when the event does not occur.
			However, \algnewtoy{} can be transformed to a \CORRECT{} algorithm with the same
			(unconditional) sample complexity bound, using the ``parallel simulation'' technique in the proof of Theorem~\ref{theo:uppb-main} in Appendix~\ref{app:samp}.
		}
		\begin{Lemma}\label{lem:newtoycorrect}
			For any \bestarm{} instance $I$ and $\delta\in(0, 0.01)$, $\algnewtoy(I, H(I), \delta)$ returns the optimal arm in $I$ with probability at least $1 - \delta$.
		\end{Lemma}
		\begin{Lemma}\label{lem:newtoysample}
			For any \bestarm{} instance $I$ and $\delta\in(0, 0.01)$, conditioning on an event that happens with probability $1 - \delta$, $\algnewtoy(I, H(I), \delta)$ takes
				$$O\left(H(I)\cdot(\ln\delta^{-1} + \arment(I)) + \Gap{2}^{-2}\ln\ln\Gap{2}^{-1}\right)$$
			samples in expectation.
		\end{Lemma}

	\subsection{Observations}
		We state a few key observations on \algnewtoy,
		which will be used throughout the analysis.
		The proofs are exactly identical to those of Observations \ref{O3}~through~\ref{O6} in Appendix~\ref{app:alg}.
		The following observation bounds the value of $\hat\mu_{a_r}$ at round $r$,
		assuming the correctness of $\US$ and $\ME$.
		\begin{Observation}\label{obs:toymuhat}
			If $\US$ returns correctly at round $r$,
				$\hat\mu_{a_r}\le\Mean{1}(S_r)+0.125\eps_r$.
			Here $\Mean{1}(S_r)$ denotes the largest mean of arms in $S_r$.
			If both $\US$ and $\ME$ return correctly,
				$\hat\mu_{a_r}\ge\Mean{1}(S_r)-0.25\eps_r$.
		\end{Observation}
		The following two observations bound the thresholds used in $\FT$ and $\EL$ by applying Observation~\ref{obs:toymuhat}.
		\begin{Observation}\label{obs:toyFT}
			At round $r$, let
				$\clow_r=\hat\mu_{a_r}-1.75\eps_r$
			and
				$\chigh_r=\hat\mu_{a_r}-1.125\eps_r$
			denote the two thresholds used in $\FT$.
			If $\US$ returns correctly,
				$\chigh_r\le\Mean{1}(S_r)-\eps_r$.
			If both $\ME$ and $\US$ return correctly,
				$\clow_r\ge\Mean{1}(S_r)-2\eps_r$.
		\end{Observation}
		\begin{Observation}\label{obs:toyEL}
			Let
				$\dlow_r=\hat\mu_{a_r}-0.75\eps_r$
			and
				$\dhigh_r=\hat\mu_{a_r}-0.625\eps_r$
			denote the two thresholds used in $\EL$.
			If $\US$ returns correctly,
				$\dhigh_r\le\Mean{1}(S_r)-0.5\eps_r$.
			If both $\ME$ and $\US$ return correctly,
				$\dlow_r\ge\Mean{1}(S_r)-\eps_r$.
		\end{Observation}

	\subsection{Correctness}
		We define $\event$ as the event that
		all calls to procedures $\US$, $\FT$, and $\EL$ return correctly.
		We will prove in the following that
		\algnewtoy{} returns the correct answer with probability $1$
		conditioning on $\event$,
		and $\pr{\event}\ge1-\delta$.
		Note that Lemma~\ref{lem:newtoycorrect} directly follows from these two claims.

		\paragraph{Event $\event$ implies correctness.}
		It suffices to show that conditioning on $\event$,
		\algnewtoy{} never removes the best arm,
		and the algorithm eventually terminates.
		Suppose that $A_1\in S_r$.
		Observation~\ref{obs:toyEL} guarantees that at round $r$,
		the upper threshold used by $\EL$ is smaller than or equal to
			$\Mean{1}(S_r)-0.5\eps_r < \Mean{1}$.
		By Fact~\ref{F4},
		the correctness of $\EL$ guarantees that $A_1\in S_{r+1}$.

		It remains to prove that \algnewtoy{} terminates conditioning on $\event$.
		Define $\rmax \coloneqq\max_{G_r\ne\emptyset}r$.
		Suppose $r^*$ is the smallest integer greater than $\rmax$
		such that $\ME$ returns correctly at round $r^*$.\footnote{
			$\ME$ returns correctly with probability at least $0.99$ in each round,
			so $r^*$ is well-defined with probability $1$.
		}
		By Observation~\ref{obs:toyEL},
		the lower threshold in $\EL$ is greater than or equal to $\Mean{1}-\eps_{r^*}$.
		The correctness of $\EL$ implies that
			$$|S_{r^*+1}|-1=|S_{r^*+1}\cap G_{\le\rmax}|
			\le|S_{r^*+1}\cap G_{<r^*}|
			=|\{A\in S_{r^*+1}:\mu_A<\Mean{1}-\eps_{r^*}\}|
			<0.1|S_{r^*+1}|\text{.}$$
		It follows that $|S_{r^*+1}|=1$.
		Therefore, the algorithm terminates either before or at round $r^*+1$.

		\paragraph{$\event$ happens with high probability.}
		We first note that at round $r$,
		the probability that either $\US$ or $\FT$ fails (i.e., returns incorrectly)
		is at most $2\delta_r$.
		By a union bound,
		the probability that at least one call to $\US$ or $\FT$ returns incorrectly
		is upper bounded by
			$$\sum_{r=1}^{\infty}2\delta_r
			=\sum_{r=1}^{\infty}\frac{\delta}{5r^2}
			<\delta/2\text{.}$$

		It remains to bound the probability that $\EL$ fails at some round,
		yet procedures $\US$ and $\FT$ are always correct.
		Define $P(r, S_r)$ as the probability that,
		given the value of $S_r$ at the beginning of round $r$,
		at least one call to $\EL$ returns incorrectly in round $r$ or later,
		yet $\US$ and $\FT$ always return correctly.
		We prove by induction that for any $S_r$ that contains the optimal arm $A_1$,
			\begin{equation}\label{eq:toyineq}
				P(r, S_r)\le \frac{\delta}{\hat H}\left(128C(r, S_r)+16M(r, S_r)\eps_r^{-2}\right)\text{,}
			\end{equation}
		where
			$M(r, S_r)\coloneqq |S_r\cap G_{\le r-2}|$
		and
			$$C(r, S_r)\coloneqq \sum_{i=r-1}^{\infty}|S_r\cap G_i|\sum_{j=r}^{i+1}\eps_j^{-2}+\sum_{i=r}^{\rmax+1}\eps_i^{-2}\text{.}$$
		The details of the induction are postponed to Appendix~\ref{app:toymiss}.

		Observe that $M(1, I) = 0$ and
			\begin{equation*}\begin{split}
			C(1, I)
			&=\sum_{i=0}^{\infty}|S_r\cap G_i|\sum_{j=1}^{i+1}4^j+\sum_{i=1}^{\rmax+1}4^i\\
			&\le\frac{16}{3}\left(\sum_{i=0}^{\infty}|S_r\cap G_i|4^i+4^{\rmax}\right)\\
			&\le\frac{16}{3}\left(\sum_{i=0}^{\infty}\sum_{A\in S_r\cap G_i}\Delta_{A}^{-2}+\Gap{2}^{-2}\right)
			\le\frac{32}{3}H(I)\text{.}
			\end{split}\end{equation*}
		Therefore we conclude that
			\begin{equation*}\begin{split}
			\pr{\event}
			&\ge 1 - P(1, S_1) - \frac{\delta}{2}\\
			&\ge 1 - \frac{\delta}{\hat H}\left(128C(1, I)+16M(1, I)\eps_1^{-2}\right) - \frac{\delta}{2}\\
			&\ge 1 - 128\cdot\frac{\delta}{4096H}\cdot\frac{32H}{3} - \frac{\delta}{2}
			\ge 1 - \delta\text{,}
			\end{split}\end{equation*}
		which completes the proof of correctness.
		Here the first step applies a union bound. The second step follows from inequality \eqref{eq:toyineq}, and the third step plugs in $C(1,I)\le 32H(I)/3$ and $\hat H = 4096H$.

	\subsection{Sample Complexity}
		As in the proof of Lemma~\ref{lem:newtoycorrect}, we define $\event$ as the event that all calls to procedures $\US$, $\FT$, and $\EL$ return correctly. We prove that $\algnewtoy{}$ takes
			\[O\left(H(I)(\ln\delta^{-1}+\arment(I))+\Gap{2}^{-2}\ln\ln\Gap{2}^{-1}\right)\]
		samples in expectation conditioning on $\event$.

		\paragraph{Samples taken by $\US$ and $\FT$.}
		By Facts \ref{F1}~and~\ref{F3}, procedures $\US$ and $\FT$ take
			$O\left(\eps_r^{-2}\ln\delta_r^{-1}\right)
			=O\left(\eps_r^{-2}(\ln\delta^{-1}+\ln r)\right)$
		samples in total at round $r$.

		In the proof of correctness, we showed that conditioning on $\event$, the algorithm does not terminate before or at round $k$ (for $k\ge\rmax+1$) implies that $\ME$ fails between round $\rmax+1$ and round $k-1$, which happens with probability at most $0.01^{k-\rmax-1}$. Thus for $k\ge\rmax+1$, the expected number of samples taken by $\US$ and $\FT$ at round $k$ is upper bounded by
			$$O\left(0.01^{k-\rmax-1}\cdot\eps_k^{-2}(\ln\delta^{-1}+\ln k)\right)\text{.}$$
		Summing over all $k=1,2,\ldots$ yields the following upper bound:
		\begin{equation*}\begin{split}
			&\sum_{k=1}^{\rmax}\eps_k^{-2}(\ln\delta^{-1}+\ln k) + \sum_{k=\rmax+1}^{\infty}0.01^{k-\rmax-1}\cdot\eps_k^{-2}(\ln\delta^{-1}+\ln k)\\
		=	&O\left(4^{\rmax}(\ln\delta^{-1}+\ln\rmax)\right)
		=	O\left(\Gap{2}^{-2}\left(\ln\delta^{-1}+\ln\ln\Gap{2}^{-1}\right)\right)\text{.}
		\end{split}\end{equation*}
		Here the first step holds since the first summation is dominated by the last term ($k = \rmax$), while the second one is dominated by the first term ($k = \rmax + 1$). The second step follows from the observation that
			$\rmax = \max_{G_r\ne\emptyset}r
			= \left\lfloor\log_2\Gap{2}^{-1}\right\rfloor\text{.}$

		\paragraph{Samples taken by $\ME$ and $\EL$.} 
		By Facts \ref{F2}~and~\ref{F4}, $\ME$ and $\EL$ (if called) take
			$$O(|S_r|\eps_r^{-2})+O(|S_r|\eps_r^{-2}\ln(1/\delta'_r))
			=O\left(|S_r|\eps_r^{-2}\left(\ln\delta^{-1}+\ln\frac{H}{|S_r|\eps_r^{-2}}\right)\right)$$
		samples in total at round $r$.

		We upper bound the number of samples by a charging argument. For each round $i$, define $r_i$ as the largest integer $r$ such that $|G_{\ge r}|\ge0.5|S_i|$.\footnote{Note that $|G_{\ge 0}|=n-1\ge0.5|S_i|$ and $|G_{\ge r}|=0 < 0.5|S_i|$ for sufficiently large $r$, so $r_i$ is well-defined.} Then we define
			$$
			T_{i,j} = \begin{cases}
			0, & j < r_i,\\
			\eps_i^{-2}\left(\ln\delta^{-1}+\ln\dfrac{H}{|G_j|\eps_i^{-2}}\right), & j\ge r_i
			\end{cases}
			$$
		as the number of samples that each arm in $G_j$ is charged at round $i$.

		We prove in Appendix~\ref{app:toymiss} that for any $i$, $\sum_{j}|G_j|T_{i,j}$ is an upper bound on the number of samples taken by $\ME$ and $\EL$ at the $i$-th round. Moreover, the expected number of samples that each arm in group $G_j$ is charged is upper bounded by
			$$\sum_{i}\Ex[T_{i,j}]
			=O\left(\eps_j^{-2}\left(\ln\delta^{-1}+\ln\frac{H}{|G_j|\eps_j^{-2}}\right)\right)\text{.}$$
		Note that $H_k = \sum_{A\in G_k}\Delta_{A}^{-2} = \Theta(|G_k|\eps_k^{-2})$. Therefore, $\ME$ and $\EL$ take
			\begin{equation*}\begin{split}
			O\left(\sum_{i,j}|G_j|\Ex[T_{i,j}]\right)
			&=O\left(\sum_{j}|G_j|\eps_j^{-2}\left(\ln\delta^{-1}+\ln\frac{H}{|G_j|\eps_j^{-2}}\right)\right)\\
			&=O\left(\sum_{j}H_j\left(\ln\delta^{-1}+\ln\frac{H}{H_j}\right)\right)\\
			&=O\left(H(I)\left(\ln\delta^{-1}+\arment(I)\right)\right)
			\end{split}\end{equation*}
		samples in expectation conditioning on $\event$.

		In total, algorithm \algnewtoy{} takes
			\begin{equation*}\begin{split}
				&O\left(\Gap{2}^{-2}\left(\ln\delta^{-1}+\ln\ln\Gap{2}^{-1}\right)\right)
			+ O\left(H(I)\left(\ln\delta^{-1}+\arment(I)\right)\right)\\
			= 	&O\left(H(I)\left(\ln\delta^{-1}+\arment(I)\right)+\Gap{2}^{-2}\ln\ln\Gap{2}^{-1}\right)
			\end{split}\end{equation*}
		samples in expectation conditioning on $\event$. This proves Lemma~\ref{lem:newtoysample}.

	\subsection{Discussion}
		In the \algnewtoy{} algorithm, knowing the complexity $H$ in advance is crucial to
		the efficient allocation of confidence levels ($\delta'_r$'s) to different calls of $\EL$.
		When $H$ is unknown, our approach is
		to run an elimination procedure similar to \algnewtoy{} with a guess of $H$.
		The major difficulty is that when our guess
		is much smaller than the actual complexity,
		the total confidence that we allocate will eventually exceed the total confidence $\delta$.
		Thus, we cannot assume in our analysis that all calls to the $\EL$ procedure are correct.
		We present our \algguess{} algorithm for the \bestarm{} problem in Appendix~\ref{app:alg}.

\bibliography{team} 

\begin{thebibliography}{30}
\providecommand{\natexlab}[1]{#1}
\providecommand{\url}[1]{\texttt{#1}}
\expandafter\ifx\csname urlstyle\endcsname\relax
  \providecommand{\doi}[1]{doi: #1}\else
  \providecommand{\doi}{doi: \begingroup \urlstyle{rm}\Url}\fi

\bibitem[Afshani et~al.(2009)Afshani, Barbay, and Chan]{afshani2009instance}
Peyman Afshani, J{\'e}r{\'e}my Barbay, and Timothy~M Chan.
\newblock Instance-optimal geometric algorithms.
\newblock In \emph{Foundations of Computer Science, 2009. FOCS'09. 50th Annual
  IEEE Symposium on}, pages 129--138. IEEE, 2009.

\bibitem[Audibert and Bubeck(2010)]{audibert2010best}
Jean-Yves Audibert and S{\'e}bastien Bubeck.
\newblock Best arm identification in multi-armed bandits.
\newblock In \emph{COLT-23th Conference on Learning Theory-2010}, pages 13--p,
  2010.

\bibitem[Bechhofer(1954)]{bechhofer1954single}
Robert~E Bechhofer.
\newblock A single-sample multiple decision procedure for ranking means of
  normal populations with known variances.
\newblock \emph{The Annals of Mathematical Statistics}, pages 16--39, 1954.

\bibitem[Bechhofer et~al.(1968)Bechhofer, Kiefer, and
  Sobel]{bechhofer1968sequential}
Robert~Eric Bechhofer, Jack Kiefer, and Milton Sobel.
\newblock \emph{Sequential identification and ranking procedures: with special
  reference to Koopman-Darmois populations}, volume~3.
\newblock University of Chicago Press, 1968.

\bibitem[Bubeck et~al.(2012)Bubeck, Cesa-Bianchi, et~al.]{bubeck2012regret}
S{\'e}bastien Bubeck, Nicolo Cesa-Bianchi, et~al.
\newblock Regret analysis of stochastic and nonstochastic multi-armed bandit
  problems.
\newblock \emph{Foundations and Trends{\textregistered} in Machine Learning},
  5\penalty0 (1):\penalty0 1--122, 2012.

\bibitem[Bubeck et~al.(2013)Bubeck, Wang, and Viswanathan]{bubeck2013multiple}
S{\'e}bastien Bubeck, Tengyao Wang, and Nitin Viswanathan.
\newblock Multiple identifications in multi-armed bandits.
\newblock In \emph{International Conference on Machine Learning}, pages
  258--265, 2013.

\bibitem[Cao et~al.(2015)Cao, Li, Tao, and Li]{cao2015top}
Wei Cao, Jian Li, Yufei Tao, and Zhize Li.
\newblock On top-k selection in multi-armed bandits and hidden bipartite
  graphs.
\newblock In \emph{Advances in Neural Information Processing Systems}, pages
  1036--1044, 2015.

\bibitem[Carpentier and Locatelli(2016)]{carpentier2016tight}
Alexandra Carpentier and Andrea Locatelli.
\newblock Tight (lower) bounds for the fixed budget best arm identification
  bandit problem.
\newblock In \emph{Proceedings of the 29th Conference on Learning Theory},
  2016.

\bibitem[Cesa-Bianchi and Lugosi(2006)]{cesa2006prediction}
Nicolo Cesa-Bianchi and G{\'a}bor Lugosi.
\newblock \emph{Prediction, learning, and games}.
\newblock Cambridge university press, 2006.

\bibitem[Chen and Li(2015)]{chen2015optimal}
Lijie Chen and Jian Li.
\newblock On the optimal sample complexity for best arm identification.
\newblock \emph{arXiv preprint arXiv:1511.03774}, 2015.

\bibitem[Chen and Li(2016)]{chen2016open}
Lijie Chen and Jian Li.
\newblock Open problem: Best arm identification: Almost instance-wise
  optimality and the gap entropy conjecture.
\newblock In \emph{29th Annual Conference on Learning Theory}, pages
  1643--1646, 2016.

\bibitem[Chen et~al.(2016)Chen, Gupta, and Li]{chen2016pure}
Lijie Chen, Anupam Gupta, and Jian Li.
\newblock Pure exploration of multi-armed bandit under matroid constraints.
\newblock In \emph{29th Annual Conference on Learning Theory}, pages 647--669,
  2016.

\bibitem[Chen et~al.(2017)Chen, Li, and Qiao]{chen2017nearly}
Lijie Chen, Jian Li, and Mingda Qiao.
\newblock Nearly instance optimal sample complexity bounds for top-k arm
  selection.
\newblock In \emph{Artificial Intelligence and Statistics}, pages 101--110,
  2017.

\bibitem[Chen et~al.(2014)Chen, Lin, King, Lyu, and
  Chen]{chen2014combinatorial}
Shouyuan Chen, Tian Lin, Irwin King, Michael~R Lyu, and Wei Chen.
\newblock Combinatorial pure exploration of multi-armed bandits.
\newblock In \emph{Advances in Neural Information Processing Systems}, pages
  379--387, 2014.

\bibitem[Even-Dar et~al.(2002)Even-Dar, Mannor, and Mansour]{even2002pac}
Eyal Even-Dar, Shie Mannor, and Yishay Mansour.
\newblock Pac bounds for multi-armed bandit and markov decision processes.
\newblock In \emph{International Conference on Computational Learning Theory},
  pages 255--270. Springer, 2002.

\bibitem[Even-Dar et~al.(2006)Even-Dar, Mannor, and Mansour]{even2006action}
Eyal Even-Dar, Shie Mannor, and Yishay Mansour.
\newblock Action elimination and stopping conditions for the multi-armed bandit
  and reinforcement learning problems.
\newblock \emph{Journal of machine learning research}, 7\penalty0
  (Jun):\penalty0 1079--1105, 2006.

\bibitem[Farrell(1964)]{farrell1964asymptotic}
RH~Farrell.
\newblock Asymptotic behavior of expected sample size in certain one sided
  tests.
\newblock \emph{The Annals of Mathematical Statistics}, pages 36--72, 1964.

\bibitem[Gabillon et~al.(2011)Gabillon, Ghavamzadeh, Lazaric, and
  Bubeck]{gabillon2011multi}
Victor Gabillon, Mohammad Ghavamzadeh, Alessandro Lazaric, and S{\'e}bastien
  Bubeck.
\newblock Multi-bandit best arm identification.
\newblock In \emph{Advances in Neural Information Processing Systems}, pages
  2222--2230, 2011.

\bibitem[Gabillon et~al.(2012)Gabillon, Ghavamzadeh, and
  Lazaric]{gabillon2012best}
Victor Gabillon, Mohammad Ghavamzadeh, and Alessandro Lazaric.
\newblock Best arm identification: A unified approach to fixed budget and fixed
  confidence.
\newblock In \emph{Advances in Neural Information Processing Systems}, pages
  3212--3220, 2012.

\bibitem[Gabillon et~al.(2016)Gabillon, Lazaric, Ghavamzadeh, Ortner, and
  Bartlett]{gabillon2016improved}
Victor Gabillon, Alessandro Lazaric, Mohammad Ghavamzadeh, Ronald Ortner, and
  Peter Bartlett.
\newblock Improved learning complexity in combinatorial pure exploration
  bandits.
\newblock In \emph{Proceedings of the 19th International Conference on
  Artificial Intelligence and Statistics}, pages 1004--1012, 2016.

\bibitem[Garivier and Kaufmann(2016)]{garivier2016optimal}
Aur{\'e}lien Garivier and Emilie Kaufmann.
\newblock Optimal best arm identification with fixed confidence.
\newblock In \emph{Proceedings of the 29th Conference On Learning Theory (to
  appear)}, 2016.

\bibitem[Jamieson et~al.(2014)Jamieson, Malloy, Nowak, and
  Bubeck]{jamieson2014lil}
Kevin Jamieson, Matthew Malloy, Robert Nowak, and S{\'e}bastien Bubeck.
\newblock lil’ucb: An optimal exploration algorithm for multi-armed bandits.
\newblock In \emph{Conference on Learning Theory}, pages 423--439, 2014.

\bibitem[Kalyanakrishnan and Stone(2010)]{kalyanakrishnan2010efficient}
Shivaram Kalyanakrishnan and Peter Stone.
\newblock Efficient selection of multiple bandit arms: Theory and practice.
\newblock In \emph{Proceedings of the 27th International Conference on Machine
  Learning (ICML-10)}, pages 511--518, 2010.

\bibitem[Kalyanakrishnan et~al.(2012)Kalyanakrishnan, Tewari, Auer, and
  Stone]{kalyanakrishnan2012pac}
Shivaram Kalyanakrishnan, Ambuj Tewari, Peter Auer, and Peter Stone.
\newblock Pac subset selection in stochastic multi-armed bandits.
\newblock In \emph{Proceedings of the 29th International Conference on Machine
  Learning (ICML-12)}, pages 655--662, 2012.

\bibitem[Karnin et~al.(2013)Karnin, Koren, and Somekh]{karnin2013almost}
Zohar Karnin, Tomer Koren, and Oren Somekh.
\newblock Almost optimal exploration in multi-armed bandits.
\newblock In \emph{Proceedings of the 30th International Conference on Machine
  Learning (ICML-13)}, pages 1238--1246, 2013.

\bibitem[Kaufmann and Kalyanakrishnan(2013)]{kaufmann2013information}
Emilie Kaufmann and Shivaram Kalyanakrishnan.
\newblock Information complexity in bandit subset selection.
\newblock In \emph{Conference on Learning Theory}, pages 228--251, 2013.

\bibitem[Kaufmann et~al.(2015)Kaufmann, Capp{\'e}, and
  Garivier]{kaufmann2015complexity}
Emilie Kaufmann, Olivier Capp{\'e}, and Aur{\'e}lien Garivier.
\newblock On the complexity of best arm identification in multi-armed bandit
  models.
\newblock \emph{The Journal of Machine Learning Research}, 2015.

\bibitem[Mannor and Tsitsiklis(2004)]{mannor2004sample}
Shie Mannor and John~N Tsitsiklis.
\newblock The sample complexity of exploration in the multi-armed bandit
  problem.
\newblock \emph{Journal of Machine Learning Research}, 5\penalty0
  (Jun):\penalty0 623--648, 2004.

\bibitem[Robbins(1985)]{robbins1985some}
Herbert Robbins.
\newblock Some aspects of the sequential design of experiments.
\newblock In \emph{Herbert Robbins Selected Papers}, pages 169--177. Springer,
  1985.

\bibitem[Zhou et~al.(2014)Zhou, Chen, and Li]{zhou2014optimal}
Yuan Zhou, Xi~Chen, and Jian Li.
\newblock Optimal pac multiple arm identification with applications to
  crowdsourcing.
\newblock In \emph{Proceedings of the 31st International Conference on Machine
  Learning (ICML-14)}, pages 217--225, 2014.

\end{thebibliography}

\newpage

\appendix

\section*{Organization of the Appendix}
The appendix contains the proofs of our main results. In Section~\ref{app:alg}, we present our algorithm for \bestarm{} along with a few useful observations. In Section~\ref{app:cor} and Section~\ref{app:samp}, we prove the correctness and the sample complexity of our algorithm, thus proving Theorem~\ref{theo:uppb-main}. We present the complete proof of Theorem~\ref{theo:lowb-main} in Section~\ref{app:lb}. Finally, Section~\ref{app:toymiss} contains the complete proofs of Lemma~\ref{lem:newtoycorrect} and Lemma~\ref{lem:newtoysample}.

\section{Upper Bound}\label{app:alg}
	\subsection{Building Blocks}
		We start by presenting the missing implementation and performance guarantees of our subroutines $\FT$ and $\EL$.
		
		\textbf{Fraction test.}
		Recall that on input $(S,\clow,c^{\hi},\theta^{\lo},\theta^{\hi},\delta)$, procedure $\FT$ decides whether a sufficiently large fraction (with respect to $\theta^{\lo}$ and $\theta^{\hi}$) of arms in $S$ have means smaller than the thresholds $\clow$ and $c^{\hi}$. The pseudocode of $\FT$ is shown below.

		\begin{algorithm2e}[H]
			\caption{$\FT(S,\clow,\chigh,\thetalo,\thetahi,\delta)$}
			\KwIn{An arm set $S$, thresholds $\clow$, $\chigh$, $\thetalo$, $\thetahi$, and confidence level $\delta$.}
			$\eps\leftarrow \chigh-\clow$;
			$\Delta\leftarrow \thetahi-\thetalo$;

			$m\leftarrow(\Delta/6)^{-2}\ln(2/\delta)$;
			$\cnt\leftarrow0$;

			\For{\upshape $i=1$ to $m$}
			{
				Pick $A\in S$ uniformly at random;

				$\hat\mu_A\leftarrow\US(\{A\},\eps/2,\Delta/6)$;

				\uIf{$\hat\mu_A<(\clow+\chigh)/2$}
					{$\cnt\leftarrow\cnt+1$;}
			}
			\uIf{$\cnt/m > (\thetalo+\thetahi)/2$} {
				\textbf{return} True;
			} \uElse {
				\textbf{return} False;
			}
		\end{algorithm2e}

		Now we prove Fact~\ref{F3}.

		$ $

		\noindent\textbf{Fact~\ref{F3}} (restated)\textit{
			$\FT(S,\clow,\chigh,\thetalo,\thetahi,\delta)$ takes
				$O((\eps^{-2}\ln\delta^{-1})\cdot(\Delta^{-2}\ln\Delta^{-1}))$
			samples, where $\eps=\chigh-\clow$ and $\Delta=\thetahi-\thetalo$. With probability $1-\delta$, the following two claims hold simultaneously:
			\begin{itemize}
			\item If $\FT$ returns True, $|\{A\in S:\mu_A<\chigh\}|>\thetalo|S|$.
			\item If $\FT$ returns False, $|\{A\in S:\mu_A<\clow\}|<\thetahi|S|$.
			\end{itemize}
		}

		$ $

		\begin{proof}
			The first claim directly follows from Fact~\ref{F1} and
			\[
				m\cdot O(\eps^{-2}\ln\Delta^{-1}) = O((\eps^{-2}\ln\delta^{-1})\cdot(\Delta^{-2}\ln\Delta^{-1})).
			\]
			It remains to prove the contrapositive of the second claim: $|\{A\in S:\mu_A<\clow\}|\ge\thetahi|S|$ implies $\FT$ returns True, and $|\{A\in S:\mu_A<\chigh\}|\le\thetalo|S|$ implies $\FT$ returns False.

			Suppose $|\{A\in S:\mu_A<\clow\}|\ge\thetahi|S|$. Then in each iteration of the for-loop, it holds that $\mu_A < \clow$ with probability at least $\thetahi$. Conditioning on $\mu_A < \clow$, by Fact~\ref{F1} we have
				$$\hat\mu_A \le \mu_A + \eps/2 < \clow + \eps/2 = (\clow+\chigh)/2$$
			with probability at least $1 - \Delta/6$. Thus, the expected increment of counter $\cnt$ is lower bounded by
				$$\thetahi(1-\Delta/6)\ge\thetahi - \Delta/6\text{.}$$

			Thus, $\cnt / m$ is the mean of $m$ i.i.d. Bernoulli random variables with means greater than or equal to $\thetahi - \Delta/6$. By the Chernoff bound, it holds with probability $1-\delta/2$ that
				$$\cnt / m \ge \thetahi - \Delta/6 - \Delta/6 > (\thetalo+\thetahi)/2\text{.}$$

			An analogous argument proves $\cnt / m < (\thetalo+\thetahi)/2$ with probability $1-\delta/2$, given $|\{A\in S:\mu_A<\chigh\}|\le\thetalo|S|$. This completes the proof.
		\end{proof}

		\textbf{Elimination.}
		We implement procedure $\EL$ by repeatedly calling $\FT$ to determine whether a large fraction of the remaining arms have means smaller than the thresholds. If so, we uniformly sample the arms, and eliminate those with low empirical means.

		\begin{algorithm2e}[H]
		\caption{$\EL(S,\dlow,\dhigh,\delta)$}
		\KwIn{An arm set $S$, thresholds $\dlow$, $\dhigh$, and confidence level $\delta$.}
		\KwOut{Arm set after the elimination.}
		$S_1\leftarrow S$;

		$d^{\mi}\leftarrow(\dlow+\dhigh)/2$;

		\For{\upshape $r=1$ to $+\infty$}
		{
			$\delta_r\leftarrow\delta/(10\cdot2^r)$;

			\uIf{$\FT(S_r,\dlow,d^{\mi},0.05,0.1,\delta_r)$} {
				$\hat\mu\leftarrow\US(S_r,(\dhigh-d^{\mi})/2,\delta_r)$;

				$S_{r+1}\leftarrow\left\{A\in S_r:\hat\mu_A>(d^{\mi}+\dhigh)/2\right\}$;
			} \uElse {
				\textbf{return} $S_r$;
			}
		}
		\end{algorithm2e}

		We prove Fact~\ref{F4} in the following.

		$ $

		\noindent\textbf{Fact~\ref{F4}} (restated)\textit{
		$\EL(S,\dlow,\dhigh,\delta)$ takes $O(|S|\eps^{-2}\ln\delta^{-1})$ samples in expectation, where $\eps=\dhigh-\dlow$. Let $S'$ be the set returned by $\EL(S,\dlow,\dhigh,\delta)$. Then we have
		$$\Pr[|\{A\in S':\mu_A<\dlow\}|\le0.1|S'|]\ge1-\delta/2\text{.}$$
		Moreover, for each arm $A\in S$ with $\mu_A\ge \dhigh$, we have
		$$\Pr[A\in S']\ge1-\delta/2\text{.}$$
		}

		$ $

		\begin{proof}
			Let $\eps = \dhigh - \dlow$. To bound the number of samples taken by $\EL$, we note that the number of samples taken in the $r$-th iteration is dominated by that taken by $\US$, $O(|S_r|\eps^{-2}\ln\delta_r^{-1})$. It suffices to show that $|S_r|$ decays exponentially (in expectation); a direct summation over all $r$ proves the sample complexity bound.

			We fix a particular round $r$. Suppose $\FT$ returns correctly (which happens with probability at least $1 - \delta_r$) and the algorithm does not terminate at round $r$. Then by Fact~\ref{F3}, it holds that
				$$|\{A\in S_r: \mu_A < \dmid\}| > 0.05|S_r|\text{.}$$
			For each $A\in S_r$ with $\mu_A < \dmid$, it holds with probability $1 - \delta_r$ that
				$$\hat\mu_A < \mu_A + (\dhigh-\dmid)/2
				< \dmid + (\dhigh-\dmid)/2
				= (\dmid + \dhigh) / 2\text{.}$$
			Note that $\delta_r = \delta/(10\cdot 2^r)\le0.1$. Thus, at most a $0.1$ fraction of arms in $\{A\in S_r: \mu_A < \dmid\}$ would remain in $S_{r+1}$ in expectation. It follows that conditioning on the correctness of $\FT$ at round $r$, the expectation of $|S_{r+1}|$ is upper bounded by
				$$0.05|S_r|\cdot\delta_r + 0.95|S_r| \le 0.05|S_r|/10 + 0.95|S_r| = 0.955|S_r|\text{.}$$
			Moreover, even if $\FT$ returns incorrectly, which happens with probability at most $0.1$, we still have $|S_{r+1}|\le|S_r|$. Therefore,
				$$\Ex[|S_{r+1}|]\le 0.9\cdot 0.955\Ex[|S_r|] + 0.1\Ex[|S_r|] < 0.96\Ex[|S_r|]\text{.}$$
			A simple induction yields $\Ex[|S_r|]\le0.96^{r-1}|S|$. Then the sample complexity of $\EL$ is upper bounded by
				\begin{equation*}\begin{split}
				\sum_{r=1}^{\infty}\Ex[|S_r|]\eps^{-2}\ln\delta_r^{-1}
				&= O\left(|S|\eps^{-2}\sum_{r=1}^{\infty}0.96^{r-1}(\ln\delta^{-1} + r)\right)\\
				&= O\left(|S|\eps^{-2}\ln\delta^{-1}\right)\text{.}
				\end{split}\end{equation*}

			Then we proceed to the proof of the second claim. Let $\event$ denote the event that all calls to procedure $\FT$ returns correctly. By Fact~\ref{F3} and a union bound,
				$$\pr{\event_A}\ge1-\sum_{r=1}^{\infty}\delta_r\ge1-\delta/2\text{.}$$
			Conditioning on event $\event$, if the algorithm terminates and returns $S_r$ at round $r$, Fact~\ref{F3} implies that
				$$|\{A\in S_r:\mu_A < \dlow\}| < 0.1|S_r|\text{.}$$
			This proves the second claim.

			Finally, fix an arm $A\in S$ with $\mu_A > \dhigh$. Define $\event_A$ as the event that every call to $\FT$ returns correctly in the algorithm, and $|\hat\mu_A - \mu_A| < (\dhigh-\dmid)/2$ in every round. By Facts \ref{F1}~and~\ref{F3},
				$$\pr{\event_A}\ge1-\sum_{r=1}^{\infty}2\delta_r\ge1-\delta/2\text{.}$$
			Then in each round $r$, it holds conditioning on $\event_A$ that
				$$\hat\mu_A \ge \mu_A - (\dhigh-\dmid)/2
				> \dhigh - (\dhigh-\dmid) / 2
				= (\dmid + \dhigh) / 2\text{.}$$
			Thus, with probability $1-\delta/2$, $A$ is never removed from $S_r$.
		\end{proof}

\subsection{Overview}
As shown in Section~\ref{sec:toy}, we can solve \bestarm{} using
	$$O\left(H\cdot(\arment+\ln\delta^{-1})+\Gap{2}^{-2}\ln\ln\Gap{2}^{-1}\right)$$
samples,
if we know in advance the complexity of the instance,
i.e., $H = \sum_{i=2}^{n}\Gap{i}^{-2}$.

The value of $H$ is essential for
allocating appropriate confidence levels to different calls of $\EL$
and achieving the near-optimal sample complexity.
When $H$ is unknown, our strategy is to guess its value.
The major difficulty with our approach is that when our guess, $\hat H$,
is much smaller than the actual complexity $H$,
the total confidence that we allocate will exceed the total confidence $\delta$.
To prevent this from happening, we maintain the total confidence
that we have allocated so far,
and terminate the algorithm as soon as the sum exceeds $\delta$.%
\footnote{For ease of analysis, we actually use $\delta^2$ instead of $\delta$ in the algorithm.}
After that, we try a guess that is a hundred times larger.
As we will see later, the most challenging part of the analysis is
to ensure that our algorithm does not return an incorrect answer
when $\hat H$ is too small.

We also keep track of the number of samples that have been taken so far. Roughly speaking, when the number exceeds $100\hat H$, we also terminate the algorithm and try the next guess of $\hat H$. This simplifies the analysis by ensuring that the number of samples we take for each guess grows exponentially, and thus it suffices to bound the number of samples taken on the last guess.

\subsection{Algorithm}
Algorithm \algent{} takes an instance of \bestarm{}, a confidence $\delta$ and a guess of complexity $\hat H_t=100^t$. It either returns an optimal arm (i.e., ``accept'' $\hat H_t$) or reports an error indicating that the given $\hat H_t$ is much smaller than the actual complexity (i.e., ``reject'' $\hat H_t$).

\begin{algorithm2e}\label{algo2}
\caption{$\algent(I, \delta, \hat H_t)$}
\KwIn{Instance $I$, confidence $\delta$ and a guess of complexity $\hat H_t=100^t$.}
\KwOut{The best arm, or an error indicating the guess is wrong.}
$S_1\leftarrow I$; $H_1\leftarrow 0$; $T_1\leftarrow 0$;

$\theta_0\leftarrow0.3$; $c\leftarrow\log_4100$;

\For{\upshape $r=1$ to $\infty$} {
	\uIf{$|S_r|=1$} {
		\textbf{return} the only arm in $S_r$;
	}
	$\eps_r\leftarrow2^{-r}$; $\delta_r\leftarrow\delta/(50r^2t^2)$;

	$\delta'_r\leftarrow (4|S_r|\eps_r^{-2}/\hat H)\delta^2$;

	$T_{r+1}\leftarrow T_r+|S_r|\eps_r^{-2}\ln\left(|S_r|\eps_r^{-2}\delta/\hat H_t\right)^{-1}$;

	\uIf{\upshape $(H_r+4|S_r|\eps_r^{-2}\ge\hat H_t)$ or $(T_{r+1}\ge100\hat H_t)$}
	  {\textbf{return}\text{ error};}\label{Line1}

	$a_r\leftarrow\ME(S_r,0.125\eps_r,0.01)$;

	$\hat\mu_{a_r}\leftarrow\US(\{a_r\},0.125\eps_r,\delta_r)$;

	$\theta_r\leftarrow\theta_{r-1}+(ct-r)^{-2}/10$;

	\uIf{\upshape$\FT(S_r,\hat\mu_{a_r}-1.75\eps_r,\hat\mu_{a_r}-1.125\eps_r,\delta_r,\theta_{r-1},\theta_r)$} {
		$H_{r+1}\leftarrow H_r+4|S_r|\eps_r^{-2}$;

		$S_{r+1}\leftarrow\EL(S_r,\hat\mu_{a_r}-0.75\eps_r,\hat\mu_{a_r}-0.625\eps_r,\delta'_r)$;
	} \uElse {
		$S_{r+1}\leftarrow S_r$;

		$H_{r+1}\leftarrow H_r$;
	}
}
\end{algorithm2e}

Throughout the algorithm, we maintain $S_r$, $H_r$ and $T_r$ for each round $r$. $S_r$ denotes the collection of arms that are still under consideration at the beginning of round $r$. We say that an arm is removed (or eliminated) at round $r$, if it is in $S_r\setminus S_{r+1}$. Roughly speaking, $H_r$ is an estimate of the total complexity of arms in group $G_1,G_2,\ldots,G_{r}$. When this quantity exceeds our guess $\hat H_t$, \algent{} directly rejects
(i.e., returns an error). $T_r$ is an upper bound on the number of samples taken by $\ME$ and $\EL$\footnote{As we will see later, the analysis of the sample complexity of $\ME$ and $\EL$ are different from the other two procedures.}  before round $r$. As mentioned before, we also terminate the algorithm when $T_r$ exceeds $100\hat H_t$. Intuitively, this prevents \algent{} from taking too many samples on small guesses of $H$, which gives rise to an inferior sample complexity.

In each round of \algent{}, we first call $\ME$ to obtain a near-optimal arm $a_r$. Then we use $\US$ to estimate the mean of $a_r$, denoted by $\hat\mu_{a_r}$. After that, we call $\FT$ with appropriate parameters to find out whether a considerable fraction of arms in $S_r$ have gaps larger than $\eps_r$. If so, we call procedure $\EL$ and update the value of $H_{r+1}$ accordingly. Note that we set the thresholds $\{\theta_r\}$ of $\FT$ such that the intervals $[\theta_{r-1},\theta_r]$ are disjoint. In particular, this property is essential for proving Lemma~\ref{L9} in the analysis of the correctness of the algorithm.

Our algorithm for \bestarm{} guesses the complexity of the instance and invokes \algent{} to check whether the guess is reasonable. If \algent{} reports an error, we try a guess that is a hundred times larger. Otherwise, we return the arm chosen by \algent{}.

\begin{algorithm2e}[H]\label{algo3}
\caption{\algguess{}}
\KwIn{Instance $I$ and confidence $\delta$.}
\KwOut{The best arm.}
\For{\upshape$t=1$ to $\infty$} {
	$\hat H_t\leftarrow100^t$;

	Call $\algent(I,\delta,\hat H_t)$;

	\uIf{\upshape \algent{} does not return an error} {
		\textbf{return} the arm returned by \algent{};
	}
}
\end{algorithm2e}

\subsection{Observations}
We start with a few simple observations on \algent{} that will be used throughout the analysis.

We first note that \algent{} lasts $O(t)$ rounds on guess $\hat H_t$, and our definition of $\theta_r$ ensures that all $\theta_r$ are in $[0.3,0.5]$.

\begin{Observation}\label{O1}
The for-loop in $\algent(I, \delta, \hat H_t)$ is executed at most $ct$ times, where $c=\log_4100$.
\end{Observation}

\begin{proof}
When $r\ge ct-1$,
$$H_r+4|S_r|\eps_r^{-2}\ge4\cdot4^{ct-1}=\hat H_t\text{.}$$
Thus \algent{} rejects at the if-statement.
\end{proof}

\begin{Observation}\label{O2}
For all $t\ge1$ and $1\le r\le ct-1$, $0.3\le\theta_{r-1}\le\theta_r\le0.5$.
\end{Observation}

\begin{proof}
Clearly $\theta_r\ge\theta_0=0.3$. Moreover,
$$\theta_r=\theta_0+\sum_{k=1}^r(ct-k)^{-2}/10\le0.3+\frac{1}{10}\sum_{k=1}^{\infty}k^{-2}\le0.5\text{.}$$
\end{proof}

The following observation bounds the value of $\hat\mu_{a_r}$ at round $r$, conditioning on the correctness of $\US$ and $\ME$.

\begin{Observation}\label{O3}\label{O4}
	If $\US$ returns correctly at round $r$,
		$\hat\mu_{a_r}\le\Mean{1}(S_r)+0.125\eps_r$.
	Here $\Mean{1}(S_r)$ denotes the largest mean of arms in $S_r$.
	If both $\US$ and $\ME$ return correctly,
		$\hat\mu_{a_r}\ge\Mean{1}(S_r)-0.25\eps_r$.
\end{Observation}

\begin{proof}
	By definition, $\mu_{a_r}\le\Mean{1}(S_r)$.
	When $\US(\{a_r\}, 0.125\eps_r, \delta_r)$ returns correctly,
	it holds that
		$$\hat\mu_{a_r}\le\mu_{a_r}+0.125\eps_r\le\Mean{1}+0.125\eps_r\text{.}$$
	When both $\ME$ and $\US$ are correct, $\mu_{a_r}\ge\Mean{1}(S_r)-0.125\eps_r$, and thus
		$$\hat\mu_{a_r}\ge\mu_{a_r}-0.125\eps_r\ge\Mean{1}(S_r)-0.25\eps_r\text{.}$$
\end{proof}

The following two observations bound the thresholds used in $\FT$ and $\EL$ by applying Observation~\ref{O3}.

\begin{Observation}\label{O5}
	At round $r$, let
		$\clow_r=\hat\mu_{a_r}-1.75\eps_r$
	and
		$\chigh_r=\hat\mu_{a_r}-1.125\eps_r$
	denote the two thresholds used in $\FT$.
	If $\US$ returns correctly,
		$\chigh_r\le\Mean{1}(S_r)-\eps_r$.
	If both $\ME$ and $\US$ return correctly,
		$\clow_r\ge\Mean{1}(S_r)-2\eps_r$.
\end{Observation}

\begin{proof}
	Observation~\ref{O3} implies that when $\US$ is correct,
		$$\chigh_r\le\Mean{1}(S_r)+0.125\eps_r-1.125\eps_r
		=\Mean{1}(S_r)-\eps_r$$
	and when both $\ME$ and $\US$ return correctly,
		$$\clow_r\ge\Mean{1}(S_r)-0.25\eps_r-1.75\eps_r
		=\Mean{1}(S_r)-2\eps_r\text{.}$$
\end{proof}

\begin{Observation}\label{O6}
	Let
		$\dlow_r=\hat\mu_{a_r}-0.75\eps_r$
	and
		$\dhigh_r=\hat\mu_{a_r}-0.625\eps_r$
	denote the two thresholds used in $\EL$.
	If $\US$ returns correctly,
		$\dhigh_r\le\Mean{1}(S_r)-0.5\eps_r$.
	If both $\ME$ and $\US$ return correctly,
		$\dlow_r\ge\Mean{1}(S_r)-\eps_r$.
\end{Observation}

\begin{proof}
	By the same argument, we have
		$$\dhigh_r\le\Mean{1}(S_r)+0.125\eps_r-0.625\eps_r
		=\Mean{1}(S_r)-0.5\eps_r$$
	when $\US$ returns correctly, and
		$$\dlow_r\ge\Mean{1}(S_r)-0.25\eps_r-0.75\eps_r
		=\Mean{1}(S_r)-\eps_r$$
	when both $\ME$ and $\US$ are correct.
\end{proof}

\section{Analysis of Correctness}\label{app:cor}
\subsection{Overview}
We start with a high-level overview of the proof of our algorithm's correctness. We first define a good event on which we condition in the rest of the analysis. Let $\event_1$ be the event that in a particular run of \algguess{}, all calls of procedure $\US$ and $\FT$ return correctly. Recall that $\delta_r$, the confidence of $\US$ and $\FT$, is set to be $\delta/(50r^2t^2)$ in the $r$-th round of iteration $t$. By a union bound,
$$\Pr[\event_1]\ge1-2\sum_{t=1}^{\infty}\sum_{r=1}^{\infty}\delta/(50t^2r^2)=1-2\delta(\pi^2/6)^2/50\ge1-\delta/3\text{.}$$

The $\delta$-correctness of our algorithm is guaranteed by the following two lemmas. The first lemma states that \algent{} accepts a guess $\hat H_t$ and returns correctly with high probability when $\hat H_t$ is sufficiently large. The second lemma guarantees that \algent{} rejects a guess $\hat H_t$ when $\hat H_t$ is significantly smaller than $H$, the actual complexity. More precisely, we define the following two thresholds:
	$$t_{\max}=\lfloor\log_{100}H\rfloor-2$$
and
	$$t'_{\max}=\left\lceil\log_{100}\left[H(\arment+\ln\delta^{-1})\delta^{-1}\right]\right\rceil+2\text{.}$$
The precise statements of the two lemmas are shown below.

\begin{Lemma}\label{L1}
With probability $1-\delta/3$ conditioning on event $\event_1$, \algguess{} halts before or at iteration $t'_{\max}$ and it never returns a sub-optimal arm between iteration $t_{\max}+1$ and $t'_{\max}$.
\end{Lemma}

\begin{Lemma}\label{L2}
With probability $1-\delta/3$ conditioning on event $\event_1$, \algguess{} never returns a sub-optimal arm in the first $t_{\max}$ iterations.
\end{Lemma}

Lemma~\ref{L1} and Lemma~\ref{L2} directly imply the following theorem.

\begin{Theorem}\label{T2}
\algguess{} is a $\delta$-correct algorithm for \bestarm{}.
\end{Theorem}

\begin{proof}
Recall that $\Pr[\event_1]\ge1-\delta/3$. It follows directly from Lemma~\ref{L1} and Lemma~\ref{L2} that with probability $1-\delta$, \algent{} accepts at least one of $\hat H_{1},\hat H_{2},\ldots,\hat H_{t'_{\max}}$. Moreover, when \algent{} accepts, it returns the optimal arm. Therefore, \algguess{} is $\delta$-correct.
\end{proof}

\subsection{Useful Lemmas}
\newcommand{\Proc}{\mathbb{P}}

To analyze our algorithm, it is essential to bound the probability that a specific guess $\hat H_t$ gets rejected by \algent{}. We hope that this probability is high when $\hat H_t$ is small (compared to the true complexity $H$), while it is reasonably low when $\hat H_t$ is large enough.

It turns out to be useful to consider the following procedure $\Proc$ obtained from \algent{} by removing the if-statement that checks whether $H_r+4|S_r|\eps_r^{-2}\ge\hat H_t$ and $T_{r+1}\ge100\hat H_t$. In other words, the modified procedure $\Proc$ never rejects, regardless the value of $\hat H_t$. Note that $r$, the number of rounds, may exceed $ct$ in $\Proc$, which leads to invalid values of $\theta_r$. In this case, we simply assume that the thresholds used in $\FT$ are $0.3$ and $0.5$ respectively, and the following analysis still works. Define random variable $H_{\infty}$ and $T_{\infty}$ to be the final estimation of the complexity and the number of samples at the end of $\Proc$. More precisely, if $\Proc$ terminates at round $r^*$, then $H_{\infty}$ and $T_{\infty}$ are defined as $H_{r^*}$ and $T_{r^*}$, respectively.

Note that there is a natural mapping from an execution of $\Proc$ to an execution of \algent{}. In particular, if both $H_{\infty}<\hat H_t$ and $T_{\infty}<100\hat H_t$ hold in an execution of procedure $\Proc$, then \algent{} accepts in the corresponding run. Therefore, we may upper bound the probability of rejection by establishing upper bounds of $H_{\infty}$ and $T_{\infty}$. The following two lemmas bound the expectation of $H_{\infty}$ and $T_{\infty}$ conditioning on the event that $\EL$ always returns correctly.

\begin{Lemma}\label{L7}
$\Ex[H_{\infty}|\text{all }\EL\text{ return correctly}]\le256H$.
\end{Lemma}

\begin{Lemma}\label{L8}
Suppose $\hat H_t\ge H$. $\Ex[T_{\infty}|\text{all }\EL\text{ return correctly}]\le16(H(\arment+\ln\delta^{-1}+\ln(\hat H_t/H)))$.
\end{Lemma}

Note that it is crucial for the two lemmas above that all $\EL$ are correct. The following lemma gives an upper bound on the probability that \emph{some} call of $\EL$ returns incorrectly. Lemmas \ref{L7}~through~\ref{L9} together can be used to upper bound the probability of rejecting a guess $\hat H$. In the statement of Lemma~\ref{L9}, we abuse the notation a little bit by assuming $A_{1}\in G_{\infty}$ and $\Gap{1}^{-2}=+\infty$. 

\begin{Lemma}\label{L9}
Suppose that $s\in\{2,3,\ldots,n\}$ and $r^*\in\mathbb{N}\cup\{\infty\}$ satisfy $A_{s-1}\in G_{r^*}$. When \algent{} runs on parameter $\hat H_t<\Gap{s-1}^{-2}$, the probability that there exists a call of procedure $\EL$ that returns incorrectly before round $r^*$ is upper bounded by $$3000s\left(\sum_{i=s}^{n}\Gap{i}^{-2}\right)\delta^2/\hat H_t\text{.}$$
\end{Lemma}

The proofs of the three lemmas above are shown below.

\begin{proof}[Proof of Lemma~\ref{L7}]
In the following analysis, we always implicitly condition on the event that all $\EL$ return correctly. Define $H(r,S)$ as the expectation of $H_{\infty}-H_r$ at the beginning of the $r$-th round of \algent{}, when the current set of arms is $S_r=S$. Let $r_{\max}$ denote $\left\lfloor\log_2\Gap{2}^{-1}\right\rfloor$. Define $$C(r,S)=\sum_{i=r-1}^{\infty}|S\cap G_i|\sum_{j=r}^{i+1}\eps_j^{-2}+\sum_{i=r}^{r_{\max}+1}\eps_i^{-2}$$ and $M(r,S)=|S\cap G_{\le r-2}|$. We prove by induction on $r$ that
\begin{equation}\label{eq1}
H(r,S)\le 128C(r,S)+16M(r,S)\eps_r^{-2}\text{.}
\end{equation}

We start with the base case at round $r_{\max}+2$. Recall that $\clow_r$ and $\dlow_r$ denote the lower threshold of $\FT$ and $\EL$ in round $r$ respectively. For all $r\ge r_{\max}+2$, if $\ME$ returns correctly at round $r$ (which happens with probability 0.99), according to Observation~\ref{O5} and Observation~\ref{O6}, we have
$$\dlow_r\ge \clow_r\ge\Mean{1}-2\eps_r\ge\Mean{1}-2^{-(r_{\max}+1)}\ge\Mean{2}\text{.}$$
Since $\FT$ returns correctly (contioning on $\event_1$) and
$$|\{A\in S_r:\mu_A\le \clow_r\}|\ge|\{A\in S_r:\mu_A\le\Mean{2}\}|=|S_r|-1\ge0.5|S_r|\ge\theta_r|S_r|$$
(the last step applies Observation~\ref{O2}), $\FT$ must return True and $\EL$ will be called. Since we assume that all calls of $\EL$ return correctly, we have
$$|S_{r+1}|-1=|\{A\in S_{r+1}:\mu_A\le\Mean{2}\}|\le|\{A\in S_{r+1}:\mu_A\le \dlow_r\}|\le0.1|S_{r+1}|\text{,}$$
which guarantees that $S_{r+1}$ only contains the optimal arm and the algorithm will return correctly in the next round. Let $r_0$ denote the first round after round $r_{\max}+2$ (inclusive) in which $\ME$ returns correctly. Then according to the discussion above, we have $\Pr[r_0=r]\le0.01^{r-r_{\max}-2}$ for all $r\ge r_{\max}+2$. Thus it follows from a direct summation on possible values of $r_0$ that
\begin{equation*}\begin{split}
H(r_{\max}+2,S)&\le\sum_{r=r_{\max}+2}^{\infty}\Pr[r_0=r]\cdot 4|S|\eps_r^{-2}\\
&\le\sum_{r=r_{\max}+2}^{\infty}4|S|\eps_r^{-2}0.01^{r-r_{\max}-2}\\
&\le8|S|\eps_{r_{\max}+2}^{-2}\le16M(r_{\max}+2,S)\eps_{r_{\max}+2}^{-2}\text{,}
\end{split}\end{equation*}
which proves the base case.

Before proving the induction step, we note the following fact: for $r=1,2,\ldots,r_{\max}+1$,
\begin{equation}\label{eq2}\begin{split}
C(r,S)-C(r+1,S)&=\sum_{i=r-1}^{\infty}|S\cap G_i|\sum_{j=r}^{i+1}\eps_j^{-2}-\sum_{i=r}^{\infty}|S\cap G_i|\sum_{j=r+1}^{i+1}\eps_j^{-2}+\eps_r^{-2}\\
&=\sum_{i=r-1}^{\infty}|S\cap G_i|\eps_r^{-2}+\eps_r^{-2}\\
&=(|S\cap G_{\ge r-1}|+1)\eps_r^{-2}\text{.}
\end{split}\end{equation}
Suppose inequality \eqref{eq1} holds for $r+1$. Consider the following three cases of the execution of \algent{} in round $r$. Let $\Ncur=|S\cap G_{r-1}|$ and $\Nbig=|S\cap G_{\ge r}|$. For brevity, let $\Nsma$ denote $M(r,S)$ in the following. We have $\Nsma+\Ncur+\Nbig=|S|-1$. Note that $S_{r+1}$ is the set of arms that survive round $r$. 

\textbf{Case 1:} $\ME$ returns correctly and $\FT$ returns True.

According to the induction hypothesis, the expectation of $H_{\infty}-H_r$ in this case can be bounded by:
\begin{equation*}\begin{split}
&H(r+1,S_{r+1})+4|S|\eps_r^{-2}\\
\le&128C(r+1,S_{r+1})+16M(r+1,S_{r+1})\eps_{r+1}^{-2}+4|S|\eps_r^{-2}\\
\le&128C(r+1,S)+16[(\Nsma+\Ncur)/10]\cdot(4\eps_r^{-2})+4|S|\eps_r^{-2}\\
=&128[C(r,S)-(\Ncur+\Nbig+1)\eps_r^{-2}]+(6.4\Nsma+6.4\Ncur+4|S|)\eps_r^{-2}\\
=&128C(r,S)+(10.4\Nsma-117.6\Ncur-124\Nbig-124)\eps_r^{-2}\\
\le&128C(r,S)+10.4\Nsma\eps_r^{-2}\text{.}
\end{split}\end{equation*}
Here the third line follows from the fact that $S_{r+1}\subseteq S$ and $C(r+1,S)$ is monotone in $S$. Moreover, the correctness of the $\EL$ procedure implies that $M(r+1,S_{r+1})\le(\Nsma+\Ncur)/10$. The fourth line applies identity \eqref{eq2}.

\textbf{Case 2:} $\ME$ returns correctly and $\FT$ returns False.

Since $\FT$ is always correct (conditioning on $\event_1$) and it returns False, Fact~\ref{F3}, Observation~\ref{O2} and Observation~\ref{O5} together imply $\Nsma\le\theta_r|S|\le|S|/2$. Thus $\Nsma\le|S|-\Nsma=\Ncur+\Nbig+1$.
As $\EL$ is not called in this round, the expectation of $H_{\infty}-H_r$ in this case can be bounded by
\begin{equation*}\begin{split}
H(r+1,S)\le&128C(r+1,S)+16M(r+1,S)\eps_{r+1}^{-2}\\
\le&128[C(r,S)-(\Ncur+\Nbig+1)\eps_r^{-2}]+64(\Nsma+\Ncur)\eps_r^{-2}\\
=&128C(r,S)+(64\Nsma-64\Ncur-128\Nbig-128)\eps_r^{-2}\le 128C(r,S)\text{.}
\end{split}\end{equation*}
Here the last step follows from $64\Nsma-64\Ncur-128\Nbig-128\le64(\Nsma-\Ncur-\Nbig-1)\le0$.

\textbf{Case 3:} $\ME$ returns incorrectly.

In this case, the worst scenario happens when we add $4|S|\eps_r^{-2}$ to the complexity $H_r$, but no arms are eliminated. Then the expectation of $H_{\infty}-H_r$ in this case can be bounded by
\begin{equation*}\begin{split}
&H(r+1,S)+4|S|\eps_r^{-2}\\
\le&128C(r+1,S)+16M(r+1,S)\eps_{r+1}^{-2}+4|S|\eps_r^{-2}\\
\le&128[C(r,S)-(\Ncur+\Nbig+1)\eps_r^{-2}]+[64(\Nsma+\Ncur)+4|S|]\eps_r^{-2}\\
=&128C(r,S)+(68\Nsma-60\Ncur-124\Nbig-124)\eps_r^{-2}\le 128C(r,S)+68\Nsma\eps_r^{-2}\text{.}
\end{split}\end{equation*}

Recall that Case 3 happens with probability at most $0.01$. Thus we have:

\begin{equation*}\begin{split}
H(r,S)&\le 0.01\left[128C(r,S)+68M(r,S)\eps_r^{-2}\right]+0.99\left[128C(r,S)+10.4M(r,S)\eps_r^{-2}\right]\\
&\le128C(r,S)+16M(r,S)\eps_r^{-2}\text{.}
\end{split}\end{equation*}

The induction is completed. Note that \eqref{eq1} directly implies our bound:
\begin{equation*}\begin{split}
&\Ex\left[H_{\infty}|\text{all }\EL\text{ return correctly}\right]\\
=&H(1,S)\le 128C(1,S)+16M(1,S)\\
=&128\sum_{i=0}^{\infty}|S\cap G_i|\cdot\left(\sum_{j=0}^{i+1}4^j\right)\\
\le&256\sum_{i=0}^{\infty}4^{i+1}|S\cap G_i|\\
\le&256\sum_{i=0}^{\infty}\sum_{A\in S\cap G_i}\Delta_A^{-2}=256H\text{.}\\
\end{split}\end{equation*}
\end{proof}

Then we prove Lemma~\ref{L8}, which is restated below.

$ $

\noindent\textbf{Lemma~\ref{L8}.} (restated)\textit{
Suppose $\hat H_t\ge H$. $\Ex[T_{\infty}|\text{all }\EL\text{ return correctly}]\le16(H(\arment+\ln\delta^{-1}+\ln(\hat H_t/H)))$.
}

$ $

\begin{proof}[Proof of Lemma~\ref{L8}]
Recall that $T_{\infty}$ is the sum of 
\begin{equation}\label{eq5}
|S_r|\eps_{r}^{-2}\ln\left(\frac{|S_r|\eps_r^{-2}}{\hat H_t}\delta\right)^{-1}=|S_r|\eps_{r}^{-2}\left(\ln\frac{H}{|S_r|\eps_{r}^{-2}}+\ln{\delta}^{-1}+\ln\frac{\hat H_t}{H}\right)
\end{equation}
for all round $r$. $T_{\infty}$ serves as an upper bound on the expected number of samples taken by $\ME$ and $\EL$ (up to a constant factor). Before the technical proof, we discuss the intuition of our analysis.

In order to bound $T_{\infty}$, we attribute each term in \eqref{eq5} to a specific subset of arms. For simplicity, we assume for now that this term is just $|S_r|\eps_r^{-2}=4^r|S_r|$. Roughly speaking, we ``charge'' a cost of $\eps_r^{-2}=4^r$ to each arm in group $G_{\ge r}$. 
We expect that $|G_{\ge r}|$ is at least a constant times $|S_r|$, so that the number of samples (i.e., $4^r|S_r|$) can be covered by the total charges. Then the analysis reduces to calculating the total cost that each arm is charged. Fix an arm $A\in G_{r'}$ for some $r'$. As described above, $A$ is charged $4^r$ in round $r$ ($1\le r\le r'$), and thus the total charge is bounded by $4^{r'}$, which is the actual complexity of $A$.

Now we start the formal proof. Consider the execution of procedure $\Proc$ on $\hat H_t$. We define a collection of random variables $\{T_{i,j}:i,j\ge1\}$, where $T_{i,j}$ corresponds to the cost we charge each arm in $G_j$ at round $i$. For each $i$, let $r_i$ denote the largest integer such that $|G_{\ge r_i}|\ge0.5|S_i|$. Note that such an $r_i$ always exists, as $|G_{\ge1}|=|S_1|\ge0.5|S_i|$ and $|G_{\ge r}|=0$ for sufficiently large $r$. We define $T_{i,j}$ as
\begin{equation*}
T_{i,j}=\begin{cases}
0,& j<r_i,\\
\eps_i^{-2}\left(\ln\frac{H}{|G_j|\eps_{i}^{-2}}+\ln{\delta}^{-1}+\ln\frac{\hat H_t}{H}\right), & j\ge r_i.\\
\end{cases}
\end{equation*}

Note that this slightly differs from the proof idea described above: $T_{i,j}$ might be positive when $i>j$ (i.e., we may not always charge $G_{\ge i}$ in round $i$). In fact, the charging argument described in the proof idea works only if, ideally, all calls of $\ME$ are correct. Since actually some $\ME$ may return incorrectly, we have to slightly modify the charging method. Nevertheless, we will show that this difference only incurs a reasonably small cost in expectation.

We first claim that 
\begin{equation}\label{eq8}
T_{\infty}\le2\sum_{i,j}|G_j|\cdot T_{i,j}\text{.}
\end{equation}
In other words, the total cost we charge is indeed an upper bound on $T_{\infty}$. Note that the contribution of round $i$ to $T_{\infty}$ is $|S_i|\eps_i^{-2}\left[\ln(H/(|S_r|\eps_{r}^{-2}))+\ln{\delta}^{-1}+\ln(\hat H_t/H)\right]$, while its contribution to the right-hand side of \eqref{eq8} is
\begin{equation*}\begin{split}
2\sum_{j}|G_j|\cdot T_{i,j}&=2\sum_{j}|G_j|\cdot\eps_i^{-2}\left(\ln(H/(|G_j|\eps_{i}^{-2}))+\ln{\delta}^{-1}+\ln(\hat H_t/H)\right)\\
&\ge2|G_{\ge r_i}|\cdot\eps_i^{-2}\left[\ln(H/(|S_r|\eps_{r}^{-2}))+\ln{\delta}^{-1}+\ln(\hat H_t/H)\right]\\
&\ge|S_i|\eps_i^{-2}\left[\ln(H/(|S_r|\eps_{r}^{-2}))+\ln{\delta}^{-1}+\ln(\hat H_t/H)\right]\text{.}
\end{split}\end{equation*}
Then identity \eqref{eq8} directly follows from a summation on $i$.

Then we bound the expectation of each $T_{i,j}$. When $i\le j$, we have the trivial bound
$$\Ex[T_{i,j}]\le\eps_i^{-2}\left(\ln\frac{H}{|G_j|\eps_{i}^{-2}}+\ln{\delta}^{-1}+\ln\frac{\hat H_t}{H}\right)\text{.}$$
When $i>j$, we bound the probability that $T_{i,j}>0$. By definition, $T_{i,j}>0$ if and only if $r_i\le j$, where $r_i$ is the largest integer that satisfies $|G_{\ge r_i}|\ge0.5|S_i|$. It follows that $T_{i,j}>0$ only if $|G_{\ge j+1}|<0.5|S_i|$.

Observe that in order to have $|G_{\ge j+1}|<0.5|S_i|$, $\ME$ must return incorrectly between round $j+1$ and round $i-1$. In fact, suppose towards a contradiction that $\ME$ is correct in round $k\in[j+1,i-1]$. Then we have
$$|G_{\ge j+1}|\ge |G_{\ge k}|\ge |S_{k+1}\cap G_{\ge k}|>0.5|S_{k+1}|\ge0.5|S_i|\text{,}$$
a contradiction. Here the third step is due to the fact that when $\EL$ returns correctly at round $k$, the fraction of arms in $S_{k+1}$ with gap greater than $2^{-k}$ is less than $0.1$.

Note that for each specific round, the probability that $\ME$ returns incorrectly is at most $0.01$. Thus, the probability that $T_{i,j}>0$ for $i>j$ is upper bounded by $0.01^{i-j-1}$. Therefore, $$\Ex[T_{i,j}]\le0.01^{i-j-1}\eps_i^{-2}\left(\ln\frac{H}{|G_j|\eps_i^{-2}}+\ln\delta^{-1}+\ln\frac{\hat H_t}{H}\right)\text{.}$$

It remains to sum up the upper bounds of $\Ex[T_{i,j}]$ to yield our bound of $\Ex[T_{\infty}]$.
$$\Ex[T_{\infty}]\le2\sum_{i,j}|G_j|\cdot\Ex[T_{i,j}]=2\sum_{i\le j}|G_j|\cdot\Ex[T_{i,j}]+2\sum_{i>j}|G_j|\cdot\Ex[T_{i,j}]\text{.}$$

Here the first part can be bounded by
\begin{equation*}\begin{split}
2\sum_{i\le j}|G_j|\cdot\Ex[T_{i,j}]&\le2\sum_{j}\sum_{i=1}^j|G_j|\cdot4^i\left(\ln\frac{H}{|G_j|4^i}+\ln\delta^{-1}+\ln\frac{\hat H_t}{H}\right)\\
&\le4\sum_{j}|G_j|\cdot4^j\left(\ln\frac{H}{|G_j|4^j}+\ln\delta^{-1}+\ln\frac{\hat H_t}{H}\right)\\
&\le4\sum_{j}\left(H_j\ln\frac{H}{H_j/4}+H_j\ln\delta^{-1}+H_j\ln\frac{\hat H_t}{H}\right)\\
&\le8H\left(\arment+\ln\delta^{-1}+\ln\frac{\hat H_t}{H}\right)\text{.}\\
\end{split}\end{equation*}
The second part can be bounded similarly.
\begin{equation*}\begin{split}
2\sum_{i>j}|G_j|\cdot\Ex[T_{i,j}]&\le2\sum_{j}\sum_{i=j+1}^{\infty}0.01^{i-j-1}|G_j|\cdot4^i\left(\ln\frac{H}{|G_j|4^i}+\ln\delta^{-1}+\ln\frac{\hat H_t}{H}\right)\\
&\le4\sum_{j}|G_j|\cdot4^j\left(\ln\frac{H}{|G_j|4^j}+\ln\delta^{-1}+\ln\frac{\hat H_t}{H}\right)\\
&\le4\sum_{j}\left(H_j\ln\frac{H}{H_j/4}+H_j\ln\delta^{-1}+H_j\ln\frac{\hat H_t}{H}\right)\\
&\le8H\left(\arment+\ln\delta^{-1}+\ln\frac{\hat H_t}{H}\right)\text{.}\\
\end{split}\end{equation*}
In fact, the crucial observation for both the two inequalities above is that the summation decreases exponentially as $i$ becomes farther away from $j$. The lemma directly follows.
\end{proof}

Finally, we prove Lemma~\ref{L9}, which is restated below. Recall that we abuse the notation a little bit by assuming $A_{1}\in G_{\infty}$ and $\Gap{1}^{-2}=+\infty$. 

$ $

\noindent\textbf{Lemma~\ref{L9}.} (restated)\textit{
Suppose that $s\in\{2,3,\ldots,n\}$ and $r^*\in\mathbb{N}\cup\{\infty\}$ satisfy $A_{s-1}\in G_{r^*}$. When \algent{} runs on parameter $\hat H_t<\Gap{s-1}^{-2}$, the probability that there exists a call of procedure $\EL$ that returns incorrectly before round $r^*$ is upper bounded by $$3000s\left(\sum_{i=s}^{n}\Gap{i}^{-2}\right)\delta^2/\hat H_t\text{.}$$
}

$ $

\begin{proof}[Proof of Lemma~\ref{L9}]
Recall that $A_{s-1}\in G_{r^*}$. Suppose $A_s\in G_{r'}$.
Suppose that we are at the beginning of round $r$ of \algent{} and the subset of arms that have not been removed is $S_r=S$. Moreover, we assume that the optimal arm, $A_1$, is still in $S_r$. Let $P(r,S)$ denote the probability that some call of procedure $\EL$ returns incorrectly in round $r,r+1,\ldots,r^*-1$.

As in the proof of Lemma~\ref{L7}, we bound $P(r,S)$ by induction using the potential function method. Define
$$C(r,S)=\sum_{i=r-1}^{r'}|S\cap G_i|\sum_{j=r}^{i+1}\eps_j^{-2}+(s-1)\sum_{j=r}^{r'+2}\eps_j^{-2}$$
and $M(r,S)=|S\cap G_{\le r-2}|$. Then it holds that for $1\le r\le r'+1$,
$$C(r,S)-C(r+1,S)=\sum_{i=r-1}^{r'}|S\cap G_i|\eps_r^{-2}+(s-1)\eps_r^{-2}\ge(|S\cap G_{\ge r-1}|+1)\eps_r^{-2}\text{.}$$
We prove by induction that
\begin{equation}\label{eq4}
P(r,S)\le\left(128C(r,S)+16M(r,S)\eps_r^{-2}\right)\delta^2/\hat H\text{.}
\end{equation}

We first prove the base case at round $r'+2$. If $r'+2\ge r^*$, the bound holds trivially. Otherwise, we consider the ratio
$$\alpha=|S_{r'+2}\cap\{A_s,A_{s+1},\ldots,A_n\}|/|S_{r'+2}|\text{,}$$
which is the fraction of arms at round $r'+2$ that are strictly worse than $A_{s-1}$. Let $r_0$ be the first round after $r'+2$ (inclusive) in which $\ME$ returns correctly. If $\FT$ returns False in round $r_0$, according to Fact~\ref{F3} and the correctness of $\FT$ conditioning on event $\event_1$, we have $\alpha\le\theta_{r_0}$. Consequently, in each of the following rounds (say, round $r>r_0$), $\FT$ always returns False since $\alpha\le\theta_{r_0}\le\theta_{r-1}$, and $\EL$ will never be called before round $r^*$. Note that it is crucial that the threshold interval of $\FT$ in diffrent rounds are disjoint. For the other case, suppose $\FT$ returns True and we call $\EL$ in round $r_0$. Then after that, assuming $\EL$ returns correctly, the fraction of arms worse than $A_{s-1}$ will be smaller than $0.1$. It also follows that $\EL$ will never be called after round $r_0$. Therefore, $\EL$ is called at most once between round $r'+2$ and $r^*-1$, and it can only be called at round $r_0$. Note that for $r\ge r'+2$, $\Pr[r_0=r]\le0.01^{r-r'-2}$. A direct summation on all possible values of $r_0$ yields
\begin{equation*}\begin{split}
P(r'+2,S)&\le\sum_{r=r'+2}^{r^*-1}\Pr[r_0=r]\cdot\delta'_{r}\\
&=\sum_{r=r'+2}^{r^*-1}0.01^{r-r'-2}\cdot4|S|\eps_{r}^{-2}\delta^2/\hat H\\
&\le\left(4|S|\eps_{r'+2}^{-2}\delta^2/\hat H\right)\sum_{k=0}^{\infty}0.01^k4^k\\
&\le5|S|\eps_{r'+2}^{-2}\delta^2/\hat H\text{.}
\end{split}\end{equation*}

Note that $C(r'+2,S)=(s-1)\eps_{r'+2}^{-2}$, $M(r'+2,S)=|S\cap G_{\le r'}|$ 
and $|S|\le|S\cap G_{\le r'}|+(s-1)$.
Thus
\begin{equation*}\begin{split}
P(r'+2,S)&\le5(|S\cap G_{\le r'}|+s-1)\eps_{r'+2}^{-2}\delta^2/\hat H\\
&\le\left(128C(r'+2,S)+16M(r'+2,S)\eps_{r'+2}^{-2}\right)\delta^2/\hat H\text{,}
\end{split}\end{equation*}
which proves the base case.

Then we proceed to the induction step. Again, we consider whether $\ME$ returns correctly and whether $\FT$ returns True. Let $\Ncur=|S\cap G_{r-1}|$ and $\Nbig=|S\cap G_{\ge r}|$. Again, we denote $M(r,S)$ by $\Nsma$ for brevity. Note that $S_{r+1}$ is the set of arms that survive round $r$.

\textbf{Case 1:} $\ME$ returns correctly and $\FT$ returns True.

In this case, $\EL$ is called with confidence level $\delta'_r$. Then the conditional probability that some $\EL$ returns incorrectly in this case is bounded by
\begin{equation*}\begin{split}
&P(r+1,S_{r+1})+\delta'_r\\
\le&\left[128C(r+1,S_{r+1})+16M(r+1,S_{r+1})\eps_{r+1}^{-2}+4|S|\eps_r^{-2}\right]\delta^2/\hat H\\
\le&\left[128C(r+1,S)+64(\Nsma+\Ncur)\eps_r^{-2}/10+4|S|\eps_r^{-2}\right]\delta^2/\hat H\\
=&\left[128C(r,S)-128(\Ncur+\Nbig+s-1)\eps_r^{-2}+(6.4\Nsma+6.4\Ncur+4|S|)\eps_r^{-2}\right]\delta^2/\hat H\\
\le&[128C(r,S)+10.4M(r,S)\eps_r^{-2}]\delta^2/\hat H\text{.}
\end{split}\end{equation*}

\textbf{Case 2:} $\ME$ returns correctly and $\FT$ returns False.

Since $\FT$ returns False, according to Fact~\ref{F3} and Observation~\ref{O5}, we have $\Nsma\le|S|-\Nsma=\Ncur+\Nbig+1$. Then the conditional probability in this case is bounded by
\begin{equation*}\begin{split}
P(r+1,S)&\le[128C(r+1,S)+16M(r+1,S)\eps_{r+1}^{-2}]\delta^2/\hat H\\
&\le[128C(r,S)-128(\Ncur+\Nbig+s-1)\eps_r^{-2}+64(\Nsma+\Ncur)\eps_r^{-2}]\delta^2/\hat H\\
&\le[128C(r,S)+(64\Nsma-64\Ncur-128\Nbig-128(s-1))\eps_r^{-2}]\delta^2/\hat H\\
&\le128C(r,S)\eps_r^{-2}\delta^2/\hat H\text{.}
\end{split}\end{equation*}
Here the last step follows from $64\Nsma-64\Ncur-128\Nbig-128(s-1)\le64(\Nsma-\Ncur-\Nbig-1)\le0$.

\textbf{Case 3:} $\ME$ returns incorrectly.

In this case, the worst scenario is that we call $\EL$ with confidence $\delta'_r\le4|S|\eps_r^{-2}\delta^2/\hat H$, yet no arms are removed. So the conditional probability in this case is bounded by
\begin{equation*}\begin{split}
&P(r+1,S)+4|S|\eps_r^{-2}\delta^2/\hat H\\
\le&\left[128C(r+1,S)+16M(r+1,S)\eps_{r+1}^{-2}+4|S|\eps_r^{-2}\right]\delta^2/\hat H\\
\le&[128C(r,S)-128(\Ncur+\Nbig+s-1)\eps_r^{-2}+64(\Nsma+\Ncur)\eps_r^{-2}+4(\Nsma+\Ncur+\Nbig+1)\eps_r^{-2}]\delta^2/\hat H\\
\le&[128C(r,S)+(68\Nsma-60\Ncur-124\Nbig-124)\eps_r^{-2}]\delta^2/\hat H\\
\le&\left[128C(r,S)+68M(r,S)\eps_r^{-2}\right]\delta^2/\hat H\text{.}
\end{split}\end{equation*}

Recall that Case 3 happens with probability at most $0.01$. Thus we have:

\begin{equation*}\begin{split}
P(r,S)&\le 0.01\left[128C(r,S)+68M(r,S)\eps_r^{-2}\right]\delta^2/\hat H+0.99\left[128C(r,S)+10.4M(r,S)\eps_r^{-2}\right]\delta^2/\hat H\\
&\le\left[128C(r,S)+16M(r,S)\eps_r^{-2}\right]\delta^2/\hat H\text{.}
\end{split}\end{equation*}

The induction is completed. It follows from \eqref{eq4} that
\begin{equation*}\begin{split}
P(1,S)&\le128\left[\sum_{i=0}^{r'}|G_i|\sum_{j=1}^{i+1}\eps_j^{-2}+(s-1)\sum_{j=1}^{r'+2}\eps_j^{-2}\right]\delta^2/\hat H\\
&\le128\left[(4/3)\sum_{i=0}^{r'}|G_i|4^{i+1}+(4/3)(s-1)4^{r'+2}\right]\delta^2/\hat H\\
&\le128\left[(16/3)\sum_{i=s}^{n}\Gap{i}^{-2}+(64/3)(s-1)4^{r'} \right]\delta^2/\hat H\\
&\le3000s\left(\sum_{i=s}^{n}\Gap{i}^{-2}\right)\delta^2/\hat H\text{.}
\end{split}\end{equation*}
\end{proof}

\subsection{Proof of Lemma~\ref{L1}}

Recall that $t_{\max}=\lfloor\log_{100}H\rfloor-2$ and $t'_{\max}=\lceil\log_{100}[H(\arment+\ln\delta^{-1})\delta^{-1}]\rceil+2$. We restate and prove Lemma~\ref{L1} in the following.

$ $

\noindent\textbf{Lemma~\ref{L1}.} (restated)\textit{
With probability $1-\delta/3$ conditioning on event $\event_1$, \algguess{} halts before or at iteration $t'_{\max}$ and it never returns a sub-optimal arm between iteration $t_{\max}+1$ and $t'_{\max}$.
}

$ $

The high-level idea of the proof is to construct three other ``good events'' $\event_2$, $\event_3$ and $\event_4$. We show that each event happens with high probability conditioning on $\event_1$. Moreover, events $\event_1$ through $\event_4$ together imply the desired event.

\begin{proof}
Recall that $t_{\max}=\lfloor\log_{100}H\rfloor-2$ and $t'_{\max}=\lceil\log_{100}[H(\arment+\ln\delta^{-1})\delta^{-1}]\rceil+2$. Let $\event_2$ denote the following event: for all $t$ such that $t\ge t_{\max}+1$ and $\hat H_t<100^3H$, \algent{} either rejects or outputs the optimal arm. Since $\hat H_{t_{\max}+1}=100^{t_{\max}+1}\ge H/10000$, there are at most $\log_{100}[100^3H/(H/10000)]+1=6$ different values of such $\hat H_t$. For each $\hat H_t$, the probability of returning a sub-optimal arm is bounded by the probability that the optimal arm is deleted, which is in turn upper bounded by $\delta^2$ as a corollary of Lemma~\ref{L3} proved in the following section.

Thus, by a union bound,
$$\Pr[\event_2|\event_1]\ge1-6\delta^2\text{.}$$

Let $\event_3$ be the event that for all $\hat H_t$ such that $t\le t'_{\max}$ and $\hat H_t\ge100^3H$ (or equivalently, $\lceil\log_{100}H\rceil+3\le t\le t'_{\max}$), \algent{} never returns an incorrect answer. In fact, in order for \algent{} to return incorrectly, some call of $\EL$ must be wrong. Thus we may apply Lemma~\ref{L9} to bound the probability of $\event_3$. Specifically, we apply Lemma~\ref{L9} with $s=2$. Then we have
\begin{equation*}\begin{split}
\Pr[\event_3|\event_1]&\ge1-\sum_{t=\lceil\log_{100}H\rceil+3}^{t'_{\max}}\frac{3000s\left(\sum_{i=s}^{n}\Gap{i}^{-2}\right)\delta^2}{\hat H_t}\\
&\ge1-\sum_{t=\lceil\log_{100}H\rceil+3}^{\infty}\frac{6000H}{100^t}\delta^2\\
&\ge1-\sum_{k=3}^{\infty}\frac{6000}{100^k}\delta^2\ge1-\delta^2/100\text{.}
\end{split}\end{equation*}
Here the third step is due to the simple fact that $100^{\lceil\log_{100}H\rceil}\ge H\text{.}$

Finally, let $\event_4$ denote the event that when \algent{} runs on $\hat H_{t'_{\max}}$, no $\EL$ is wrong and the algorithm finally accepts. In order to bound the probability of the last event, we simply apply Markov inequality based on Lemma~\ref{L7} and Lemma~\ref{L8}. Let $\event_0$ be the event that no $\EL$ is wrong when \algent{} runs on $\hat H_{t'_{\max}}$. Then we have
\begin{equation*}\begin{split}
\Pr[\event_4|\event_1]&\ge\Pr[\event_0|\event_1]-\frac{\Ex[H_{\infty}|\event_0]}{\hat H_{t'_{\max}}}-\frac{\Ex[T_{\infty}|\event_0]}{100\hat H_{t'_{\max}}}\\
&\ge1-\delta^2-\frac{256H}{100^2H(\arment+\ln\delta^{-1})\delta^{-2}}-\frac{16H\left[\arment+\ln\delta^{-1}+\ln(\hat H_{t'_{\max}}/H)\right]}{100^3H(\arment+\ln\delta^{-1})\delta^{-2}}\\
&\ge1-\delta^2-\frac{256}{100^2}\delta^2-\frac{16\left[\arment+3\ln\delta^{-1}+\ln(100^2(\arment+\ln\delta^{-1}))\right]}{100^3(\arment+\ln\delta^{-1})}\delta^2\\
&\ge1-2\delta^2\text{.}
\end{split}\end{equation*}

Note that conditioning on events $\event_1$ through $\event_4$, \algent{} never outputs an incorrect answer between iteration $t_{\max}+1$ and $t'_{\max}$. Moreover, our algorithm terminates before or at iteration $t'_{\max}$. The lemma directly follows from a union bound and the observation that for all $\delta\in(0,0.01)$, $$6\delta^2+\delta^2/100+2\delta^2\le\delta/3\text{.}$$
\end{proof}

\begin{Remark}\label{R1}
The last part of the proof implies a more general fact: for fixed $\hat H_t$, \algent{} accepts with probability at least
$$1-\delta^2-\frac{256H}{\hat H_t}-\frac{16H(\arment+\ln\delta^{-1}+\ln(\hat H_t/H))}{100\hat H_t}\text{.}$$
\end{Remark}

\subsection{Mis-deletion of Arms}

We prove Lemma~\ref{L2} in the following. Again, our analysis 
in this subsection conditions on event $\event_1$, which guarantees that all calls of $\FT$ and $\US$ in \algent{} are correct. The high-level idea of the proof is to show that a large proportion of arms will not be accidentally removed before they have contributed a considerable amount to the total complexity. Formally, we define the mis-deletion of arms as follows.

\begin{Definition}
An arm $A\in G_r$ is \textbf{mis-deleted} in a particular run of \algent{}, if $A$ is deleted before or at round $r-1$. In particular, the optimal arm is \textbf{mis-deleted} if it is deleted in any round.
\end{Definition}

The following lemma bounds the probability that a certain collection of arms are all mis-deleted.

\begin{Lemma}\label{L3}
For a fixed collection of $k$ arms, the probability that all of them are mis-deleted is at most $\delta^{2k}$.
\end{Lemma}

\begin{proof}
Let $S=\{A_1,A_2,\ldots,A_k\}$ be a fixed set of $k$ arms. (Here we temporarily drop the convention that $A_i$ denotes the arm with the $i$-th largest mean.) For each $A_i$, let $\Ebad_i$ denote the event that $A_i$ is mis-deleted, and let $r_i$ denote the group that contains $A_i$ (i.e., $A_i\in G_{r_i}$). By definition, $\mu_{A_i}\ge\Mean{1}-\eps_{r_i}$.

We start by proving the following fact: suppose $\EL$ is called with confidence level $\delta'_r$ in round $r$. Then the probability that all arms in $S$ are mis-deleted in round $r$ simultaneously is bounded by ${\delta'}_r^k$.

We assume that $r<r_i$ for all $i=1,2,\ldots,k$. Otherwise, if $r\ge r_i$ for some $i$, then $A_i$ cannot be \emph{mis-deleted} in round $r$, since the definition of mis-deletion requires that $r<r_i$. To analyze the behaviour of $\EL$, we recall that each run of $\EL$ consists of several stages. (Here we use the term ``stage'' for an iteration of $\EL$, while the term for \algent{} is ``round''.) In each stage, procedure $\US$ is called at line 6 to estimate the means of the arms that have not been eliminated. Let $\rbad_i$ denote the stage in which $A_i$ gets deleted. 

Recall that $\dhigh_r$ is the upper threshold used in $\EL$ in round $r$. According to Observation~\ref{O6},
$$\dhigh_r\le\Mean{1}(S_r)-0.5\eps_r=\Mean{1}(S_r)-2^{-(r+1)}\le\Mean{1}-\eps_{r_i}\le\mu_{A_i}\text{.}$$
Here the third step follows from our assumption that $r<r_i$. In order for $\EL$ to eliminate an arm $A_i$ with mean greater than $\dhigh$ in stage $\rbad_i$, the $\US$ subroutine must return an incorrect estimation for $A_i$ (i.e., $|\hat\mu_{A_i}-\mu_{A_i}|>(\dhigh-d^{\mi})/2$), which happens with probability at most $\delta'_r/\left(10\cdot2^{\rbad_i}\right)$. Since the samples taken on different arms are independent, the events that $\US$ returns incorrect estimates for different arms are also independent, and it follows that the probability that each arm $A_i$ is removed at stage $\rbad_i$ is bounded by $\prod_{i=1}^{k}\left(\delta'_r/\left(10\cdot2^{\rbad_i}\right)\right)$.

Therefore, the probability that all the $k$ arms in $S$ are mis-deleted in $\EL$ is upper bounded by
\begin{equation*}\begin{split}
&\sum_{\rbad_1=1}^{\infty}\sum_{\rbad_2=1}^{\infty}\cdots\sum_{\rbad_k=1}^{\infty}\prod_{i=1}^{k}\left(\delta'_r/(10\cdot2^{\rbad_i})\right)\\
=&\prod_{i=1}^{k}\left[\sum_{\rbad_i=1}^{\infty}\left(\delta'_r/\left(10\cdot2^{\rbad_i}\right)\right)\right]\\
\le&\prod_{i=1}^k\delta'_r={\delta'}_r^{k}\text{.}
\end{split}\end{equation*}

Then we start with the proof of the lemma. Suppose that we are at the beginning of round $r$. $m$ arms among $S$ are still in $S_r$, while the sum of confidence levels allocated in the previous rounds is $\delta'$ (i.e., $\delta'=\sum_{i=1}^{r-1}\delta'_i$). Let $P(r,\delta',m)$ denote the probability that all the $m$ remaining arms are mis-deleted in the future. We prove by induction that
\begin{equation}\label{eq11}
P(r,\delta',m)\le(\delta^2-\delta')^{m}\text{.}
\end{equation}

Recall that the number of rounds that \algent{} lasts is bounded by $ct$ according to Observation~\ref{O1}. Thus when $r=\lceil ct\rceil+1$, we have $P(r,\delta',m)=0$. Observe that $\delta'$ never exceeds $\delta^2$ according to the behaviour of \algent{}. Therefore \eqref{eq11} holds for the base case. Now we proceed to the induction step. If $\EL$ is not called in round $r$, by induction hypothesis we have
$$P(r,\delta',m)\le P(r+1,\delta',m)\le(\delta^2-\delta')^{m}\text{,}$$
which proves inequality \eqref{eq11}.
If, on the other hand, $\EL$ is called with confidence $\delta'_r$. We observe that by the claim we proved above, the probability that exactly $j$ arms among the $m$ arms are mis-deleted is at most
$\dbinom{m}{j}{\delta'}_r^{j}\text{.}$
Thus by a simple summation,
$$P(r,\delta',m)\le\sum_{j=0}^{m}\binom{m}{j}{\delta'}_r^{j}\cdot P(r+1,\delta'+\delta'_r,m-j)\le\sum_{j=0}^{m}\binom{m}{j}{\delta'}_r^{j}(\delta^2-\delta'-\delta'_r)^{m-j}=(\delta^2-\delta')^{m}\text{,}$$
which completes the induction step.

Finally, the lemma directly follows from \eqref{eq11} by plugging in $r=1$, $\delta'=0$ and $m=k$.
\end{proof}

\begin{Remark}\label{R2}
Let $\Ebad_i$ denote the event that $A_i$ is mis-deleted. Note that although the events $\{\Ebad_i\}$ are not independent, we can still obtain an exponential bound (i.e., $\delta^{2k}$) on the probability that $k$ such events happen simultaneously. 
We call such events \textbf{quasi-independent} to reflect this property. Formally, a collection of $n$ events $\{\event_i\}_{i=1}^n$ are $\delta$-quasi-independent, if for all $1\le k\le n$ and $1\le a_1<a_2<\cdots<a_k\le n$, we have
$$\Pr[\event_{a_1}\cap\event_{a_2}\cap\cdots\cap\event_{a_k}]\le\delta^k\text{.}$$
Then the collection of events $\{\Ebad_i\}$ are $\delta^2$-quasi-independent.
\end{Remark}

The following lemma proves a generalized Chernoff bound for quasi-independent events.

\begin{Lemma}\label{LCher}
Suppose $v_1,v_2,\ldots,v_n>0$. $\{Y_i\}_{i=1}^{n}$ is a collection of random variables, where the support of $Y_i$ is $\{0,v_i\}$. Moreover, the collection of events $\{Y_i=v_i\}$ are $\delta$-quasi-independent. Let $(S_1,S_2,\ldots,S_m)$ be a partition of $\{1,2,\ldots,n\}$ such that $\sum_{j\in S_i}v_j\le 1$ for all $i$. Define $X_i=\sum_{j\in S_i}Y_j$. Let $X=\frac{1}{m}\sum_{i=1}^mX_i$ and $p=\frac{\delta}{m}\sum_{i=1}^{n}v_i$. Then for all $q\in(p,1)$,
$$\Pr[X\ge q]\le e^{-mD(q||p)}\text{,}$$
where
$$D(x||y)=x\ln(x/y)+(1-x)\ln[(1-x)/(1-y)]$$
is the relative entropy function.
\end{Lemma}

\begin{proof}
Let $p_i=\delta\sum_{j\in S_i}v_j$. Then $p=\frac{1}{m}\sum_{i=1}^{m}p_i$. For $t>0$, we have
$$\Pr[X\ge q]=\Pr[e^{tmX}\ge e^{tmq}]\le\frac{\Ex[e^{tmX}]}{e^{tmq}}\text{.}$$

To bound $\Ex[e^{tmX}]$, we consider a collection of \emph{independent} random variables $\tilde Y_1,\tilde Y_2,\ldots,\tilde Y_n$ defined by $\Pr[\tilde Y_i=v_i]=\delta$ and $\Pr[\tilde Y_i=0]=1-\delta$. Define $\tilde X_i=\sum_{j\in S_i}\tilde Y_j$ for $i=1,2,\ldots,m$, and $\tilde X=\frac{1}{m}\sum_{i=1}^{m}\tilde X_i$. Note that each term in the Taylor expansion of $e^{tmX}$ can be written as $\alpha\prod_{i=1}^{l}Y_{n_l}$, where $l\ge0$, $(n_1,n_2,\ldots,n_l)\in\{1,2,\ldots,n\}^l$, and $\alpha=t^l/(l!)>0$. The corresponding term in $e^{tm\tilde X}$ is then $\alpha\prod_{i=1}^{l}\tilde Y_{n_l}$. Let $U=|\{n_i:i\in\{1,2,\ldots,l\}\}|$ denote the set of distinct numbers among $n_1,n_2,\ldots,n_l$. We have
$$\Ex\left[\prod_{i=1}^{l}Y_{n_l}\right]=\Pr[Y_i=v_i\text{ for all }i\in U]\cdot\prod_{i=1}^{l}v_{n_l}\le\delta^{|U|}\prod_{i=1}^{l}v_{n_l}=\Ex\left[\prod_{i=1}^{l}\tilde Y_{n_l}\right]\text{.}$$
Summing over all terms in the expansion yields
$$\Ex\left[e^{tmX}\right]\le\Ex\left[e^{tm\tilde X}\right]=\prod_{i=1}^{m}\Ex\left[e^{t\tilde X_i}\right]\text{.}$$
Here the last step holds since $\{\tilde X_i\}$ are independent. Note that since $\tilde X_i\in[0,1]$, it follows from Jensen's inequality that
$$\Ex\left[e^{t\tilde X_i}\right]\le\Ex\left[e^{t}\tilde X_i+1-\tilde X_i\right]=p_ie^{t}+1-p_i\text{.}$$
Then
$$\Ex\left[e^{tmX}\right]\le\prod_{i=1}^{m}(p_ie^{t}+1-p_i)\le(pe^t+1-p)^m\text{.}$$
Recall that $p=\frac{1}{m}\sum_{i=1}^{m}p_i$. Here the last step follows from Jensen's inequality and the concavity of $\ln(e^tx+1-x)$ for $t>0$.

By setting $t=\ln\frac{q(1-p)}{p(1-q)}$, we have
$$\Pr[X\ge q]\le\frac{\Ex[e^{tmX}]}{e^{tmq}}\le\left[\frac{pe^t+1-p}{e^{tq}}\right]^{m}=e^{-mD(q||p)}\text{.}$$
\end{proof}

The following lemma states that if a collection of arms with a considerable amount of total complexity are not mis-deleted, \algent{} rejects $\hat H$.

\begin{Lemma}\label{L4}
$S$ is a set of sub-optimal arms with complexity $H(S)>\hat H$. Let $r^*=\max_{A\in S}\left\lfloor\log_2\Delta_A^{-1}\right\rfloor$. If in a particular run of \algent{}, no arm in $S$ is mis-deleted and there exists an arm $A^*$ outside $S$ with $\mu_{A^*}\ge\max_{A\in S}\mu_{A}$ such that $A^*$ is not deleted in the first $r^*-1$ rounds, then $\hat H$ is rejected in that run.
\end{Lemma}

\begin{proof}
Suppose $S=\{A_1,A_2,\ldots,A_k\}$ and $A_i\in G_{r_i}$. Without loss of generality, $\mu_{A_1}\le\mu_{A_2}\le\cdots\le\mu_{A_k}$. By definition of $r^*$, we have $r^*=\max_{1\le i\le k}r_i=r_k$. According to our assumption, both $A_k$ and $A^*$ are not deleted in the first $r^*-1$ rounds. Thus \algent{} does not accept in the first $r^*$ rounds.

Suppose for contradiction that $\hat H$ is not rejected by \algent{} in a particular run. Define $\mathcal{R}=\{r\in[1,r^*-1]:\EL\text{ is called in round }r\}\text{.}$ Let $N_1=\{i\in[k]:\exists r\in\mathcal{R},r\ge r_i\}$ and $N_2=[k]\setminus N_1$. For each $i\in N_1$, since $A_i$ is not mis-deleted, $A_i\in S_{r_i}$. Define $r'_i=\min\{r\in\mathcal{R}:r\ge r_i\}$ as the first round after $r_i$ (inclusive) in which $\EL$ is called. It follows that $A_i\in S_{r'_i}$. At round $r'_i$ of \algent{}, $H_{r'_i+1}$ is set to $H_{r'_i}+4|S_{r'_i}|\eps_{r'_i}^{-2}$. Therefore we can ``charge'' $A_i$ a cost of $4\eps_{r'_i}^{-2}=\eps_{r'_i+1}^{-2}$. It follows that $H_{r^*}$ is at least the total cost that arms in $N_1$ are charged, $\sum_{i\in N_1}\eps_{r'_i+1}^{-2}$.

For each $i\in N_2$, we have $A_i\in S_{r_i}$ and $S_{r_i}=S_{r^*}$. Thus it holds that $|S_{r^*}|\ge|N_2|$. When the if-statement in \algent{} is checked in round $r^*$, we have
$$H_{r^*}+4|S_{r^*}|\eps_{r^*}^{-2}\ge\sum_{i\in N_1}\eps_{r'_i+1}^{-2}+N_2\eps_{r^*+1}^{-2}\ge\sum_{i=1}^k\eps_{r_i+1}^{-2}\ge\sum_{i=1}^k\Delta_{A_i}^{-2}=H(S)>\hat H\text{.}$$
Here the second step follows from $r'_i\ge r_i$ and $r^*\ge r_i$, while the third step follows from $\Delta_{A_i}\ge2^{-(r_i+1)}$. Therefore \algent{} rejects in round $r^*$, a contradiction.
\end{proof}

\subsection{Proof of Lemma~\ref{L2}}

Lemma~\ref{L2} is restated below. Recall that $t_{\max}=\lfloor\log_{100}H\rfloor-2$.

$ $

\noindent\textbf{Lemma~\ref{L2}.} (restated)\textit{
With probability $1-\delta/3$ conditioning on event $\event_1$, \algguess{} never returns a sub-optimal arm in the first $t_{\max}$ iterations.
}

$ $

The high-level idea of the proof is simple. For each $\hat H_t$, we identify a collection of near-optimal ``crucial arms''. By Lemma~\ref{L3}, the probability that all ``crucial arms'' are mis-deleted is small, thus we may assume that at least one crucial arm survives. This crucial arm serves as $A^*$ in Lemma~\ref{L4}. Then according to Lemma~\ref{L4}, in order for \algent{} to accept $\hat H_t$, it must mis-delete a collection of ``non-crucial'' arms with a significant fraction of complexity. The probability of this event can also be bounded by using the generalized Chernoff bound proved in Lemma~\ref{LCher}.

The major technical difficulty is the choice of ``crucial arms''. We deal the following three cases separately: (1) $\hat H_t$ is greater than $\Gap{2}^{-2}$, the complexity of the arm with the second largest mean; (2) $\hat H_t$ is between $\Gap{s}^{-2}$ and $\Gap{s-1}^{-2}$ for some $3\le s\le n$; and (3) $\hat H_t$ is smaller than $\Gap{n}^{-2}$. We bound the probability that the lemma is violated in each case, and sum them up using a union bound.

\begin{proof}[Proof of Lemma~\ref{L2}]

\textbf{Case 1:} $\Gap{2}^{-2}\le\hat H_t\le\hat H_{t_{\max}}$.

We first deal with the case that $\hat H_t$ is relatively large. We partition the sequence of sub-optimal arms $A_2,A_3,\ldots,A_n$ into contiguous blocks $B_1,B_2,\ldots,B_m$ such that the total complexity in each block $B_i$, denoted by $H(B_i)=\sum_{A\in B_i}\Delta_{A}^{-2}$, is between $\Gap{2}^{-2}$ and $3\Gap{2}^{-2}$. To construct such a partition, we append arms to the current block one by one from $A_2$ to $A_n$. When the complexity of the current block exceeds $\Gap{2}^{-2}$, we start with another block. Clearly, the complexity of each resulting block is upper bounded by $2\Gap{2}^{-2}$. Note that the last block may have a complexity less than $\Gap{2}^{-2}$. In that case, we simply merge it into the second last block. As a result, the total complexity of every block is in $\left[\Gap{2}^{-2},3\Gap{2}^{-2}\right]$. It follows that $H\in\left[m\Gap{2}^{-2},3m\Gap{2}^{-2}\right]$.

For brevity, let $B_{\le i}$ denote $B_1\cup B_2\cup\cdots\cup B_i$ and $B_{<i}=B_{\le i-1}$. Since $H(B_1)=\Gap{2}^{-2}\le\hat H_t<H=H(B_{\le m})$, there exists a unique integer $k\in[2,m]$ that satisfies $H(B_{<k})\le\hat H_t< H(B_{\le k})$. Then we have $\hat H_t\in\left[(k-1)\Gap{2}^{-2},3k\Gap{2}^{-2}\right]\text{.}$ Since $B_{\le k}$ contains at least $k$ arms, it follows from Lemma~\ref{L3} that with probability $1-\delta^{2k}$, at least one arm in $B_{\le k}$ is not mis-deleted. Recall that by Lemma~\ref{L4}, \algent{} accepts $\hat H_t$ only if either of the following two events happens: (a) no arm in $B_{\le k}$ survives, which happens with probability $\delta^{2k}$; (b) a collection of arms among $B_{>k}$ with total complexity of at least $H(B_{>k})-\hat H$ are mis-deleted.

For $i=2,3,\ldots,n$, define $v_i=\Gap{i}^{-2}/(3\Gap{2}^{-2})$ and $Y_i=v_i\cdot\mathbb{I}[A_i\text{ is mis-deleted}]$. For $i=1,2,\ldots,m$, $X_i$ is defined as
$$X_i=\sum_{A_j\in B_i}Y_j=\sum_{A\in B_i}\frac{\Delta_{A}^{-2}}{3\Gap{2}^{-2}}\cdot\mathbb{I}[A\text{ is mis-deleted}]\text{.}$$
In other words, $X_i$ is the total complexity of the arms in block $B_i$ that are mis-deleted, divided by a constant $3\Gap{2}^{-2}$. Recall that $H(B_i)\le3\Gap{2}^{-2}$, so $X_i$ is between $0$ and $1$. Let
$$X=\frac{1}{m}\sum_{i=1}^{m}X_i=\frac{1}{3m\Gap{2}^{-2}}\sum_{i=2}^{n}\Gap{i}^{-2}\cdot\mathbb{I}[A_i\text{ is mis-deleted}]$$
denote the mean of these random variables. Since the events of mis-deletion of arms are $\delta^2$-quasi-independent, we may apply Lemma~\ref{LCher}. 
Note that
$$p=\frac{\delta^2}{m}\sum_{i=2}^{n}v_i=\frac{\delta^2}{m}\sum_{i=2}^{n}\frac{\Gap{i}^{-2}}{3\Gap{2}^{-2}}=\frac{H\delta^{2}}{3m\Gap{2}^{-2}}\le\delta^{2}\text{.}$$
Here the last step applies $H\le3m\Gap{2}^{-2}$.
On the other hand, conditioning on event (b) (i.e., a collection of arms with total complexity $H(B_{>k})-\hat H$ are mis-deleted), we have
\begin{equation*}\begin{split}
X&=\frac{1}{3m\Gap{2}^{-2}}\sum_{i=2}^{n}\Gap{i}^{-2}\cdot\mathbb{I}[A_i\text{ is mis-deleted}]\\
&\ge\frac{H(B_{>k})-\hat H}{3m\Gap{2}^{-2}}\ge\frac{(m-k)\Gap{2}^{-2}-3k\Gap{2}^{-2}}{3m\Gap{2}^{-2}}\\
&\ge\frac{m-4k}{3m}\ge\frac{m-12m/10000}{3m}\ge\frac{1}{6}\text{.}
\end{split}\end{equation*}
Here the third step follows from $H(B_{>k})\ge(m-k)\Gap{2}^{-2}$ and $\hat H\le 3k\Gap{2}^{-2}$. The last line holds since
$$k\Gap{2}^{-2}\le\hat H\le\hat H_{t_{\max}}\le H/10000\le 3m\Gap{2}^{-2}/10000\text{,}$$
which implies $k\le3m/10000$.

According to Lemma~\ref{LCher}, we have
\begin{equation*}\begin{split}
\Pr[X\ge1/6]&\le\exp\left(-mD\left(1/6||\delta^2\right)\right)\\
&=\exp\left(-\frac{m}{6}\ln\frac{1}{6\delta^2}-\frac{5m}{6}\ln\frac{5}{6(1-\delta^2)}\right)\\
&\le(6\delta^2)^{m/6}\cdot(6/5)^{5m/6}\le\delta^{m/6}\text{.}
\end{split}\end{equation*}
Recall that $D(x||y)$ stands for the relative entropy function. The last step follows from $6\delta\cdot(6/5)^{5}\le1$.

Therefore, \begin{equation}\label{eq3}
\Pr\left[\text{\algent{} accepts }\hat H_t\right]\le\delta^{2k}+\delta^{m/6}\text{.}
\end{equation}

It remains to apply a union bound to \eqref{eq3} for all values of $\hat H_t$ in $\left[\Gap{2}^{-2},\hat H_{t_{\max}}\right]$. Recall that $k\ge2$, and the ratio between different guesses $\hat H_t$ is at least $100$. It follows that the values of $k$ are distinct for different values of $\hat H_t$, and thus the sum of the first term, $\delta^{2k}$, can be bounded by
$$\sum_{k=2}^{\infty}\delta^{2k}=\frac{\delta^4}{1-\delta^2}\le2\delta^4\text{.}$$
For the second term, we note that the number of guesses $\hat H_t$ between $\Gap{2}^{-2}$ and $\hat H_{t_{\max}}$ is at most
$$t_{\max}-\left\lceil\log_{100}\Gap{2}^{-2}\right\rceil+1\le\log_{100}H-2-\log_{100}\Gap{2}^{-2}+1=\log_{100}\frac{H}{\Gap{2}^{-2}}-1\le\log_{100}(3m)-1\text{.}$$
In particular, if $m<100^2/3$, no $\hat H_t$ will fall into $[\Gap{2}^{-2},t_{\max}]$. Thus we focus on the nontrivial case $m\ge100^2/3$. Then the sum of the second term $\delta^{m/6}$ can be bounded by
$$\delta^{m/6}\cdot(\log_{100}(3m)-1)\le\delta^{100^2/18}\text{,}$$
since $\delta^{m/6}\cdot(\log_{100}(3m)-1)$ decreases on $[100^2/3,+\infty)$ for $\delta\in(0,0.01)$.
Finally, we have 
\begin{equation*}
\Pr\left[\text{\algent{} accepts }\hat H_t\text{ for some }\hat H_t\in[\Gap{2}^{-2},H_{t_{\max}}]\right]\le2\delta^4+\delta^{100^2/18}\le3\delta^4\text{.}
\end{equation*}

\textbf{Case 2:} $\Gap{s}^{-2}\le\hat H<\Gap{s-1}^{-2}$ for some $3\le s\le n$.

In this case, $\hat H$ is between the complexity of $A_{s-1}$ and $A_s$. Our goal is to prove an upper bound of $\delta^{\Omega(s)}$ on the probability of returning a sub-optimal arm for each specific $s$. Summing over all $s$ yields a bound on the total probability. Our analysis depends on the ratio between $\hat H$ and $\sum_{i=s}^{n}\Gap{i}^{-2}$, the complexity of arms that are worse than $A_s$. Intuitively, when $\hat H$ is greater than the sum (Case 2-1), the contribution of the arms worse than $A_s$ to the complexity is negligible. Thus we have to rely on the fact that the $s-1$ arms with the largest means will not be mis-deleted simultaneously with high probability. On the other hand, when $\hat H$ is significantly smaller than the sum (Case 2-2), we may apply the same analysis as in Case 1. Finally, if the value of $\hat H$ is between the two cases (Case 2-3), it suffices to prove a relatively loose bound, since the number of possible values is small.

\textbf{Case 2-1:} $\hat H>300000s\sum_{i=s}^{n}\Gap{i}^{-2}$.

In this case, our guess $\hat H$ is significantly larger than the total complexity of $A_s,A_{s+1},\ldots,A_{n}$, yet $\hat H$ is smaller than the complexity of any one among the remaining arms. Thus intuitively, in order to reject $\hat H$, \algent{} should not mis-delete all the first $s-1$ arms. More formally, we have the following fact: in order for \algent{} to return a sub-optimal arm, it must delete $A_1$ along with at least $s-3$ arms among $A_2,A_3,\ldots,A_{s-1}$ before round $r^*$, where $r^*$ is the group that contains $A_{s-1}$. In fact, since $4\eps_{r^*}^{-2}=4^{r^*+1}\ge\Gap{s-1}^{-2}\ge\hat H_t$, \algent{} terminates before or at round $r^*$. If $A_1$ is not deleted before round $r^*$, \algent{} can only return $A_1$ as the optimal arm, which is correct. If less than $s-3$ arms among $A_2,A_3,\ldots,A_{s-1}$ are deleted before round $r^*$, for example $A_i$ and $A_j$ are not deleted ($2\le i<j\le s-1$), then both of them are contained in $S_{r^*}$. It follows that \algent{} does not return before round $r^*$.

We first bound the probability that $A_1$ is deleted before round $r^*$. In order for this to happen, some $\EL$ must return incorrectly. By Lemma~\ref{L9}, the probability of this event is upper bounded by
$$3000s\left(\sum_{i=s}^{n}\Gap{i}^{-2}\right)\delta^2/\hat H_t\text{.}$$
In fact, we have a more general fact: the probability that a fixed set of $k$ arms among $\{A_2,A_3,\ldots,A_{s-1}\}$ together with $A_1$ are deleted before round $r^*$ is bounded by
$$3000s\left(\sum_{i=s}^{n}\Gap{i}^{-2}\right)\delta^2/\hat H_t\cdot\delta^{2k}=3000s\left(\sum_{i=s}^{n}\Gap{i}^{-2}\right)\delta^{2(k+1)}/\hat H_t\text{.}$$
The proof follows from combining the two inductions in the proof of Lemma~\ref{L9} and Lemma~\ref{L3}, and we omit it here.
Since $\{A_2,A_3,\ldots,A_{s-1}\}$ contains $s-2$ subsets of size $s-3$, the probability that \algent{} returns an incorrect answer on a particular guess $\hat H_t$ is at most
$$(s-2)\cdot3000s\left(\sum_{i=s}^{n}\Gap{i}^{-2}\right)\delta^{2(s-2)}/\hat H_t\text{.}$$

It remains to apply a union bound on all $\hat H_t$ that fall into this case. Recall that $\hat H_t>300000s\sum_{i=s}^{n}\Gap{i}^{-2}$ and $\hat H_t$ grows exponentially in $t$ at a rate of $100$. Thus the total probability is upper bounded by
\begin{equation*}
\sum_{k=0}^{\infty}\frac{3000s\left(\sum_{i=s}^{n}\Gap{i}^{-2}\right)\delta^{2(s-2)}(s-2)}{100^k\cdot300000s\sum_{i=s}^{n}\Gap{i}^{-2}}
=\sum_{k=0}^{\infty}\frac{\delta^{2(s-2)}(s-2)}{100^{k+1}}
=\frac{1}{99}\delta^{2(s-2)}(s-2)\text{.}
\end{equation*}

\textbf{Case 2-2:} $\hat H<\sum_{i=s}^{n}\Gap{i}^{-2}/(78s)$.

In this case, we apply the technique in the proof of Case 1. We partition the sequence $A_s,A_{s+1},\ldots,A_n$ into $m$ consecutive blocks $B_1,B_2,\ldots,B_m$ such that $H(B_i)\in\left[\Gap{s}^{-2},3\Gap{s}^{-2}\right]$. Let $B_{\le i}$ denote $B_1\cup B_2\cup\cdots\cup B_i$. Since $H(B_1)=\Gap{s}^{-2}\le\hat H<\sum_{i=s}^{n}\Gap{i}^{-2}/(78s)<H(B_{\le m})$, there exists a unique integer $k\in[2,m]$ such that $H(B_{<k})\le\hat H<H(B_{\le k})$. It follows that $\hat H\in\left[(k-1)\Gap{s}^{-2},3k\Gap{s}^{-2}\right]$.

By Lemma~\ref{L4}, in order for \algent{} to accept $\hat H$, one of the following two events happens: (a) \algent{} mis-deletes all arms in $B_{\le k}\cup\{A_1,A_2,\ldots,A_{s-1}\}$; (b) the total complexity of mis-deleted arms among $B_{>k}$ is greater than $H(B_{>k})-\hat H$. Since $B_{\le k}$ contains at least $k$ arms, by Lemma~\ref{L3}, the probability of event (a) is bounded by $\delta^{2(s+k-1)}$.

Again, we bound the probability of event (b) using the generalized Chernoff bound in Lemma~\ref{LCher}. For each $i=s,s+1,\ldots,n$, define $v_i=\Gap{i}^{-2}/(3\Gap{s}^{-2})$ and $Y_i=v_i\cdot\mathbb{I}[A_i\text{ is mis-deleted}]$.
Define random variables $\{X_i:i\in\{1,2,\ldots,m\}\}$ as
$$X_i=\sum_{A_j\in B_i}Y_j=\frac{1}{3\Gap{s}^{-2}}\sum_{A\in B_i}\Delta_{A}^{-2}\cdot\mathbb{I}[A\text{ is mis-deleted}]\text{.}$$
Since $H(B_i)\le3\Gap{s}^{-2}$, $X_i$ is between $0$ and $1$. Let
$$X=\frac{1}{m}\sum_{i=1}^{m}X_i=\frac{1}{3m\Gap{s}^{-2}}\sum_{i=s}^{n}\Gap{i}^{-2}\cdot\mathbb{I}[A_i\text{ is mis-deleted}]$$
denote the mean of these random variables. Since the events $\{Y_i=v_i\}$ are $\delta^2$-quasi-independent, we may apply Lemma~\ref{LCher}. We have
$$p=\frac{\delta^2}{m}\sum_{i=s}^{n}v_i=\frac{H(B_{\le m})\delta^{2}}{3m\Gap{s}^{-2}}\le\delta^{2}\text{.}$$
Here the last step applies $H(B_{\le m})\le3m\Gap{s}^{-2}$.
On the other hand, conditioning on event (b) (i.e., a collection of arms in $B_{>k}$ with total complexity $H(B_{>k})-\hat H$ are mis-deleted), we have
\begin{equation*}\begin{split}
X&=\frac{1}{3m\Gap{s}^{-2}}\sum_{i=s}^{n}\Gap{i}^{-2}\cdot\mathbb{I}[A_i\text{ is mis-deleted}]\\
&\ge\frac{H(B_{>k})-\hat H}{3m\Gap{s}^{-2}}\ge\frac{(m-k)\Gap{s}^{-2}-3k\Gap{s}^{-2}}{3m\Gap{s}^{-2}}\\
&\ge\frac{m-4k}{3m}\ge\frac{m-4m/(26s)}{3m}\ge\frac{1}{6}\text{.}
\end{split}\end{equation*}
Here the third step follows from $H(B_{>k})\ge(m-k)\Gap{s}^{-2}$ and $\hat H\le 3k\Gap{s}^{-2}$. The last line holds since
$$k\Gap{s}^{-2}\le\hat H\le H(B_{\le m})/(78s)\le m\Gap{s}^{-2}/(26s)\text{,}$$
which implies $k\le m/(26s)$.
By Lemma~\ref{LCher}, we have
$$\Pr[X\ge1/6]\le\delta^{m/6}\text{,}$$
and thus the probability that \algent{} return an incorrect answer on $\hat H_t$ is bounded by $\delta^{2(s+k-1)}+\delta^{m/6}$.

It remains to apply a union bound on all valus of $\hat H_t$ that fall into this case.
Since $k\ge2$ and the values of $k$ are distinct, the sum of the first term is bounded by
$$\sum_{k=2}^{\infty}\delta^{2(s+k-1)}=\frac{\delta^{2s+2}}{1-\delta^2}\le2\delta^{2s+2}\text{.}$$
For the second term, note that the number of different values of $\hat H_t$ between $\Gap{s}^{-2}$ and $\sum_{i=s}^{n}\Gap{i}^{-2}/(78s)=H(B_{\le m})/(78s)$ is bounded by
$$\log_{100}\left[H(B_{\le m})/(78s)/\Gap{s}^{-2}\right]+1\le\log_{100}[m/(26s)]+1\text{.}$$
In particular, if $m<26s$, no $\hat H_t$ will fall into this case. So in the following we focus on the nontrivial case that $m\ge 26s$. Since
The sum of the second term is at most
$$\delta^{m/6}(\log_{100}[m/(26s)]+1)\le\delta^{13s/3}\le\delta^{2s+2}\text{.}$$
Here the first step follows from the fact that $\delta^{m/6}(\log_{100}[m/(26s)]+1)$ decreases on $[26s,+\infty)$ for all $\delta\in(0,0.01)$ and $s\ge3$. The second step follows from $s\le3$.

Therefore, the total probability that \algent{} returns incorrectly in this sub-case is bounded by
$$2\delta^{2s+2}+\delta^{2s+2}\le3\delta^{2s+2}\text{.}$$

\textbf{Case 2-3:} $\hat H\in[\sum_{i=s}^{n}\Gap{i}^{-2}/(78s),300000s\sum_{i=s}^{n}\Gap{i}^{-2}]$.

In this case, we simply bound the probability of returning an incorrect answer by the probability that at least $s-2$ arms in $\{A_1,A_2,\ldots,A_{s-1}\}$ are mis-deleted, which is in turn bounded by $(s-1)\delta^{2(s-2)}$ according to Lemma~\ref{L3}. As in the argument of Case 2-1, suppose that two arms $A_i$ and $A_j$ ($1\le i<j\le s-1$) are not mis-deleted. Let $r^*$ be the group that contain $A_{s-1}$. Then both $A_i$ and $A_j$ are contained in $S_{r^*}$. However, as $4\eps_{r^*}^{-2}=4^{r^*+1}\ge\Gap{s-1}^{-2}\ge\hat H_t$, \algent{} will reject in round $r^*$, which implies that \algent{} will never return a sub-optimal arm.

Note that at most
$$\log_{100}\frac{300000s}{1/(78s)}+1\le2\log_{100}s+5=\log_{10}s+5$$
different values of $\hat H$ fall into this case. Therefore, the total probability is bounded by 
$$\delta^{2(s-2)}(\log_{10}s+5)(s-1)\text{.}$$

Combining Case 2-1 through Case 2-3 yields the following bound: the probability that \algent{} outputs an incorrect answer for some $3\le s\le n$ and $\hat H\in[\Gap{s}^{-2},\Gap{s-1}^{-2})$ is at most
\begin{equation*}\begin{split}
&\sum_{s=3}^n\left[\frac{1}{99}\delta^{2(s-2)}(s-2)+3\delta^{2s+2}+\delta^{2(s-2)}(\log_{10}s+5)(s-1)\right]\\
=&\sum_{s=3}^{n}\delta^{2(s-2)}\left[\frac{s-2}{99}+3\delta^6+(\log_{10}s+5)(s-1)\right]\\
\le&\sum_{s=3}^{\infty}\delta^{2(s-2)}(\log_{10}s+6)(s-1)\\
\le&\delta^2\sum_{s=3}^{\infty}0.01^{2(s-3)}(\log_{10}s+6)(s-1)\le20\delta^2\text{.}\\
\end{split}\end{equation*}

\textbf{Case 3:} $\hat H_t<\Gap{n}^{-2}$.

Finally, we turn to the case that $\hat H$ is smaller than $\Gap{n}^{-2}$. In this case, \algguess{} always rejects. Suppose $A_n\in G_{r^*}$. Then in the first $r^*-1$ rounds of \algent{}, $\FT$ always returns False. Thus no elimination is done before round $r^*$. Since $\hat H_t<\Gap{n}^{-2}\le 4\eps_{r^*}^{-2}$, \algent{} directly rejects when checking the if-statement at round $r^*$.

Case 1 through Case 3 together directly imply the lemma, as $3\delta^4+20\delta^2<\delta/3$ for all $\delta\in(0,0.01)$.
\end{proof}

\section{Analysis of Sample Complexity}\label{app:samp}
Recall that $\event_1$ is the event that all calls of $\FT$ and $\US$ in \algent{} return correctly. We bound the sample complexity of our algorithm using the following two lemmas.

\begin{Lemma}\label{L6}
Conditioning on $\event_1$, the expected number of samples taken by $\ME$ and $\EL$ in \algguess{} is $$O(H(\arment+\ln\delta^{-1}))\text{.}$$
\end{Lemma}

\begin{Lemma}\label{L5}
Conditioning on $\event_1$, the expected number of samples taken by $\US$ and $\FT$ in \algguess{} is $$O(\Gap{2}^{-2}\ln\ln\Gap{2}^{-1}\polylog(n,\delta^{-1}))\text{.}$$
\end{Lemma}

The two lemmas above directly imply the following theorem.

\begin{Theorem}\label{T3}
Conditioning on $\event_1$, the expected sample complexity of \algguess{} is
	$$O\left(H(\ln\delta^{-1}+\arment)+\Gap{2}^{-2}\ln\ln\Gap{2}^{-1}\polylog(n,\delta^{-1})\right)\text{.}$$
\end{Theorem}

Theorems \ref{T2}~and~\ref{T3} together imply that \algguess{} is a \CORRECT{} algorithm for \bestarm{}, and its expected sample complexity is
	$$O\left(H(\ln\delta^{-1}+\arment)+\Gap{2}^{-2}\ln\ln\Gap{2}^{-1}\polylog(n,\delta^{-1})\right)$$
conditioning on an event which happens with probability at least $1-\delta$. However, to prove Theorem~\ref{theo:uppb-main}, we need a \CORRECT{} algorithm with the desired sample complexity in expectation (not conditioning on another event). In the following, we prove Theorem~\ref{theo:uppb-main} using a parallel simulation trick developed in \cite{chen2015optimal}.

\begin{proof}[Proof of Theorem~\ref{theo:uppb-main}]
Given an instance $I$ of \bestarm{} and a confidence level $\delta$, we define a collection of algorithms $\{\alg_k:k\in\mathbb{N}\}$, where $\alg_k$ simulates \algguess{} on instance $I$ and confidence level $\delta_k=\delta/2^k$. Then we construct the following algorithm $\alg$:
\begin{itemize}
\item $\alg$ runs in iterations. In iteration $t$, for each $k$ such that $2^{k-1}$ divides $t$, $\alg$ simulates $\alg_k$ until $\alg_k$ requires a sample from some arm $A$. $\alg$ draws a sample from $A$, feeds it to $\alg_k$, and continue simulating $\alg_k$ until it requires another sample. After that, $\alg$ temporarily suspends $\alg_k$.
\item When some algorithm $\alg_k$ terminates, $\alg$ also terminates and returns the same answer.
\end{itemize}

We first note that if all algorithms in $\{\alg_k\}$ are correct, $\alg$ eventually returns the correct answer. Recall that $\alg_k$ is a $\delta/2^k$-correct algorithm for \bestarm{}. Thus by a simple union bound, $\alg$ is correct with probability $1-\sum_{k=1}^{\infty}\delta/2^k=1-\delta$, thus proving that $\alg$ is \CORRECT{}.

It remains to bound the sample complexity of $\alg$. According to Theorem~\ref{T3}, there exist constants $C$ and $m$, along with a collection of events $\{\event_k\}$, such that for each $k$, $\Pr[\event_k]\ge1-\delta_k$, and the expected number of samples taken by $\alg_k$ conditioning on $\event_k$ is at most
\begin{equation*}\begin{split}
&C\left[H\cdot(\ln\delta_k^{-1}+\arment)+\Gap{2}^{-2}\ln\ln\Gap{2}^{-1}(\ln^m n+\ln^m\delta_k^{-1})\right]\\
\le&C\left[H\cdot(k\ln\delta+\arment)+\Gap{2}^{-2}\ln\ln\Gap{2}^{-1}(\ln^m n+(k\ln\delta^{-1})^m)\right]\\
\le&k^m\cdot T(I)\text{.}
\end{split}\end{equation*}
Here $T(I)$ denotes $C\left[H\cdot(\ln\delta^{-1}+\arment)+\Gap{2}^{-2}\ln\ln\Gap{2}^{-1}(\ln^m n+\ln^m\delta^{-1})\right]$, the desired sample complexity. The first step follows from the fact that $\ln\delta_k^{-1}=\ln\delta^{-1}+k\le k\ln\delta^{-1}$ for $\delta<0.01$.

Since different algorithms in $\{\alg_k\}$ take independent samples, the events $\{\event_k\}$ are independent. Define random variable $\sigma$ as the minimum number such that event $\event_\sigma$ happens. Then it follows that
$$\Pr[\sigma=k]\le\Pr[\overline{\event}_1\cap\overline{\event}_2\cap\cdots\cap\overline{\event}_{k-1}]\le\prod_{i=1}^{k-1}\delta_i\le0.01^{k-1}\text{.}$$
Let $T_k$ denote the number of samples taken by $\alg_k$ if it is allowed to run indefinitely (i.e., $\alg$ does not terminate). Conditioning on $\sigma=k$, we have $\Ex[T_k]\le k^m\cdot T(I)$. Moreover, $\alg$ terminates before or at iteration $2^{k-1}k^m\cdot T(I)$. It follows that the number of samples taken by $\alg$ is bounded by
$$\sum_{i=1}^{\infty}\lfloor2^{k-1}k^m\cdot T(I)/2^{i-1}\rfloor\le 2^{k-1}k^m\cdot T(I)\sum_{i=1}^{\infty}2^{-(i-1)}\le 2^kk^m\cdot T(I)\text{.}$$

Thus the expected sample complexity of $\alg$ is bounded by
\begin{equation*}\begin{split}
&\sum_{k=1}^{\infty}\Pr[\sigma=k]\cdot 2^kk^m\cdot T(I)\\
\le&\sum_{k=1}^{\infty}0.01^{k-1}\cdot 2^kk^m\cdot T(I)\\
\le&100T(I)\sum_{k=1}^{\infty}0.02^{k}k^m=O(T(I))\text{.}
\end{split}\end{equation*}
Therefore, $\alg$ is a \CORRECT{} algorithm for \bestarm{} with expected sample complexity of
$$O(H\cdot(\ln\delta^{-1}+\arment)+\Gap{2}^{-2}\ln\ln\Gap{2}^{-1}\polylog(n,\delta^{-1}))\text{.}$$
\end{proof}

We conclude the section with the proofs of Lemmas \ref{L6}~and~\ref{L5}.

\begin{proof}[Proof of Lemma~\ref{L6}]
Suppose that \algguess{} terminates after iteration $t_0$. According to \algent{}, for each $1\le t\le t_0$, the algorithm takes $O(\hat H_t)=O(100^t)$ samples in $\ME$ and $\EL$ when it runs on $\hat H_t$. As $100^t$ grows exponentially in $t$, it suffices to bound the expectation of the last term, namely $100^{t_0}$.

Let $t^*=\lceil\log_{100}H(\arment+\ln\delta^{-1})+3\rceil$. We first show that when $t\ge t^*$, \algent{} accepts $\hat H_t$ with constant probability. According to Remark~\ref{R1}, the probability that \algent{} rejects $\hat H_t$ is upper bounded by
$$\frac{256H}{\hat H_t}+\frac{H(\arment+\ln\delta^{-1}+\ln(\hat H_t/H))}{100\hat H_t}\le\frac{256H}{100^3H}+\frac{\hat H_t/20}{100\hat H_t}\le1/200\text{.}$$ The first step follows from the following two observations. First, as $\hat H_{t}\ge\hat H_{t^*}\ge100^3H(\arment+\ln\delta^{-1})$, we have $H(\arment+\ln\delta^{-1})\le100^{-3}\hat H_t$. Second, since $\hat H_t/H\ge100^3$ and $x\ge100\ln x$ holds for all $x\ge10^6$, we have $H\ln(\hat H_t/H)\le H\cdot\dfrac1{100}(\hat H_t/H)=\hat H_t/100$.

Therefore, the probability that $t_0$ equals $t^*+k$ is bounded by $200^{-k}$ for all $k\ge1$. It follows from a simple summation on all possible $t_0$ that
\begin{equation*}\begin{split}
\Ex\left[100^{t_0}\right]&=\sum_{t=1}^{\infty}100^t\Pr[t_0=t]\\
&\le\sum_{t=1}^{t^*}100^t\cdot 1+\sum_{k=1}^{\infty}100^{t^*+k}\cdot200^{-k}\\
&=O\left(100^{t^*}\right)=O\left(H(\arment+\ln\delta^{-1})\right).
\end{split}\end{equation*}
\end{proof}

\begin{proof}[Proof of Lemma~\ref{L5}]
When \algent{} runs on guess $\hat H_t$, $\US$ takes $O(\eps_r^{-2}\ln\delta_r^{-1})$ samples in the $r$-th round, while the number of samples taken by $\FT$ is
$$O\left(\eps_r^{-2}\ln\delta_r^{-1}(\theta_r-\theta_{r-1})^{-2}\ln(\theta_r-\theta_{r-1})^{-1}\right)=O\left(\eps_r^{-2}\ln\delta_r^{-1}(\theta_r-\theta_{r-1})^{-3}\right)\text{.}$$ As the second term dominates the first, we focus on the complexity of $\FT$ in the following analysis.

Recall that $\eps_r=2^{-r}$, $\delta_r=\delta/(50r^2t^2)\ge\delta^2/(r^2t^2)$ and $\theta_r-\theta_{r-1}=(ct-r)^{-2}/10$. For each $t$, suppose $r$ ranges from $1$ to $r_{\max}$, then the complexity at iteration $t$ is bounded by
\begin{equation*}\begin{split}
&\sum_{r=1}^{r_{\max}}\eps_r^{-2}\ln\delta_r^{-1}(\theta_r-\theta_{r-1})^{-3}\\
\le&2\sum_{r=1}^{r_{\max}}4^r(\ln\delta^{-1}+\ln r+\ln t)[(ct-r)^{-2}/10]^{-3}\\
\le&2000\sum_{r=1}^{r_{\max}}4^r(\ln\delta^{-1}+\ln r+\ln t)(ct-r)^{6}\\
=&O(4^{r_{\max}}(\ln\delta^{-1}+\ln t)(ct-r_{\max})^6)\\
\end{split}\end{equation*}
The last step follows from the observation that the last term dominates the summation, and the fact $\ln r_{\max}=O(\ln t)$ due to Observation~\ref{O1}.

Let random variable $t_0$ denote the last $t$ in the execution of \algguess{}. As in the proof of Lemma~\ref{L6}, we define $t^*=\lceil\log_{100}H(\arment+\ln\delta^{-1})+3\rceil$. We have also shown that $\Pr[t\ge t_0+k]\le200^{-k}$ for all $k\ge1$. Thus, the expected complexity incurred after iteration $t^*$ can be bounded by the complexity at iteration $t^*$.

When $t<\log_{100}\Gap{2}^{-2}$, it follows from $r_{\max}\le ct-1$ that the complexity is
$$O(4^{ct}(\ln\delta^{-1}+\ln t))=O(100^t(\ln\delta^{-1}+\ln t))\text{.}$$
Summing over $t=1,2,\ldots,\log_{100}\Gap{2}^{-2}$ yields
$$\sum_{t=1}^{\log_{100}\Gap{2}^{-2}}100^t(\ln\delta^{-1}+\ln t)=O(\Gap{2}^{-2}(\ln\delta^{-1}+\ln\ln\Gap{2}^{-1}))\text{.}$$
Clearly this term is bounded by the desired complexity.

When $\log_{100}\Gap{2}^{-2}\le t\le t^*$, we choose $r_{\max}=\log_{2}\Gap{2}^{-1}=\log_{4}\Gap{2}^{-2}$. Note that in fact the algorithm may not always terminate before or at round $r_{\max}$. However, since the probability that the algorithm lasts $r_{\max}+k$ rounds is bounded by $O(100^{-k})$, the contribution of those rounds to total complexity is also dominated. Thus we have
\begin{equation*}\begin{split}
&\sum_{t=\log_{100}\Gap{2}^{-2}}^{t_0}O(4^{r_{\max}}(\ln\delta^{-1}+\ln t)(ct-r_{\max})^6)\\
=&\sum_{t=\log_{100}\Gap{2}^{-2}}^{t_0}O(\Gap{2}^{-2}(\ln\delta^{-1}+\ln t)(ct-\log_4\Gap{2}^{-2})^6)\\
=&O\left(t^*-\log_{100}\Gap{2}^{-2}\right)\cdot O\left(\Gap{2}^{-2}(\ln\delta^{-1}+\ln t^*)(ct^*-\log_4\Gap{2}^{-2})^6\right)\\
=&O\left(\Gap{2}^{-2}(\ln\delta^{-1}+\ln t^*)(ct^*-\log_4\Gap{2}^{-2})^7\right)\\
=&O\left(\Gap{2}^{-2}(\ln\delta^{-1}+\ln\ln H)(\ln(H/\Gap{2}^{-2})+\ln\arment+\ln\delta^{-1})^7\right)\\
=&O\left(\Gap{2}^{-2}\ln\ln\Gap{2}^{-1}(\ln\delta^{-1}+\ln\ln n)(\ln n+\ln\delta^{-1})^7\right)\\
=&O\left(\Gap{2}^{-2}\ln\ln\Gap{2}^{-1}\polylog(n,\delta^{-1})\right)\text{.}
\end{split}\end{equation*}
The fourth step follows from $$O(\ln t^*)=O(\ln\ln(H(\arment+\ln\delta^{-1})))=O(\ln\ln H+\ln\ln\arment+\ln\ln\ln\delta^{-1})\text{,}$$ while the last two terms are dominated by $\ln\delta^{-1}+\ln\ln H$. The fifth step follows from the simple observation that $H/\Gap{2}^{-2}\le n$ and $\arment=O(\ln\ln n)$.
\end{proof}

\section{Lower Bound}\label{app:lb}
In this section, we prove Lemma~\ref{L10}. We restate it here for convenience.

$ $

\noindent\textbf{Lemma \ref{L10}.} (restated)\textit{
	Suppose $\delta\in(0,0.04)$, $m\in\mathbb{N}$ and
	$\mathbb{A}$ is a $\delta$-correct algorithm for \sign{}.
	$P$ is a probability distribution on $\{2^{-1},2^{-2},\ldots,2^{-m}\}$
	defined by $P(2^{-k})=p_k$.
	$\arment(P)$ denotes the Shannon entropy of distribution $P$.
	Let $T_{\alg}(\mu)$ denote the expected number of samples taken by $\alg$
	when it runs on an arm with distribution $\Normal(\mu,1)$ and $\xi=0$.
	Define $\alpha_k=T_{\alg}(2^{-k})/4^k$. Then,
		\[\sum_{k=1}^{m}p_k\alpha_k=\Omega(\arment(P)+\ln\delta^{-1}).\]
}

\subsection{Change of Distribution}
We introduce a lemma that is essential for proving the lower bound for \sign{} in Lemma \ref{L10}, which is a special case of~\cite[Lemma~1]{kaufmann2015complexity}. In the following, $\KL$ stands for the Kullback-Leibler divergence, while $D(x||y)=x\ln(x/y)+(1-x)\ln[(1-x)/(1-y)]$ is the relative entropy function.

\begin{Lemma}[Change of Distribution]\label{LCoD}
Let $\alg$ be an algorithm for \sign{}.
Let $A$ and $A'$ be two instances of \sign{} (i.e., two arms).
$\Pr_{A}$ and $\Pr_{A'}$ ($\Ex_{A}$ and $\Ex_{A'}$) denote the probability law (expectation)
when $\alg$ runs on instance $A$ and $A'$ respectively.
Random variable $\tau$ denotes the number of samples taken by the algorithm.
For all event $\event$ in $\mathcal{F}_{\sigma}$,
where $\sigma$ is a stopping time with respect to the filtration $\{\mathcal{F}_t\}$,
we have
	$$
		\Ex_{A}[\tau]\KL(A,A')
		\ge D\left(\Pr_{A}[\event]\Big|\Big|\Pr_{A'}[\event]\right)\text{.}
	$$
\end{Lemma}

\subsection{Proof of Lemma \ref{L10}}
We start with an overview of our proof of Lemma \ref{L10}. For each $k$, we consider the number of samples taken by Algorithm $\alg$ when it runs on an arm with mean $2^{-k}$. We first show that with high probability, this number is between $\Omega(4^k)$ and $O(4^k\alpha_k)$.  Then we apply Lemma \ref{LCoD} to show that the same event happens with probability at least $e^{-\alpha_k}$ when the input is an arm with mean zero.

Since the probability of an event is at most $1$, we would like to bound the sum of $e^{-\alpha_k}$ by $1$, yet the problem is that the events for different $k$ may not be disjoint. To avoid this difficulty, we carefully select a collection of disjoint events denoted by $S$. We bound $\sum_{k=1}^{m}e^{-d\alpha_k}$ (for appropriate constant $d$) by $\sum_{k\in S}e^{-\alpha_k}$ based on the way we construct $S$. After that, we use the ``change of distribution'' argument (Lemma \ref{LCoD}) to bound $\sum_{k\in S}e^{-\alpha_k}$ by $1$. As a result, we have the following inequality for appropriate constant $M$, which is reminiscent of Kraft's inequality in coding theory.
\begin{equation}\label{eq7}
\sum_{k=1}^{m}e^{-d\alpha_k}\le M\text{.}
\end{equation}
Once we obtain \eqref{eq7}, the desired bound directly follows from a simple calculation.

\begin{proof}[Proof of Lemma \ref{L10}]
Fix $m\in\mathbb{N}$. Recall that all arms are normal distributions with a standard deviation of $1$ and $\xi$ is always equal to zero. $4^k\alpha_k$ is the expected number of samples taken by $\alg$ on an arm $A$ with distribution $\Normal(2^{-k},1)$. It is well-known that to distinguish $\Normal(2^{-k},1)$ from $\Normal(-2^{-k},1)$ with confidence level $\delta$, $\Omega(4^k\ln\delta^{-1})$ samples are required in expectation. Therefore, we have $\alpha_k=\Omega(\ln\delta^{-1})$ for all $k$. It follows that $\sum_{k=1}^{m}p_k\alpha_k=\Omega(\ln\delta^{-1})$.

It remains to prove that $\sum_{k=1}^{m}p_k\alpha_k=\Omega(\arment(P))$ for all $0.04$-correct algorithms. For each $\mu\in\mathbb{R}$, let $\Pr_{\mu}$ and $\Ex_{\mu}$ denote the probability and expectation when $\alg$ runs on an arm with mean $\mu$ (i.e., $\Normal(\mu,1)$). Define random variable $\tau_{\alg}$ as the number of samples taken by $\alg$. Let $c=1/64$. Let $\event_k$ denote the event that $\alg$ outputs ``$\mu>0$'' and $\tau_{\alg}\in[4^kc,16\cdot4^k\alpha_k]$. The following lemma gives a lower bound of $\Pr_{0}[\event_k]$.

\begin{Lemma}\label{L11}
$$\Pr_{0}[\event_k]\ge\frac{1}{4}e^{-\alpha_k}\text{.}$$
\end{Lemma}

Our second step is to choose a collection of disjoint events from $\{\event_k:1\le k\le m\}$. We have the following lemma.

\begin{Lemma}\label{L12}
There exists a set $S\subseteq\{1,2,\ldots,m\}$ such that:
\begin{itemize}
\item $\{\event_k:k\in S\}$ is a collection of disjoint events.
\item $\sum_{k=1}^{m}e^{-d\alpha_k}\le M\sum_{k\in S}e^{-\alpha_k}$ for universal constants $d$ and $M$ independent of $m$ and $\alg$.
\end{itemize}
\end{Lemma}

It follows that
$$\sum_{k=1}^me^{-d\alpha_k}\le M\sum_{k\in S}e^{-\alpha_k}\le4M\sum_{k\in S}\Pr_{0}[\event_k]=4M\text{.}$$
Here the first two steps follow from Lemma \ref{L12} and Lemma \ref{L11}, respectively. The last step follows from the fact that $\{\event_k:k\in S\}$ is a disjoint collection of events.

Finally, for a distribution $P$ on $\{2^{-1},2^{-2},\ldots,2^{-m}\}$ defined by $P(2^{-k})=p_k$, we consider the following optimization problem with variables $\alpha_1,\alpha_2,\ldots,\alpha_m$:
\begin{equation*}\begin{split}
\textrm{minimize}~~~~&\sum_{k=1}^{m}p_k\alpha_k\\
\textrm{subject to}~~~~&\sum_{k=1}^{m}e^{-d\alpha_k}\le 4M\\
\end{split}\end{equation*}
The method of Lagrange multipliers yields that the minimum value is obtained when $\sum_{k=1}^{m}e^{-d\alpha_k}=4M$ and $e^{-d\alpha_k}$ is proportional to $p_k$. It follows that $\alpha_k=-\dfrac1d\ln(4Mp_k)$ and consequently
$$\sum_{k=1}^mp_k\alpha_k\ge\frac1d\sum_{k=1}^mp_k(\ln(4M)^{-1}+\ln p_k^{-1})=\frac1d\left(\arment(P)-\ln(4M)\right)\text{.}$$
Note that $d$ and $M$ are constants independent of $m$, distribution $P$ and algorithm $\alg$. This completes the proof.

\end{proof}

\subsection{Proofs of Lemma~\ref{L11} and Lemma~\ref{L12}}
Finally, we prove the two technical lemmas.

\begin{proof}[Proof of Lemma \ref{L11}]
Recall that our goal is to lower bound $\Pr_{0}[\event_k]$. We first show that $\Pr_{2^{-k}}[\event_k]\ge1/2$ and then prove the desired lower bound by applying change of distribution. Recall that $\event_k=(\alg\text{ outputs }\mu>0)\wedge(\tau_{\alg}\in[4^kc,16\cdot4^k\alpha_k])$. We have
\begin{equation*}\begin{split}
\Pr_{2^{-k}}[\event_k]&\ge\Pr_{2^{-k}}\left[\alg\text{ outputs }\mu>0\right]-\Pr_{2^{-k}}\left[\tau_{\alg}>16\cdot4^k\alpha_k\right]-\Pr_{2^{-k}}\left[\tau_{\alg}<4^kc\right]\\
&\ge1-0.04-1/16-\Pr_{2^{-k}}\left[\tau_{\alg}<4^kc\right]\\
&\ge0.8-\Pr_{2^{-k}}\left[\tau_{\alg}<4^kc\right]\text{.}
\end{split}\end{equation*}
Here the second step follows from Markov's inequality and the fact that $\Ex_{2^{-k}}[\tau_{\alg}]=4^k\alpha_k$.

It remains to show that $\Pr_{2^{-k}}\left[\tau_{\alg}<4^kc\right]\le0.3$. Suppose towards a contradiction this does not hold. Then we consider the algorithm $\algp$ that simulates $\alg$ in the following way: if $\alg$ terminates within $4^kc$ samples, $\algp$ outputs the same answer; otherwise $\algp$ outputs nothing. Let $\Pr_{\algp,\mu}$ denote the probability when $\algp$ runs on an arm of mean $\mu$. Moreover, let $\Ebad_k$ denote the event that the output is ``$\mu>0$''. Then we have
$$\Pr_{\algp,2^{-k}}[\Ebad_k]=\Pr_{2^{-k}}\left[\Ebad_k\wedge\tau_{\alg}<4^kc\right]\ge\Pr_{2^{-k}}\left[\tau_{\alg}<4^kc\right]-0.04>0.26\text{.}$$
On the other hand, when we run $\algp$ on an arm with mean $-2^{-k}$, we have $$\Pr_{\algp,-2^{-k}}\left[\Ebad_k\right]\le\Pr_{-2^{-k}}\left[\Ebad_k\right]\le0.04\text{.}$$

Since $\algp$ never takes more than $4^kc$ samples, it follows from Lemma \ref{LCoD} that
\begin{equation*}\begin{split}
2c&=4^kc\cdot\KL(\Normal(2^{-k},1),\Normal(-2^{-k},1))\\
&\ge\Ex_{\algp,2^{-k}}[\tau_{\algp}]\cdot\KL(\Normal(2^{-k},1),\Normal(-2^{-k},1))\\
&\ge D\left(\Pr_{\algp,2^{-k}}[\Ebad_k]\Big|\Big|\Pr_{\algp,-2^{-k}}[\Ebad_k]\right)\\
&\ge D(0.26||0.04)\ge0.2\text{,}
\end{split}\end{equation*}
which leads to a contradiction as $c=1/64$.

In the following, we lower bound $\Pr_{0}[\event_k]$ using change of distribution. Note that
\begin{equation*}\begin{split}
D\left(\Pr_{2^{-k}}[\event_k]\Big|\Big|\Pr_{0}[\event_k]\right)&\le4^k\alpha_k\cdot\KL(\Normal(2^{-k},1),\Normal(0,1))\\
&\le4^k\alpha_k\cdot\frac{1}{2}\left(2^{-k}\right)^2=\alpha_k/2\text{.}
\end{split}\end{equation*}

Let $\theta_k=e^{-\alpha_k}/4$. We have
$$D(1/2||\theta_k)=\frac{1}{2}\ln\frac{1}{4\theta_k(1-\theta_k)}
\ge\frac{1}{2}\ln\frac{1}{4\theta_k}
=\alpha_k/2\text{.}$$
Since we have shown $\Pr_{2^{-k}}[\event_k]\ge1/2$, the two inequalities above imply
$$\Pr_{0}[\event_k]\ge \theta_k=\frac{1}{4}e^{-\alpha_k}\text{.}$$
\end{proof}

\begin{proof}[Proof of Lemma \ref{L12}]
We map each event $\event_k$ to an interval
$$\INT_k=[\log_4(4^kc)+3,\log_4(16\cdot4^k\alpha_k)+3]=[k,k+\log_4\alpha_k+5]\text{.}$$
By construction, two events $\event_i$ and $\event_j$ are disjoint if and only if their corresponding intervals, $\INT_i$ and $\INT_j$, are disjoint.

We construct a subset of $\{1,2,\ldots,m\}$ using the following greedy algorithm:
\begin{itemize}
\item Sort $(1,2,\ldots,m)$ into a list $(l_1,l_2,\ldots,l_m)$ such that $\alpha_{l_1}\le\alpha_{l_2}\le\cdots\le\alpha_{l_m}$.
\item While the list is not empty, we add the first element $x$ in the list into set $S$. Let $S_x=\{y:y\text{ is in the current list, and }\INT_x\cap\INT_y\ne\emptyset\}$. We remove all elements in $S_x$ from the list.
\end{itemize}

Note that the way we construct $S$ ensures that $\{\event_k:k\in S\}$ is indeed a disjoint collection of events, which proves the first part of the lemma. Moreover, $\{S_k:k\in S\}$ is a partition of $\{1,2,\ldots,m\}$. Thus we have
\begin{equation}\label{eq9}
\sum_{k=1}^me^{-d\alpha_k}=\sum_{k\in S}\sum_{j\in S_k}e^{-d\alpha_j}\text{.}
\end{equation}
It suffices to bound $\sum_{j\in S_k}e^{-d\alpha_j}$ by $Me^{-\alpha_k}$ for appropriate constants $d$ and $M$. Summing over all $k$ yields the desired bound
$$\sum_{k=1}^{m}e^{-d\alpha_k}\le\sum_{k\in S}e^{-\alpha_k}\text{.}$$

According to our construction of $S$, for all $j\in S_k$ we have $\alpha_j\ge\alpha_k$. For each integer $l\ge\left\lfloor\log_4\alpha_k\right\rfloor$, we consider the values of $j$ such that $\log_4\alpha_j\in[l,l+1)$. Recall that the interval corresponding to event $\event_k$ is $\INT_k=[k,k+\log_4\alpha_k+5]$. In order for $\INT_j$ to intersect $\INT_k$, we must have $j\in[k-\log_4\alpha_j-5,k+\log_4\alpha_k+5]$. Since $\log_4\alpha_j<l+1$, $j$ must be contained in $[k-l-6,k+\log_4\alpha_k+5]$, and thus there are at most $\left(\log_4\alpha_k+l+12\right)$ such values of $j$.

Recall that since $\INT_k=[k,k+\log_4\alpha_k+5]$ is nonempty, we have $\alpha_k\ge4^{-5}$. In the following calculation, we assume for simplicity that $\alpha_k\ge1$ for all $k$, since it can be easily verified that the contribution of the terms with $\alpha_k<1$ (i.e., $l=-5,-4,\ldots,-1$) is a constant, and thus can be covered by a sufficiently large constant $M$ in the end. Then we have 
\begin{equation}\label{eq10}\begin{split}
	\sum_{j\in S_k}e^{-d\alpha_j}
	&\le\sum_{l=\lfloor\log_4\alpha_k\rfloor}^{\infty}\sum_{j\in S_k}\exp(-d\alpha_j)\mathbb{I}[\log_4\alpha_j\in[l,l+1)]\\
	&\le\sum_{l=\lfloor\log_4\alpha_k\rfloor}^{\infty}\exp(-d4^l)(\log_4\alpha_k+l+12)\\
	&=(\log_4\alpha_k+12)\sum_{l=\lfloor\log_4\alpha_k\rfloor}^{\infty}\exp\left(-d4^l\right)+\sum_{l=\lfloor\log_4\alpha_k\rfloor}^{\infty}l\exp\left(-d4^l\right)\\
	&=(\log_4\alpha_k+12)\cdot O(\exp(-d\alpha_k))+O(\exp(-d\alpha_k)\cdot \log_4\alpha_k)\\
	&\le M(\log_4\alpha_k+12)\exp(-d\alpha_k)\\
	&=M\exp(-d\alpha_k+\ln\log_4\alpha_k+\ln12)\le Me^{-\alpha_k}\text{.}
\end{split}\end{equation}
The first step rearranges the summation based on the value of $l$. The second step follows from the observation that $S_k$ contains at most $\log_4\alpha_k+l+12$ values of $j$ corresponding to each $l$. The fourth step holds since both summations decrease double-exponentially, and thus can be bounded by their respective first terms. Then we find a sufficiently large constant $M$ (which depends on $d$) to cover the hidden constant in the big-O notation. Finally, the last step holds for sufficiently large $d$. In fact, we first choose $d$ according to the last step, and then find the appropriate constant $M$. Clearly the choice of $M$ and $d$ is independent of the value of $m$ and the algorithm $\alg$.
\end{proof}

\begin{Remark}\label{Rlb}
Recall that all distributions are assumed to be Gaussian distributions with a fixed variance of $1$. In fact, our proof of Lemma \ref{L10} only uses the following property: the KL-divergence between two distributions with mean $\mu_1$ and $\mu_2$ is $\Theta((\mu_1-\mu_2)^2)$. Note that this property is indeed essential to the ``change of distribution'' argument in the proof of Lemma \ref{L11}.

In general, suppose $U$ is a set of real numbers and $\mathcal{D}=\{D_{\mu}:\mu\in U\}$ is a family of distributions with the following two properties: (1) the mean of distribution $D_{\mu}$ is $\mu$; (2) $\KL(D_{\mu_1},D_{\mu_2})\le C(\mu_1-\mu_2)^2$ for fixed constant $C>0$. Then Lemma \ref{L10} also holds for distributions from $\mathcal{D}$.

For instance, suppose $\mathcal{D}=\{B(1,\mu):\mu\in[1/2-\eps,1/2+\eps]\}$, where $\eps\in(0,1/2)$ is a constant and $B(1,\mu)$ denotes the Bernoulli distribution with mean $\mu$. Since
$$\KL(B(1,p),B(1,q))\le\frac{(p-q)^2}{q(1-q)}\le\frac{(p-q)^2}{1/4-\eps^2}\text{,}$$
distribution family $D$ satisfies the condition above with $C=\dfrac{4}{1-4\eps^2}$. It follows that Lemma \ref{L10} also holds for Bernoulli distributions with means sufficiently away from $0$ and $1$.
\end{Remark}

\section{Missing Proofs in Section~\ref{sec:toy}}\label{app:toymiss}
	In this section, we present the technical details in the proofs of Lemma \ref{lem:newtoycorrect} and Lemma \ref{lem:newtoysample}. These are essentially identical to the proofs of Lemmas \ref{L7} and \ref{L8}, which either use a potential function or apply a charging argument.
	\subsection{Proof of Lemma \ref{lem:newtoycorrect}}
	\begin{proof}[Proof of Lemma \ref{lem:newtoycorrect} (continued)]
		Recall that $P(r, S_r)$ is defined as the probability that, given the value of $S_r$ at the beginning of round $r$, at least one call to $\EL$ returns incorrectly at round $r$ or later rounds, while $\US$ and $\FT$ always return correctly. We prove inequality \eqref{eq:toyineq} by induction: for any $S_r$ that contains the optimal arm $A_1$,
			$$P(r, S_r)\le \frac{\delta}{\hat H}\left(128C(r, S_r)+16M(r, S_r)\eps_r^{-2}\right)\text{,}$$
		where
			$$C(r, S_r)\coloneqq \sum_{i=r-1}^{\infty}|S_r\cap G_i|\sum_{j=r}^{i+1}\eps_j^{-2}+\sum_{i=r}^{\rmax+1}\eps_i^{-2}\text{,}$$
		and
			$$M(r, S_r)\coloneqq |S_r\cap G_{\le r-2}|\text{.}$$
		Note that if $|S_r|=1$, the algorithm directly terminates at round $r$, and the inequality clearly holds. Thus, we assume $|S_r|\ge2$ in the following. 

		\textbf{Base case.}
		We prove the base case $r = \rmax+2$, where
			$\rmax = \max_{G_r\ne\emptyset}r$.
		Note that $C(r,S)=0$ and $M(r,S)=|S|-1$ for $r = \rmax+2$ and any $S\subseteq I$ with $A_1\in S$.

		Let random variable $r^*$ be the smallest integer greater than or equal to $r$, such that $\ME$ is correct at round $r^*$. Note that for $k\ge r$,
			$\pr{r^*=k}\le0.01^{k-r}$.
		We claim that conditioning on $r^* = k$, if $\EL$ is correct between round $r$ and round $k$, the algorithm will terminate at round $k + 1$. Consequently, the probability that $\EL$ fails in some round is bounded by the probability that it fails between round $r$ and $k$. This allows us to upper bound the conditional probability by $\sum_{i=r}^{k}\delta'_i$.

		Now we prove the claim. By Observation \ref{obs:toyFT},
		the lower threshold used in $\FT$ at round $k$, denoted by $\clow_k$,
		is greater than or equal to $\Mean{1}-\eps_k$.
		Since $k\ge r = \rmax+2$,
			$$|\{A\in S:\mu_A<\clow_k\}|
			\ge |\{A\in S:\mu_A<\Mean{1}-\eps_k\}|
			=   |S\cap G_{\le k - 1}|
			\ge |S\cap G_{\le \rmax + 1}|
			= |S| - 1
			\ge 0.5|S|\text{.}$$
		Thus by Fact \ref{F3}, $\FT$ is guaranteed to return True in round $k$, and the algorithm calls $\EL$.
		Then, by Observation \ref{obs:toyEL}, it holds that
			$\dlow_k \ge \Mean{1}-0.5\eps_k$.
		Assuming that $\EL$ returns correctly at round $k$, the set returned by $\EL$, denoted by $S_{k+1}$, satisfies
			$|\{A\in S_{k+1}:\mu_A < \dlow_k\}|<0.1|S_{k+1}|\text{,}$
		which implies
			$$|S_{k+1}|-1\le |S_{k+1}\cap G_{\le k}| = |\{A\in S_{k+1}:\mu_A < \Mean{1}-0.5\eps_k\}| < 0.1|S_{k+1}|\text{.}$$
		Thus we have $|S_{k+1}|=1$, which proves the claim.

		Summing over all possible $k$ yields that the probability that $\EL$ returns incorrectly is upper bounded by
			\begin{equation*}\begin{split}
				\sum_{k=r}^{\infty}\pr{r^*=k}\sum_{j=r}^{k}\delta'_j
			\le	&\sum_{k=r}^{\infty}0.01^{k-r}\sum_{j=r}^{k}\frac{|S_j|\eps_j^{-2}}{\hat H}\delta\\
			\le &\frac{4}{3}\cdot\frac{|S_r|\eps_r^{-2}\delta}{\hat H}\sum_{k=r}^{\infty}0.01^{k-r}\cdot4^{k-r}\\
			\le &\frac{2|S_r|\eps_r^{-2}\delta}{\hat H}
			\le \frac{\delta}{\hat H}\left(128C(r, S_r)+16M(r, S_r)\eps_r^{-2}\right)\text{.}
			\end{split}\end{equation*}

		\textbf{Inductive step.} Assuming that the inequality holds for $r+1$ and all $S_{r+1}$ that contains $A_1$, we bound the probability $P(r, S_r)$. We first note that both $C$ and $M$ are monotone in the following sense:
			$C(r, S)\le C(r, S')$ and $M(r, S)\le M(r, S')$ for $S\subseteq S'$.
		Moreover, we have
			\begin{equation}\label{eq:Cdiff}
			C(r, S_r) - C(r+1, S_r)
			= \sum_{i=r-1}^{\infty}|S_r\cap G_i|\eps_r^{-2}+\eps_r^{-2}
			= \eps_r^{-2}(|S_r\cap G_{\ge r-1}|+1)\text{.}
			\end{equation}

		We consider the following three cases separately:
			\begin{itemize}
				\item Case 1. $\ME$ is correct and $\FT$ returns True.
				\item Case 2. $\ME$ is correct and $\FT$ returns False.
				\item Case 3. $\ME$ is incorrect.
			\end{itemize}
		Let $P_1$ through $P_3$ denote the conditional probability of the event that $\EL$ fails at some round while $\US$ and $\FT$ are correct in Case 1 through Case 3.

		\textbf{Upper bound $P_1$.}
		Assuming that $\FT$ returns True, procedure $\EL$ will be called at round $r$. By a union bound, we have $P_1\le P(r+1, S_{r+1}) + \delta'_r$, where $S_{r+1}$ is the set of arms returned by $\EL$. According to the inductive hypothesis, the monotonicity of $C$, and identity \eqref{eq:Cdiff},
			\begin{equation}\begin{split}\label{eq:toyPbound}
				P(r+1, S_{r+1})
			&\le\frac{\delta}{\hat H}\left(128C(r+1, S_{r+1})+16M(r+1, S_{r+1})\eps_{r+1}^{-2}\right)\\
			&\le\frac{\delta}{\hat H}\left(128C(r+1, S_r)+64M(r+1, S_{r+1})\eps_r^{-2}\right)\\
			&=	\frac{\delta}{\hat H}\left[128C(r, S_r)+\eps_r^{-2}(64M(r+1, S_{r+1})-128|S_r\cap G_{\ge r-1}|-128)\right]\text{.}
			\end{split}\end{equation}

		By Observation \ref{obs:toyEL}, $\dlow_r\le\Mean{1}-\eps_r$. If $\EL$ returns correctly at round $r$, we have
			$$M(r+1,S_{r+1})
			=|\{A\in S_{r+1}:\mu_A < \Mean{1}-\eps_r\}|
			\le|\{A\in S_{r+1}:\mu_A < \dlow_r\}|
			<0.1|S_{r+1}|\le0.1|S_r|\text{.}$$
		For brevity, let $\Nsma$, $\Ncur$ and $\Nbig$ denote
		$|S_r\cap G_{\le r-2}|$, $|S_r\cap G_{r-1}|$ and $|S_r\cap G_{\ge r}|$, respectively.
		Note that
			$|S_r| = \Nsma + \Ncur + \Nsma + 1$.
		Then we have
			\begin{equation*}\begin{split}
				P_1
			&\le{P(r+1, S_{r+1})} + \delta'_r\\
			&\le\frac{\delta}{\hat H}\left[128C(r, S_r)+\eps_r^{-2}(|S_r|+64M(r+1, S_{r+1})-128|S_r\cap G_{\ge r-1}|-128)\right]\\
			&\le\frac{\delta}{\hat H}\left[128C(r, S_r)+\eps_r^{-2}(7.4(\Nsma+\Ncur+\Nbig+1)-128(\Ncur+\Nbig+1))\right]\\
			&\le\frac{\delta}{\hat H}\left[128C(r, S_r)+7.4\eps_r^{-2}\Nsma\right]\text{.}
			\end{split}\end{equation*}

		\textbf{Upper bound $P_2$.} Since $\FT$ returns True, procedure $\EL$ is not called. Then $P_2\le P(r+1, S_{r+1}) = P(r+1, S_r)$. By inequality \eqref{eq:toyPbound},
			\begin{equation*}\begin{split}
			P_2
			&\le \frac{\delta}{\hat H}\left[128C(r, S_r)+\eps_r^{-2}(64M(r+1, S_r)-128|S_r\cap G_{\ge r-1}|-128)\right]\\
			&\le \frac{\delta}{\hat H}\left[128C(r, S_r)+\eps_r^{-2}(64(\Nsma+\Ncur)-128(\Ncur+\Nbig+1))\right]\\
			&\le \frac{\delta}{\hat H}\left[128C(r, S_r)+64\eps_r^{-2}(\Nsma-\Ncur-\Nbig-1)\right]
			\le \frac{\delta}{\hat H}\cdot 128C(r, S_r)\text{.}
			\end{split}\end{equation*}
		Here the last step holds since by Observation \ref{obs:toyFT}, $\clow_r\ge\Mean{1}-2\eps_r$, and thus $\FT$ returns False implies that
			$$\Nsma
			= |S_r\cap G_{\le r - 2}|
			= |\{A\in S_r:\mu_A < \eps_{r-1}\}|
			< 0.5|S_r| = (\Nsma+\Ncur+\Nbig+1)/2\text{.}$$

		\textbf{Upper bound $P_3$.} By \eqref{eq:toyPbound} and $M(r+1,S_{r+1})\le M(r+1,S_r)$, we have
			\begin{equation*}\begin{split}
			P_3
			&\le P(r+1, S_{r+1}) + \delta'_r\\
			&\le \frac{\delta}{\hat H}\left[128C(r, S_r)+\eps_r^{-2}(64M(r+1, S_r)-128|S_r\cap G_{\ge r-1}|-128+|S_r|)\right]\\
			&= 	\frac{\delta}{\hat H}\left[128C(r, S_r)+\eps_r^{-2}(64(\Nsma+\Ncur)-128(\Ncur+\Nbig)-128+(\Nsma+\Ncur+\Nbig+1))\right]\\
			&\le \frac{\delta}{\hat H}\left[128C(r, S_r)+65\eps_r^{-2}\Nsma\right]\text{.}
			\end{split}\end{equation*}

		Recall that Case 3 happens with probability at most $0.01$, and
			$\Nsma = |S_r\cap G_{\le r-2}| = M(r, S_r)$.
		Therefore, we obtain the following bound on $P(r, S_r)$, which finishes the proof.
			\begin{equation*}\begin{split}
			P(r, S_r)
			&\le 0.01\cdot\frac{\delta}{\hat H}\left[128C(r, S_r)+65\eps_r^{-2}\Nsma\right] + 0.99\cdot\frac{\delta}{\hat H}\left[128C(r, S_r)+7.4\eps_r^{-2}\Nsma\right]\\
			&\le \frac{\delta}{\hat H}\left(128C(r, S_r)+16\eps_r^{-2}\Nsma\right)\\
			&\le \frac{\delta}{\hat H}\left(128C(r, S_r)+16M(r, S_r)\eps_r^{-2}\right)\text{.}
			\end{split}\end{equation*}
	\end{proof}

	\subsection{Proof of Lemma \ref{lem:newtoysample}}
	\begin{proof}[Proof of Lemma \ref{lem:newtoysample} (continued)]
		Recall that for each round $i$, $r_i$ is defined as the largest integer $r$ such that $|G_{\ge r}|\ge0.5|S_i|$, and
			$$
			T_{i,j} = \begin{cases}
			0, & j < r_i,\\
			\eps_i^{-2}\left(\ln\delta^{-1}+\ln\dfrac{H}{|G_j|\eps_i^{-2}}\right), & j\ge r_i
			\end{cases}
			$$
		is the number of samples that each arm in $G_j$ is charged at round $i$.

		We first show that
			$\sum_{j}|G_j|T_{i,j}$
		is an upper bound on
			$|S_i|\eps_i^{-2}\left(\ln\delta^{-1}+\ln\dfrac{H}{|S_i|\eps_i^{-2}}\right)$,
		the number of samples taken by $\ME$ and $\EL$ at round $i$.
		Recall that $|G_{\ge r_i}|\ge0.5|S_i|$. By definition of $T_{i,j}$,
			\begin{equation*}\begin{split}
				\sum_{j}|G_j|T_{i,j}
			&=	\sum_{j\ge r_i}|G_j|\eps_i^{-2}\left(\ln\delta^{-1}+\ln\frac{H}{|G_j|\eps_i^{-2}}\right)\\
			&\ge|G_{\ge r_i}|\eps_i^{-2}\left(\ln\delta^{-1}+\ln\frac{H}{|S_i|\eps_i^{-2}}\right)\\
			&\ge\frac{1}{2}|S_i|\eps_i^{-2}\left(\ln\delta^{-1}+\ln\frac{H}{|S_i|\eps_i^{-2}}\right)\text{.}
			\end{split}\end{equation*}

		Then we prove the upper bound on $\sum_{i}\Ex[T_{i,j}]$, the expected number of samples that each arm in $G_j$ is charged. For $i\le j + 1$, we have the straightforward bound
			\begin{equation}\label{eq:toyTbound1}
				\Ex[T_{i,j}]\le \eps_i^{-2}\left(\ln\delta^{-1}+\ln\dfrac{H}{|G_j|\eps_i^{-2}}\right)\text{.}
			\end{equation}
		For $i \ge j + 2$, we note that $T_{i,j}$ is non-zero only if $j\ge r_i$, which implies that $|G_{\ge j+1}|<0.5|S_i|$. We claim that this happens only if $\ME$ fails between round $j+2$ and round $i-1$, which happens with probability at most $0.01^{i-j-1}$. In fact, suppose $\ME$ is correct at some round $k$, where $j+2\le k\le i-1$. By Observations \ref{obs:toyFT} and \ref{obs:toyEL}, $\clow_k\ge\Mean{1}-2\eps_k$ and $\dlow_k\ge\Mean{1}-\eps_k$, where $\clow$ and $\dlow$ are the two lower thresholds used in $\FT$ and $\EL$. If $\FT$ returns False, by Fact \ref{F3}, we have
			$$
				|S_k\cap G_{< k-1}|
			=	\{A\in S_k:\mu_A<\Mean{1}-2\eps_k\}
			\le	\{A\in S_k:\mu_A<\clow_k\}
			<	0.5|S_k|\text{.}$$
		Since $S_{k+1}=S_k$ in this case, it follows that
			$|S_{k+1}\cap G_{< k-1}|<0.5|S_{k+1}|\text{.}$
		If $\FT$ returns True and the algorithm calls $\EL$, by Fact \ref{F4},
			$$
			|S_{k+1}\cap G_{<k}|
			=|\{A\in S_{k+1}:\mu_A<\Mean{1}-\eps_k\}|
			\le|\{A\in S_{k+1}:\mu_A<\dlow_k\}|
			<0.1|S_{k+1}|\text{.}$$
		In either case, we have
			$|S_{k+1}\cap G_{\ge k-1}|>0.5|S_{k+1}|$,
		and thus,
			$$|G_{\ge j+1}|\ge|G_{\ge k-1}|\ge|S_{k+1}\cap G_{\ge k-1}|>0.5|S_{k+1}|\ge0.5|S_i|\text{,}$$
		which contradicts $|G_{\ge j+1}|<0.5|S_i|$.
		Therefore, for $i \ge j + 2$, we have
			\begin{equation}\begin{split}\label{eq:toyTbound2}
				\Ex[T_{i,j}] &= \pr{T_{i,j}>0}\cdot\eps_i^{-2}\left(\ln\delta^{-1}+\ln\dfrac{H}{|G_j|\eps_i^{-2}}\right)\\
				&\le 0.01^{i-j-1}\cdot\eps_i^{-2}\left(\ln\delta^{-1}+\ln\dfrac{H}{|G_j|\eps_i^{-2}}\right)\text{.}
			\end{split}\end{equation}

		By \eqref{eq:toyTbound1} and \eqref{eq:toyTbound2}, a direct summation gives
			$$\sum_{i}\Ex[T_{i,j}] = O\left(\eps_j^{-2}\left(\ln\delta^{-1}+\ln\frac{H}{|G_j|\eps_j^{-2}}\right)\right)\text{.}$$
	\end{proof}

\end{document}